\documentclass[11pt, dvipsnames, DIV=12]{scrartcl}
\usepackage[english]{babel}

\usepackage[utf8]{inputenc}
\usepackage[T1]{fontenc}
\usepackage{natbib}
\usepackage{subcaption}
\usepackage{booktabs}
\usepackage{multirow}
\usepackage{url}
\usepackage{soul}
\usepackage{tikz}
\usetikzlibrary{calc}
\usetikzlibrary{decorations.pathmorphing}
\usetikzlibrary{positioning}
\usetikzlibrary{shapes}
\usetikzlibrary{arrows.meta,backgrounds,automata}
\usetikzlibrary{positioning,shapes,matrix,calc}
\tikzstyle{vertex}=[circle, draw, fill=gray!80!white,thick,scale=1.2]
\tikzstyle{edge}=[draw=black, thick,-]
\usetikzlibrary{hobby,arrows,decorations.pathmorphing,backgrounds,positioning,fit,petri,calc}
\pgfdeclarelayer{background}
\pgfsetlayers{background,main}

\usepackage{todonotes}

\usepackage{paralist}
\usepackage[stable]{footmisc}
\usepackage{adjustbox}

\usepackage{ifthen}
\newcommand{\CC}[1][]{$\text{C\hspace{-.25ex}}^{_{_{_{++}}}}
\ifthenelse{\equal{#1}{}}{}{\text{\hspace{-.625ex}#1}}$}

\usepackage{bm}
\usepackage{bbm}
\usepackage{amsmath}
\usepackage{amssymb}
\usepackage{amsthm}
\usepackage{amsfonts}
\usepackage{thmtools}		
\usepackage{mleftright}
\usepackage{stmaryrd}
\usepackage{nicefrac}
\usepackage{algorithm}
\usepackage{algorithmicx}
\usepackage[noend]{algpseudocode}

\usepackage{thm-restate}

\let\originalleft\left
\let\originalright\right
\renewcommand{\left}{\mathopen{}\mathclose\bgroup\originalleft}
\renewcommand{\right}{\aftergroup\egroup\originalright}

\usepackage{pifont}

\usepackage{enumitem}
\setlist[enumerate]{itemsep=0.2ex, topsep=0.5\topsep}
\setlist[description]{itemsep=0.2ex, topsep=0.5\topsep}
\setlist[itemize]{itemsep=0.2ex, topsep=0.5\topsep}

\makeatletter
\def\thmt@refnamewithcomma #1#2#3,#4,#5\@nil{%
\@xa\def\csname\thmt@envname #1utorefname\endcsname{#3}%
\ifcsname #2refname\endcsname
\csname #2refname\expandafter\endcsname\expandafter{\thmt@envname}{#3}{#4}%
\fi
}
\makeatother

\usepackage[pagebackref,
pdfa,
hidelinks,
pdftex, 
pdfdisplaydoctitle,
pdfpagelabels,
pdfauthor={Christopher Morris},
pdftitle={Weisfeiler--Leman at the margin: When more expressivity matters},
pdfsubject={},
pdfkeywords={GNNs, generalization, margin, expressivity},
pdfproducer={Latex with the hyperref package},
pdfcreator={pdflatex}
]{hyperref}

\usepackage[capitalise,noabbrev]{cleveref}   
\theoremstyle{definition}
\newtheorem{theorem}{Theorem}

\newtheorem{proposition}[theorem]{Proposition}

\newtheorem{lemma}[theorem]{Lemma}
\newtheorem{corollary}[theorem]{Corollary}

\newtheorem{remark}[theorem]{Remark}

\usepackage{thm-restate}
\usepackage[mathic=true]{mathtools}
\usepackage{fixmath}
\usepackage{siunitx}

\usepackage{microtype}
\usepackage{ellipsis}

\usepackage[scaled=0.86]{helvet}
\usepackage{lmodern}

\usepackage[auth-lg]{authblk}

\newcommand{\mX}{\mathbold{X}}

\newcommand{\mW}{\mathbold{W}}

\newcommand{\cE}{\mathcal{E}}
\newcommand{\cF}{\mathcal{F}}
\newcommand{\cG}{\mathcal{G}}
\newcommand{\cH}{\mathcal{H}}

\newcommand{\cO}{\mathcal{O}}

\newcommand{\cR}{\mathcal{R}}

\newcommand{\cX}{\mathcal{X}}

\newcommand{\Hb}{\mathbb{H}}
\newcommand{\Nb}{\mathbb{N}}

\newcommand{\Rb}{\mathbb{R}}
\newcommand{\Sb}{\mathbb{S}}

\newcommand{\MPNN}{\textsf{MPNN}}

\newcommand{\vcdim}{\ensuremath{\mathsf{VC}\text{-}\mathsf{dim}}}

\newcommand{\wlone}{$1$\textrm{-}\textsf{WL}}
\newcommand{\wlonef}{$1$\textrm{-}\textsf{WL}$_{\cF}$}
\newcommand{\MPNNF}{\text{MPNN}$_{\cF}$}
\newcommand{\tMPNNF}{\text{MPNN}$_{\cF}$}
\newcommand{\wloa}{$1$\textrm{-}\textsf{WLOA}}
\newcommand{\wloaf}{$1$\textrm{-}\textsf{WLOA}$_{\cF}$}
\newcommand{\kwl}{$k$\textrm{-}\textsf{WL}}

\newcommand{\vc}[1]{\textsf{VC}(#1)}
\newcommand{\conv}[1]{\textsf{conv}(#1)}

\newcommand{\fb}{\mathbold{f}}
\newcommand{\hb}{\mathbold{h}}

\newcommand{\Ibb}{\mathbb{I}}

\newcommand{\Nbb}{\mathbb{N}}

\newcommand{\UPD}{\mathsf{UPD}}
\newcommand{\AGG}{\mathsf{AGG}}
\newcommand{\RO}{\mathsf{READOUT}}

\newcommand{\REL}{\mathsf{RELABEL}}

\newcommand{\relu}{\mathsf{reLU}}
\newcommand{\sign}{\mathsf{sign}}

\newcommand{\new}[1]{\emph{#1}}

\newcommand*{\tran}{\intercal}
\newcommand{\smmpnn}{\mathsf{mpnn}}
\renewcommand{\vec}[1]{\mathbold{#1}}
\newcommand{\oms}{\{\!\!\{}
\newcommand{\cms}{\}\!\!\}}
\newcommand{\tup}[1]{{(#1)}}
\DeclarePairedDelimiterX{\norm}[1]{\lVert}{\rVert}{#1}

\DeclareFontFamily{U}{mathx}{\hyphenchar\font45}
\DeclareFontShape{U}{mathx}{m}{n}{<-> mathx10}{}
\DeclareSymbolFont{mathx}{U}{mathx}{m}{n}
\DeclareMathAccent{\widebar}{0}{mathx}{"73}

\newcommand{\wprod}{\vec{W}_\mathrm{prod}}
 
\DeclareMathOperator*{\argmax}{arg\,max}

\newcommand{\wj}{\vec{W}^{(j)}}
\newcommand{\wjp}{\vec{W}^{(j+1)}}
\newcommand{\wmat}{\vec{W}}

\title{\huge \normalfont\textbf{Weisfeiler--Leman at the margin: When more expressivity matters}}

\author[1]{Billy J.\,Franks$^*$}
\author[2]{Christopher Morris$^*$}
\author[3]{Ameya Velingker}
\author[4]{Floris Geerts}

\affil[1]{University of Kaiserslautern-Landau}
\affil[2]{RWTH Aachen University}
\affil[3]{Google Research}
\affil[4]{University of Antwerp}
\date{\vspace{-30pt}}

\recalctypearea
\begin{document}

\maketitle
\def\thefootnote{*}\footnotetext{First authors with equal contributions.}
\def\thefootnote{\arabic{footnote}}

\begin{abstract}
	The Weisfeiler--Leman algorithm (\wlone) is a well-studied heuristic for the graph isomorphism problem. Recently, the algorithm has played a prominent role in understanding the expressive power of message-passing graph neural networks (MPNNs) and being effective as a graph kernel. Despite its success, \wlone{} faces challenges in distinguishing non-isomorphic graphs, leading to the development of more expressive MPNN and kernel architectures. However, the relationship between enhanced expressivity and improved generalization performance remains unclear. Here, we show that an architecture's expressivity offers limited insights into its generalization performance when viewed through graph isomorphism. Moreover, we focus on augmenting \wlone{} and MPNNs with subgraph information and employ classical margin theory to investigate the conditions under which an architecture's increased expressivity aligns with improved generalization performance. In addition, we show that gradient flow pushes the MPNN's weights toward the maximum margin solution. Further, we introduce variations of expressive \wlone-based kernel and MPNN architectures with provable generalization properties. Our empirical study confirms the validity of our theoretical findings.
\end{abstract}

\section{Introduction}
Graph-structured data are prevalent in application domains ranging from chemo- and bioinformatics~\citep{Jum+2021,Sto+2020,Won+2023}, combinatorial optimization~\citep{Cap+2021}, to image~\citep{Sim+2017} and social-network analysis~\citep{Eas+2010}, underlying the importance of machine learning methods for graphs. Nowadays, there are numerous approaches for machine learning on graphs, most notably those based on \new{graph kernels}~\citep{Borg+2020,Kri+2019} or \new{message-passing graph neural networks} (MPNNs)~\citep{Gil+2017,Sca+2009}. Here, graph kernels~\citep{She+2011} based on the \new{$1$-dimensional Weisfeiler--Leman algorithm} (\wlone)~\citep{Wei+1968}, a well-studied heuristic for the graph isomorphism problem, and corresponding MPNNs~\citep{Mor+2019,Xu+2018b}, have recently advanced the state-of-the-art in supervised vertex- and graph-level learning~\citep{Mor+2022}.

However, due to \wlone{}'s limitations in distinguishing non-isomorphic graphs~\citep{Arv+2015,Cai+1992}, numerous recent works proposed more expressive extensions of the \wlone{} and corresponding MPNNs~\citep{Mor+2022}. For example,~\citet{Bou+2020} introduced an approach to enhance the \wlone{} and MPNNs by incorporating subgraph information, achieved by labeling vertices based on their structural roles regarding a set of predefined (sub)graphs. Through the careful selection of such graphs, \citet{Bou+2020} demonstrated that these enhanced \wlone{} and MPNNs variants can effectively discriminate between pairs of non-isomorphic graphs, which the \wlone{} and \kwl~\citep{Cai+1992}, \wlone's more expressive generalization, cannot. \emph{Furthermore, empirical results~\citep{Bou+2020} indicate that this added expressive power often results in improved predictive performance. However, the exact mechanisms underlying this performance improvement remain unclear.}

Although recent work~\citep{Mor+2023} has used \wlone{} to establish upper and lower bounds on the Vapnik--Chervonenkis (VC) dimension of MPNNs, these findings do not explain the above empirical observations. Specifically, they do not explain the empirical trend that increased expressive power corresponds to enhanced generalization performance while keeping the size of the training set fixed. Concretely,~\cite{Mor+2023} demonstrated a strong correlation between the VC dimension of MPNNs and the number of non-isomorphic graphs that \wlone{} can differentiate. Consequently, increasing \wlone's expressive capabilities increases the VC dimension, worsening generalization performance. A parallel argument can be made regarding \wlone-based kernels.

\paragraph{Present work}  Here, we investigate to what extent the \wlone{} and more expressive extensions can be used as a proxy for an architecture's predictive performance. First, we show that data distributions exist such that the \wlone{} and corresponding MPNNs distinguish every pair of non-isomorphic graphs with different class labels while no classifier can do better than random outside of the training set. Hence, we show that the graph isomorphism perspective is too limited to understand MPNNs' generalization properties. Secondly, based on~\citet{Alo+2021}'s theory of partial concepts, we derive \emph{tight} upper and lower bounds for the VC dimension of the \wlone-based kernels and corresponding MPNNs parameterized by the \new{margin} separating the data. In addition, building on~\citet{JiT19}, we show that gradient flow pushes the MPNN's weights toward the maximum margin solution. Thirdly, we show when \wlone{} variants using subgraph information can make the data linearly separable, leading to a positive margin. Building on this, we derive conditions under which more expressive \wlone{} variants lead to better generalization performance and derive \wlone{} variants with favorable generalization properties. Our empirical study confirms the validity of our theoretical findings.

Our theory establishes the first link between increased expressive power and improved generalization performance. Moreover, our results provide the first margin-based lower bounds for MPNNs' VC dimension. \emph{Overall, our results provide new insights into when more expressive power translates into better generalization performance, leading to a more fine-grained understanding of designing expressive MPNNs.}

\subsection{Related work}

In the following, we discuss relevant related work.

\paragraph{Graph kernels based on the \wlone} \citet{She+2011} were the first to utilize the \wlone{} as a graph kernel. Later,~\citet{Mor+2017,Morris2020b,Mor+2022b} generalized this to variants of the \kwl{}. Moreover, \citet{Kri+2016} derived the \new{Weisfeiler-Leman optimal assignment kernel}, using the \wlone{} to compute optimal assignments between vertices of two given graphs. \citet{Yan+2015} successfully employed Weisfeiler--Leman kernels within frameworks for smoothed~\citep{Yan+2015} and deep graph kernels~\citep{Yan+2015a}. For a theoretical investigation of graph kernels based on the \wlone, see~\cite{Kri+2018}. See also~\cite {Mor+2022} for an overview of the Weisfeiler--Leman algorithm in machine learning and \citet{Borg+2020,Kri+2019} for a detailed review of graph kernels.

\paragraph{MPNNs} Recently, MPNNs~\citep{Gil+2017,Sca+2009} emerged as the most prominent graph maschine learning architecture. Notable instances of this architecture include, e.g.,~\citet{Duv+2015,Ham+2017}, and~\citet{Vel+2018}, which can be subsumed under the message-passing framework introduced in~\citet{Gil+2017}. In parallel, approaches based on spectral information were introduced in, e.g.,~\citet{Bru+2014,Defferrard2016,Gam+2019,Kip+2017,Lev+2019}, and~\citet{Mon+2017}---all of which descend from early work in~\citet{bas+1997,Gol+1996,Kir+1995,Mer+2005,mic+2005,mic+2009,Sca+2009}, and~\citet{Spe+1997}.

\paragraph{Expressive power of MPNNs} Recently, connections between MPNNs and Weisfeiler--Leman-type algorithms have been shown~\citep{Bar+2020,Gee+2020a,Mor+2019,Xu+2018b}. Specifically,~\citet{Mor+2019} and~\citet{Xu+2018b} showed that the \wlone{} limits the expressive power of any possible MPNN architecture in distinguishing non-isomorphic graphs. In turn, these results have been generalized to the \kwl, e.g.,~\citet{Mar+2019,Mor+2019,Morris2020b,Mor+2022b}, and connected to the permutation-equivariant function approximation over graphs, see, e.g.,~\citet{Che+2019,geerts2022,Mae+2019,Azi+2020}. Furthermore,~\citet{Aam+2022,Ami+2023} devised an improved analysis using randomization and moments of neural networks, respectively. Recent works have extended the expressive power of MPNNs, e.g., by encoding vertex identifiers~\citep{Mur+2019b, Vig+2020}, using random features~\citep{Abb+2020,Das+2020,Sat+2020} or individualization-refinement algorithms~\citep{Fra+2023}, affinity measures~\citep{Vel+2022}, equivariant graph polynomials~\citep{Pun+2023}, homomorphism and subgraph counts~\citep{Bar+2021,Bou+2020,Hoa+2020}, spectral information~\citep{Bal+2021}, simplicial~\citep{Bod+2021} and cellular complexes~\citep{Bod+2021b}, persistent homology~\citep{Hor+2021}, random walks~\citep{Toe+2021,Mar+2022}, graph decompositions~\citep{Tal+2021}, relational~\citep{Bar+2022}, distance~\citep{li2020distance} and directional information~\citep{beaini2020directional}, graph rewiring~\citep{Qia+2023} and adaptive message passing~\citep{Fin+2023}, subgraph information~\citep{Bev+2021,Cot+2021,Fen+2022,Fra+2022,Hua+2022,Mor+2022,Pap+2021,Pap+2022,Qia+2022,Thi+2021,Wij+2022,You+2020,Zha+2021,Zha+2022,Zha+2023b}, and biconnectivity~\citep{Zha+2023}. See~\citet{Mor+2022} for an in-depth survey on this topic. \citet{geerts2022} devised a general approach to bound the expressive power of a large variety of MPNNs using \wlone{} or \kwl{}.

Recently,~\citet{Kim+2022} showed that transformer architectures~\citep{Mue+2023} can simulate the $2$-$\mathsf{WL}$. \citet{Gro+2023} showed tight connections between MPNNs' expressivity and circuit complexity. Moreover,~\citet{Ros+2023} investigated the expressive power of different aggregation functions beyond sum aggregation. Finally, \citet{Boe+2023} defined a continuous variant of the \wlone, deriving a more fine-grained topological characterization of the expressive power of MPNNs.

\paragraph{Generalization abilities of graph kernels and MPNNs}
\citet{Sca+2018} used classical techniques from learning theory~\citep{Kar+1997} to show that MPNNs' VC dimension~\citep{Vap+95} with piece-wise polynomial activation functions on a \emph{fixed} graph, under various assumptions, is in $\cO(P^2n\log n)$, where $P$ is the number of parameters and $n$ is the order of the input graph; see also~\citet{Ham+2001}. We note here that~\citet{Sca+2018} analyzed a different type of MPNN not aligned with modern MPNN architectures~\citep{Gil+2017}. \citet{Gar+2020} showed that the empirical Rademacher complexity (see, e.g.,~\citet{Moh+2012}) of a specific, simple MPNN architecture, using sum aggregation, is bounded in the maximum degree, the number of layers, Lipschitz constants of activation functions, and parameter matrices' norms. We note here that their analysis assumes weight sharing across layers. \citet{Lia+2021} refined these results via a PAC-Bayesian approach, further refined in~\citet{Ju+2023}. \citet{Mas+2022} used random graphs models to show that MPNNs' generalization ability depends on the (average) number of vertices in the resulting graphs. In addition,~\citet{Lev+2023} defined a measure of a natural graph-signal similarity notion, resulting in a generalization bound for MPNNs depending on the covering number and the number of vertices. \citet{Ver+2019} studied the generalization abilities of $1$-layer MPNNs in a transductive setting based on algorithmic stability. Similarly,~\citet{Ess+2021} used stochastic block models to study the transductive Rademacher complexity~\citep{Yan+2009,Tol+2016} of standard MPNNs. For semi-supervised node classification,~\citet{Bar+2021a} studied the classification of a mixture of Gaussians, where the data corresponds
to the node features of a stochastic block model, under which conditions the mixture model is linearly separable using the GCN layer~\citep{Kip+2017}. Recently,~\cite{Mor+2023} made progress connecting MPNNs' expressive power and generalization ability via the Weisfeiler--Leman hierarchy. They studied the influence of graph structure and the parameters' encoding lengths on MPNNs' generalization by tightly connecting \wlone's expressivity and MPNNs' Vapnik--Chervonenkis (VC) dimension. They derived that MPNNs' VC dimension depends tightly on the number of equivalence classes computed by the \wlone{} over a given set of graphs. Moreover, they showed that MPNNs' VC dimension depends logarithmically on the number of colors computed by the \wlone{} and polynomially on the number of parameters. \citet{Kri+2018} leveraged results from graph property testing~\citep{Gol2010} to study the sample complexity of learning to distinguish various graph properties, e.g., planarity or triangle freeness, using graph kernels~\citep{Borg+2020,Kri+2019}. Finally,~\cite{Yeh+2021} showed negative results for MPNNs' ability to generalize to larger graphs.

\paragraph{Margin theory and VC dimension} Using the margin as a regularization mechanism dates back to~\cite{Vap+1964}. Later, the concept of margin was successfully applied to \new{support vector machines} (SVMs)~\citep{Cor+1995,Vap+1998} and connected to VC dimension theory; see~\cite{Moh+2012} for an overview. \citet{Grn+2020} derived the so far tightest generalization bounds for SVMs.

\citet{Alo+2021} introduced the theory of VC dimension of \new{partial concepts}, i.e., the hypothesis set allows partial functions and showed, analogous to the standard case, that finite VC dimension implies learnability and vice versa.

\section{Background}\label{sec:background}

Let $\Nb \coloneqq \{ 1,2,3, \dots \}$. For $n \geq 1$, let $[n] \coloneqq \{ 1, \dotsc, n \} \subset \Nb$. We use $\oms \dotsc \cms$ to denote multisets, i.e., the generalization of sets allowing for multiple instances for each of its elements. For two sets $X$ and $Y$, let $X^Y$ denote the set of functions mapping from $Y$ to $X$. Let $S \subset \Rb^d$, then the \new{convex hull} \conv{S} is the minimal convex set containing the set $S$. For $\vec{p} \in \Rb^d, d > 0$, and $\varepsilon > 0$, the \new{ball} $B(\vec{p}, \varepsilon, d) \coloneqq \{ \vec{s} \in \Rb^d \mid \lVert \vec{p} - \vec{s} \rVert \leq \varepsilon  \}$. Here, and in the remainder of the paper, $\|\cdot\|$ refers to the \new{$2$-norm} $\|\vec{x}\|\coloneqq\sqrt{x_1^2+\cdots+x_d^2}$ for $\vec{x}\in\Rb^d$.

\paragraph{Graphs} An \new{(undirected) graph} $G$ is a pair $(V(G),E(G))$ with \emph{finite} sets of
\new{vertices} or \new{nodes} $V(G)$ and \new{edges} $E(G) \subseteq \{ \{u,v\} \subseteq V(G) \mid u \neq v \}$. For ease of notation, we denote an edge $\{u,v\}$ in $E(G)$ by $(u,v)$ or $(v,u)$. The \new{order} of a graph $G$ is its number $|V(G)|$ of
vertices. If not stated otherwise, we set $n \coloneqq |V(G)|$ and call $G$ an $n$-order graph. We denote the set of all $n$-order graphs by $\cG_n$. For a graph $G\in\cG_n$, we denote its \new{adjacency matrix} by $\vec{A}(G) \in \{ 0,1 \}^{n \times n}$, where $A(G)_{vw} = 1$ if, and only, if $(v,w) \in E(G)$. The \new{neighborhood} of $v \in V(G)$ is denoted by $N(v) \coloneqq  \{ u \in V(G) \mid (v, u) \in E(G) \}$ and the \new{degree} of a vertex $v$ is $|N(v)|$. A \new{(vertex-)labeled graph} $G$ is a triple $(V(G),E(G),\ell)$ with a (vertex-)label function $\ell \colon V(G) \to \Nb$. Then $\ell(v)$ is a \new{label} of $v$, for $v \in V(G)$. For $X \subseteq V(G)$, the graph $G[X] \coloneqq  (X,E_X)$ is the \new{subgraph induced by $X$}, where $E_X \coloneqq \{ (u,v) \in E(G) \mid u,v \in X \}$.  Two graphs $G$ and $H$ are \new{isomorphic}, and we write $G \simeq H$ if there exists a bijection $\varphi \colon V(G) \to V(H)$ preserving the adjacency relation, i.e., $(u,v)$ is in $E(G)$ if, and only, if $(\varphi(u),\varphi(v))$ is in $E(H)$. Then $\varphi$ is an \new{isomorphism} between $G$ and $H$. In the case of labeled graphs, we additionally require that $l(v) = l(\varphi(v))$ for all $v$ in $V(G)$. We denote the \new{complete graph} on $n$ vertices by $K_n$ and a \new{cycle} on $n$ vertices by $C_n$. for $r \geq 0$, a graph is \new{$r$-regular} if all of its vertices have degree $r$. Given two graphs $G$ and $H$ with disjoint vertex sets, we denote their disjoint union by $G \,\dot\cup\, H$.

\paragraph{Kernels} A \emph{kernel} on a non-empty set $\mathcal{X}$ is a symmetric, positive semidefinite function $k \colon \mathcal{X} \times \mathcal{X} \to \mathbb{R}$. Equivalently, a function $k \colon \mathcal{X} \times \mathcal{X} \to \mathbb{R}$ is a kernel if there is a \emph{feature map} $\phi \colon \mathcal{X} \to \mathcal{H}$ to a Hilbert space $\mathcal{H}$ with inner product $\langle \cdot, \cdot \rangle$ such that $k(x,y) = \langle \phi(x),\phi(y) \rangle$ for all $x$ and $y \in \mathcal{X}$. We also call $\phi(x) \in \cH$ a \new{feature vector}. A \emph{graph kernel} is a kernel on the set $\mathcal{G}$ of all graphs. In the context of graph kernels, we also refer to a feature vector as a \new{graph embedding}.

\paragraph{VC Dimension of partial concepts} Let $\cX$ be a non-empty set. As outlined in~\citet{Alo+2021}, we consider \new{partial concepts} $\Hb \subseteq \{0,1,\star\}^{\cX}$, where each concept $c \in \Hb$ is a \emph{partial} function. That is, if $x \in \cX$ such that $c(x) = \star$, then $c$ is \new{undefined} at $x$. The \new{support} of a partial concept $h \in \Hb$ is the set
\begin{equation*}
	\mathsf{supp}(h) \coloneqq \{ x \in \cX \mid h(x) \neq \star \}.
\end{equation*}
The VC dimension of (total) concepts~\citep{Vap+95} straightforwardly generalizes to partial concepts. That is, the \new{VC dimension} of a partial concept class $\Hb$, denoted $\vc{\Hb}$, is the maximum cardinality of a shattered set $U \coloneqq \{ x_1, \dots, x_m \} \subseteq \cX$. Here, the set $U$ is \new{shattered} if for any $\pmb\tau \in \{0,1\}^m$ there exists $c \in \cH$ such that for all $i\in[m]$
\begin{equation*}
	c(x_i) = \tau_i.
\end{equation*}
In essence,~\citet{Alo+2021} showed that the standard definition of PAC learnability extends to partial concepts, recovering the equivalence of finite VC dimension and PAC learnability.

A  prime example of such a partial concept is defined in terms of separability by hyperplanes. That is,  let $\vec w\in \Rb^d$ with $\|\vec w\|=1$ and let $\lambda\in\Rb$ with $\lambda>0$. Consider the partial concept
\begin{equation*}
	h_{\vec{w},\lambda} \colon \Rb^d\to\{0,1,\star\} \colon \vec x\mapsto\begin{cases}
		1     & \text{if  $\vec w^{\tran} \vec{x}\geq\lambda$}  \\
		0     & \text{if  $\vec w^{\tran} \vec{x}\leq-\lambda$} \\
		\star & \text{otherwise,}
	\end{cases}
\end{equation*}
where $\lambda$ corresponds to the so-called \new{geometric margin} of the classifier.

\paragraph{Geometric margin classifiers}
As outlined above, classifiers with a geometric margin, e.g., \new{support vector machines}~\citep{Cor+1995}, are a cornerstone of machine learning. A \new{sample} $(\vec{x}_1, y_1), \dotsc, (\vec{x}_s, y_s) \in \Rb^{d} \times \{ 0,1 \}$, for $d > 0$, is \emph{$(r,\lambda)$-separable} if (1) there exists $\vec{p} \in \Rb^d$, $r > 0$, and a ball $B(\vec{p}, r, d)$ such that $\vec{x}_1, \dotsc, \vec{x}_s \in B(\vec{p}, r, d)$ and (2) the Euclidean distance between $\conv{\{ \vec{x}_i \mid y_i = 0  \}}$ and $\conv{\{ \vec{x}_i \mid y_i = 1 \}}$ is at least $2\lambda$. Then, the sample $S$ is \new{linearly separable} with \new{margin} $\lambda$. We define the set of concepts
\begin{align*}
	\Hb_{r,\lambda}(\Rb^d) \coloneqq \Big\{ h \in \{0,1,\star\}^{\mathbb{R}^d} \mathrel{\Big|}\, \forall\, & \vec{x}_1, \dots, \vec{x}_s \in \mathsf{supp}(h) \colon                                                \\
	(                                                                                                      & \vec{x}_1, h(\vec{x}_1)), \dotsc, (\vec{x}_s, h(\vec{x}_s)) \text{ is $(r,\lambda)$-separable} \Big\}.
\end{align*}
\citet{Alo+2021} showed that the VC dimension of the concept class $\Hb_{r,\lambda}(\Rb^d)$ is asymptotically lower- and upper-bounded by $\nicefrac{r^2}{\lambda^2}$. Importantly, the above bounds are independent of the dimension $d$, while standard VC dimension bounds scale linearly with $d$~\citep{Ant+2002}.

\subsection{The \texorpdfstring{$1$}{1}-dimensional Weisfeiler--Leman algorithm}\label{subsec:1WL} The \wlone{} or \new{color refinement} is a well-studied heuristic for the graph isomorphism problem, originally proposed by~\citet{Wei+1968}.\footnote{Strictly speaking, the \wlone{} and color refinement are two different algorithms. That is, the \wlone{} considers neighbors and non-neighbors to update the coloring, resulting in a slightly higher expressive power when distinguishing vertices in a given graph; see~\cite {Gro+2021} for details. Following customs in the machine learning literature, we consider both algorithms to be equivalent.} Intuitively, the algorithm determines if two graphs are non-isomorphic by iteratively coloring or labeling vertices. Given an initial coloring or labeling of the vertices of both graphs, e.g., their degree or application-specific information, in each iteration, two vertices with the same label get different labels if the number of identically labeled neighbors is unequal. These labels induce a vertex partition, and the algorithm terminates when, after some iteration, the algorithm does not refine the current partition, i.e., when a \new{stable coloring} or \new{stable partition} is obtained. Then, if the number of vertices annotated with a specific label is different in both graphs, we can conclude that the two graphs are not isomorphic. It is easy to see that the algorithm cannot distinguish all non-isomorphic graphs~\citep{Cai+1992}. However, it is a powerful heuristic that can successfully decide isomorphism for a broad class of graphs~\citep{Arv+2015,Bab+1979}.

Formally, let $G = (V(G),E(G),\ell)$ be a labeled graph. In each iteration, $t > 0$, the \wlone{} computes a vertex coloring $C^1_t \colon V(G) \to \Nb$, depending on the coloring of the neighbors. That is, in iteration $t>0$, we set
\begin{equation*}
	C^1_t(v) \coloneqq \REL\Big(\!\big(C^1_{t-1}(v),\oms C^1_{t-1}(u) \mid u \in N(v)  \cms \big)\! \Big),
\end{equation*}
for all vertices $v \in V(G)$, where $\REL$ injectively maps the above pair to a unique natural number, which has not been used in previous iterations. In iteration $0$, the coloring $C^1_{0}\coloneqq \ell$ is used. To test whether two graphs $G$ and $H$ are non-isomorphic, we run the above algorithm in ``parallel'' on both graphs. If the two graphs have a different number of vertices colored $c \in \Nb$ at some iteration, the \wlone{} \new{distinguishes} the graphs as non-isomorphic. Moreover, if the number of colors between two iterations, $t$ and $(t+1)$, does not change, i.e., the cardinalities of the images of $C^1_{t}$ and $C^1_{i+t}$ are equal, or, equivalently,
\begin{equation*}
	C^1_{t}(v) = C^1_{t}(w) \iff C^1_{t+1}(v) = C^1_{t+1}(w),
\end{equation*}
for all vertices $v$ and $w$ in $V(G\,\dot\cup H)$, then the algorithm terminates. For such $t$, we define the \new{stable coloring} $C^1_{\infty}(v) = C^1_t(v)$, for $v \in V(G\,\dot\cup H)$. The stable coloring is reached after at most $\max \{ |V(G)|,|V(H)| \}$ iterations~\citep{Gro2017}.

\paragraph{Graph kernels based on the $\boldsymbol{1}$-\textsf{WL}} Let $G$ be a graph. Following~\cite{She+2009}, the idea for a kernel based on the \wlone{} is to run the \wlone{} for $T \geq 0$ iterations, resulting in a coloring function $C^1_t \colon V(G) \to \Nb$ for each iteration $t \leq T$. Let $\Sigma_t$ denote the \emph{range} of $C^1_t$, i.e., $\Sigma_t\coloneqq\{ c\mid \exists\, v\in V(G) \colon C^1_t(v)=c\}$. We assume $\Sigma_t$ to be ordered by the natural order of $\Nb$, i.e., we assume that $\Sigma_t$ consists of $c_1< \cdots < c_{|\Sigma_t|}$. After each iteration, we compute a \new{feature vector} $\phi_t(G) \in \Rb^{|\Sigma_t|}$ for each graph $G$. Each component $\phi_t(G)_{i}$ counts the number of occurrences of vertices of $G$ labeled by $c_i \in \Sigma_t$. The overall feature vector $\phi_{\textsf{WL}}(G)$ is defined as the concatenation of the feature vectors of all $T$ iterations, i.e.,
\begin{equation*}
	\phi^{(T)}_{\textsf{WL}}(G) \coloneqq \big[\phi_0(G), \dotsc, \phi_T(G) \big],
\end{equation*}
where $[\dotsc]$ denote column-wise vector concatenation. This results in the kernel $k^{(T)}_{\textsf{WL}}(G,H) \coloneqq \langle \phi^{(T)}_{\textsf{WL}}(G), \phi^{(T)}_{\textsf{WL}}(H) \rangle$, where $\langle \cdot, \cdot \rangle$ denotes the standard inner product. We further define the \new{normalized \wlone{} feature vector},
\begin{equation*}\label{kernel}
	\widebar{\phi^{(T)}_{\textsf{WL}}}(G) \coloneqq \phi^{(T)}_{\textsf{WL}}(G) / \norm{\phi^{(T)}_{\textsf{WL}}(G) },
\end{equation*}
obtained by normalizing the \wlone{} feature vector to unit length.

\paragraph{Weisfeiler--Leman optimal assignment kernel}
Based on the \wlone, \cite{Kri+2016} defined the \new{Weisfeiler--Leman optimal assignment kernel} (\wloa), which computes an optimal assignment between the colors computed by the \wlone{} for all iterations; see~\citet{Kri+2016} for details. Given two graphs $G$ and $H$ and let $T\geq 0$, the \wloa{} computes
\begin{equation*}
	k_{\textsf{WLOA}}(G,H) \coloneqq \sum_{t \in [T] \cup
		\{ 0 \}} \sum_{c \in \Sigma_t} \min(\phi_t(G)_c,\phi_t(H)_c).
\end{equation*}
For a fixed but arbitrary number of vertices, we can compute a corresponding finite-dimensional feature map $\phi^{(T)}_{\textsf{WLOA}}$ for the set of $n$-order graphs. From the theory developed in~\citet{Kri+2016}, it follows that the \wloa{} kernel has the same expressive power as the \wlone{} in distinguishing non-isomorphic graphs.

\subsection{More expressive variants of the \texorpdfstring{\wlone}{1-WL}}\label{sec:wlf}
It is easy to see that the \wlone{} cannot distinguish all pairs of non-isomorphic graphs~\citep{Arv+2015,Cai+1992}. However, there exists a large set of more expressive extensions of the \wlone, which have been successfully leveraged as kernel or neural architectures~\citep{Mor+2022}. Moreover, empirical results suggest that such added expressive power often translates into increased predictive or generalization performance. Nonetheless, the precise mechanisms underlying this performance boost remain unclear.

In the following, we define a simple modification of the \wlone{} that is more expressive in distinguishing non-isomorphic graphs, the \wlonef. It is a simplified variant of the algorithms defined in~\citet{Bou+2020}, which take into account orbit information. Let $G$ be a graph and $\cF$ be a finite set of graphs. For $F \in \cF$, we define a vertex labeling $\ell_{F} \colon V(G) \to \Nb$ such that $\ell_{F}(v) = \ell_{F}(w)$ if, and only, if there exists $X_v \subseteq V(G)$ with $v \in X_v$ and $X_w \subseteq V(G)$ with  $w \in X_w$  such that $G[X_v] \simeq F$ and $G[X_w] \simeq F$. In other words, $\ell_F$ encodes the presence of subgraphs $G[X_v]$ in $G$, isomorphic to $F$ and containing vertex $v$.
Furthermore, we define the vertex labeling $\ell_{\cF} \colon V(G) \to \Nb$, where $\ell_{\cF}(v) = \ell_{\cF}(w)$ if, and only, if, for all $F \in \cF$, $\ell_{F}(v) = \ell_{F}(w)$. Finally, for $t \geq 0$, we define the vertex coloring  $C^{1,\cF}_t \colon V(G) \to \Nb$, where $C^{1,\cF}_0(v) \coloneqq \ell_{\cF}(v)$, and
\begin{equation*}
	C^{1,\cF}_t(v) \coloneqq \REL\Big(\!\big(C^{1,\cF}_{t-1}(v),\oms C^{1,\cF}_{t-1}(u) \mid u \in N(v)  \cms \big)\! \Big),
\end{equation*}
for $v \in V(G)$. Hence, the \wlonef{} only differs from the \wlone{} at the initialization step. The following result implies that, for all sets of graphs $\cF$, the \wlonef{} is at least as strong as \wlone{} in distinguishing non-isomorphic graphs.
\begin{proposition} \label{prop:WLFfiner}
	Let $G$ be a graph and $\cF$ be a set of graphs. Then, for all rounds, the \wlonef{} distinguishes at least the same vertices as the \wlone.  \qed
\end{proposition}
In addition, by choosing the set of graphs $\cF$ appropriately, the \wlonef{} becomes strictly more expressive than the \wlone{} in distinguishing non-isomorphic graphs.
\begin{proposition}\label{thm:wlonef_strong}
	For every $n \geq 6$, there exists at least one pair of non-isomorphic graphs and a set of graphs $\cF$ containing a single constant-order graph, such that, for all rounds, \wlone{} does not distinguish them, while \wlonef{} distinguishes them after a single round. \qed
\end{proposition}
We can also define a \wloa{} variant of the \wlonef{}, which we denote by \wloaf.

\paragraph{Graph kernels based on the \wlonef} Similar to the \wlone, we can also define a graph kernel based on the \wlonef{}. Let $G$ be a graph, we run the \wlonef{} for $T \geq 0$ iterations, resulting in a coloring function $C^{1,\cF}_{t} \to \Sigma_t$  for each iteration $t \leq T$. Let $\Sigma_t$ denote the \emph{range} of $C^{1,\cF}_{t}$, i.e., $\Sigma_t\coloneqq\{ c\mid \exists\, v\in V(G) \colon C^{1,\cF}_{t}(v)=c\}$. Again, we assume $\Sigma_t$ to be ordered by the natural order of $\Nb$, i.e., we assume that $\Sigma_t$ consists of $c_1<\cdots < c_{|\Sigma_t|}$. After each iteration, we compute a feature vector  $\phi_{\cF,t}(G) \in \Rb^{|\Sigma_t|}$ for each graph $G$. Each component $\phi_{\cF,t}(G)_{i}$ counts the number of occurrences of vertices of $G$ labeled by $c_i \in \Sigma_t$. The overall feature vector $\phi_{\textsf{WL}_{\cF}}(G)$ is defined as the concatenation of the feature vectors of all $T$ iterations, i.e.,
\begin{equation*}
	\phi^{(T)}_{\textsf{WL}_{\cF}}(G) \coloneqq \big[\phi_{\cF,0}(G), \dotsc, \phi_{\cF,T}(G) \big],
\end{equation*}
where $[\dots]$ denote column-wise vector concatenation. We then define the kernel and its normalized counterpart in the same way as with the \wlone.

\subsection{Message-passing graph neural networks}\label{sec:MPNN}
Intuitively, MPNNs learn a vectorial representation, i.e.,  $d$-dimensional real-valued vector, representing each vertex in a graph by aggregating information from neighboring vertices.
Formally, let $G = (V(G),E(G),\ell)$ be a labeled graph with initial vertex features $\hb_{v}^\tup{0} \in \Rb^{d}$ that are \emph{consistent} with $\ell$. That is, each vertex $v$ is annotated with a feature  $\hb_{v}^\tup{0} \in \Rb^{d}$ such that $\hb_{v}^\tup{0} = \hb_{u}^\tup{0}$ if, and only, if $\ell(v) = \ell(u)$. An example is an one-hot encoding of the labels $\ell(u)$ and $\ell(v)$. Alternatively,  $\hb_{v}^\tup{0}$ can be an attribute or a feature of the vertex $v$, e.g., physical measurements in the case of chemical molecules. An MPNN architecture consists of a stack of neural network layers, i.e., a composition of permutation-equivariant parameterized functions. Similarly to the \wlone, each layer aggregates local neighborhood information, i.e., the neighbors' features around each vertex, and then passes this aggregated information on to the next layer. Following,~\citet{Sca+2009} and~\citet{Gil+2017}, in each layer, $t > 0$,  we compute vertex features
\begin{equation}\label{def:MPNN}
	\hb_{v}^\tup{t} \coloneqq
	\UPD^\tup{t}\Bigl(\hb_{v}^\tup{t-1},\AGG^\tup{t} \bigl(\oms \hb_{u}^\tup{t-1}
	\mid u\in N(v) \cms \bigr)\Bigr) \in \Rb^{d},
\end{equation}
for each $v\in V(G)$, where  $\UPD^\tup{t}$ and $\AGG^\tup{t}$ may be differentiable parameterized functions, e.g., neural networks.\footnote{Strictly speaking, \citet{Gil+2017} consider a slightly more general setting in which vertex features are computed by $\hb_{v}^\tup{t} \coloneqq
		\UPD^\tup{t}\Bigl(\hb_{v}^\tup{t-1},\AGG^\tup{t} \bigl(\oms (\hb_v^\tup{t-1},\hb_{u}^\tup{t-1},\ell(v,u))
		\mid u\in N(v) \cms \bigr)\Bigr)$, where $\ell(v,u)$ denotes the edge label of the edge $(v,u)$.}
In the case of graph-level tasks, e.g., graph classification, one uses
\begin{equation}\label{def:readout}
	\hb_G \coloneqq \RO\bigl( \oms \hb_{v}^{\tup{L}}\mid v\in V(G) \cms \bigr) \in \Rb^d,
\end{equation}
to compute a single vectorial representation based on learned vertex features after iteration $L$. Again, $\RO$  may be a differentiable parameterized function. To adapt the parameters of the above three functions, they are optimized end-to-end, usually through a variant of stochastic gradient descent, e.g.,~\citet{Kin+2015}, together with the parameters of a neural network used for classification or regression.

\paragraph{More expressive  MPNNs} Since the expressive power of MPNNs is strictly limited by the \wlone{} in distinguishing non-isomorphic graphs~\citep{Mor+2019,Xu+2018b}, a large set of more expressive extensions of MPNNs~\citep{Mor+2022} exists. Here, we introduce the \new{\tMPNNF{} architecture}, an MPNN variant of the \wlonef{}; see~\cref{sec:wlf}. In essence, an \tMPNNF{} is a standard MPNN, where we set the initial features consistent with the initial vertex-labeling of the \wlonef, e.g., one-hot encodings of $\ell_{\mathcal{F}}$. Following~\citet{Mor+2019}, it is straightforward to derive an \tMPNNF{} architecture that has the same expressive power as the \wlonef{} in distinguishing non-isomorphic graphs.

\paragraph{Notation} In the subsequent sections, we use the following notation for MPNNs. We denote the class of all (labeled) graphs by $\cG$. For $d,l>0$, we denote the class of MPNNs using summation for aggregation, and such that update and readout functions are \new{multilayer perceptrons} (MLPs), all of a width of at most $d$, by $\MPNN_{\mathsf{mlp}}(d,L)$. We refer to elements in $\MPNN_{\mathsf{mlp}}(d,L)$ as \textit{simple MPNNs}; see~\cref{sec:simpleMPNNs} for details. We stress that simple MPNNs are already expressive enough to be equivalent to the \wlone{} in distinguishing non-isomorphic graphs~\citep{Mor+2019}. The class $\MPNN_{\mathsf{mlp},\cF}(d,L)$ is defined similarly, based on \tMPNNF s.

\section{When more expressivity matters---and when it does not}
We start by investigating under which conditions using more expressive power leads to better generalization performance and, when not, using the \emph{data's margin}. To this aim, we first prove lower and upper bounds on the VC dimension of \wlone-based kernels, MPNNs, and their more expressive generalizations.

\subsection{Margin-based bounds on the VC dimension of Weisfeiler--Leman kernels and MPNN architectures}\label{subsec:VCdim}

We first derive a general condition to prove margin-based lower and upper bounds. For a subset $\Sb\subseteq\Rb^d, d>0$,  we consider the following set of partial concepts from $\Sb$ to $\{0,1,\star\}$,
\begin{align*}
	\Hb_{r,\lambda}(\Sb) \coloneqq \Big\{ h \in \{0,1,\star \}^{\Sb}  \mathrel{\Big|}  \forall\, & \vec{x}_1, \dots, \vec{x}_s \in \mathsf{supp}(h) \colon                                                \\
	(                                                                                            & \vec{x}_1, h(\vec{x}_1)), \dotsc, (\vec{x}_s, h(\vec{x}_s)) \text{ is $(r,\lambda)$-separable} \Big\}.
\end{align*}
For the upper bound, since $\Sb\subseteq\Rb^d$, the VC dimension of $\Hb_{r,\lambda}(\Sb)$ is upper-bounded by the VC dimension of $\Hb_{r,\lambda}(\Rb^d)$. As already mentioned, the latter is known to be bounded by $\nicefrac{r^2}{\lambda^2}$ \citep{bartlett1999generalization, Alo+2021}. For the lower bound, the following lemma, implicit in~\citet{Alo+2021}, states sufficient conditions for $\Sb$ such that the VC dimension of $\Hb_{r,\lambda}(\Sb)$ is also lower-bounded by $\nicefrac{r^2}{\lambda^2}$.

\begin{lemma}\label{lem:VC_lower_F}
	Let  $\Sb\subseteq\Rb^d$. If $\Sb$ contains $m\coloneqq\lfloor\nicefrac{r^2}{\lambda^2}\rfloor$
	vectors $\vec b_1,\dotsc,\vec b_m \in \Rb^d$ with $\vec b_i:=(\vec b_i^{(1)}, \vec b_i^{(2)})$ and $\vec b_1^{(2)},\dotsc,\vec b_m^{(2)}$ being pairwise orthogonal, $\|\vec b_i\|=r'$, and $\|\vec b_i^{(2)}\|=r$,
	then
	\begin{equation*}
		\vcdim(\Hb_{r',\lambda}(\Sb)) \in \Theta(\nicefrac{r^2}{\lambda^2}).
	\end{equation*}
\end{lemma}
Next, we derive lower- and upper-bounds on the VC dimension of graphs separable by some graph embedding, e.g., the \wlone{} kernel. For $n,d>0$, let $\cE(n,d)$ be a class of graph embedding methods consisting of mappings from $\cG_n$ to $\Rb^d$, e.g., \wlone{} feature vectors. A (graph) sample $(G_1,y_1),\dotsc,(G_s,y_s) \in \cG_n\times\{0,1\}$ is \new{$(r,\lambda)$-$\cE(n,d)$-separable} if there is an embedding $\mathsf{emb}\in\cE(n,d)$ such that $(\mathsf{emb}(G_1),y_1),\dotsc,(\mathsf{emb}(G_s),y_s)\in \Rb^d\times\{0,1 \}$ is $(r,\lambda)$-separable, resulting in the set of partial concepts
\begin{align*}
	\Hb_{r,\lambda}(\cE(n,d)) \coloneqq \Big\{
	h\in\{0,1,\star\}^{\cG_n} \mathrel{\Big|}\,\forall\, & G_1,\dotsc,G_s\in\mathsf{supp}(h)\colon \\
	(                                                    & G_1,h(G_1)),\dotsc,(G_s,h(G_s))
	\text{ is $(r,\lambda)$-$\cE(n,d)$-separable} \Big\}.\nonumber
\end{align*}
Now, consider the subset
$\Sb(n,d)\coloneqq\{\mathsf{emb}(G)\in\Rb^d\mid G\in \cG_n,\mathsf{emb}\in\cE(n,d)\}$ of $\Rb^d$. It is clear that the VC dimension of $\Hb_{r,\lambda}(\cE(n,d))$ is equal to the VC dimension of
$\Hb_{r,\lambda}(\Sb(n,d))$, which in turn is upper-bounded by $\nicefrac{r^2}{\lambda^2}$. We next use \cref{lem:VC_lower_F} to obtain lower bounds on the VC dimension of  $\Hb_{r,\lambda}(\cE(n,d))$ for specific classes of embeddings.

\paragraph{\wlone-based embeddings} We first consider the class of graph embeddings obtained by the \wlone{} feature map after $T\geq0$ iterations, i.e.,
$\cE_{\mathsf{WL}}(n,d_T)\coloneqq \{G\mapsto  \phi^{(T)}_{\textsf{WL}}(G)\mid G\in\cG_n\}$ and its
normalized counterpart
$\widebar{\cE}_{\mathsf{WL}}(n,d_T)\coloneqq \{ G\mapsto \widebar{\phi^{(T)}_{\textsf{WL}}}(G) \mid G\in\cG_n \}$,
where $d_T$ is the dimension of the corresponding Hilbert space after $T$ rounds of \wlone; see~\cref{subsec:1WL} for details. The following result shows that the VC dimension of the normalized and unnormalized \wlone{} kernel can be lower- and upper-bounded in the margin $\lambda$, the number of iterations, and the number of vertices.

\begin{theorem}\label{thm:VCWL}
	For any $T, \lambda>0$, we have,
	\begin{align*}
		 & \vcdim(\Hb_{r,\lambda}(\cE_{\mathsf{WL}}(n,d_T))) \in{\Theta}(\nicefrac{r^2}{\lambda^2}), \text{ for } r=\sqrt{T+1}n  \text{ and } n\geq \nicefrac{r^2}{\lambda^2},          \\
		 & \vcdim(\Hb_{1,\lambda}(\widebar{\cE}_{\mathsf{WL}}(n,d_T)))\in{\Theta}(\nicefrac{1}{\lambda^2}), \text{ for } r=\sqrt{T/(T+1)}  \text{ and } n\geq\nicefrac{r^2}{\lambda^2}.
	\end{align*}
\end{theorem}
In~\cref{sec:colored_margins}, we also derive margin-based bounds for graphs with a finite number of colors under the \wlone, circumventing the dependence on the order in the above result.

Further, by defining $\cE_{\mathsf{WL}, \cF}(n,d_T)$, $\cE_{\mathsf{WLOA}}(n,d_T)$, and $\cE_{\mathsf{WLOA}, \cF}(n,d_T)$ analogously to the above, we can show the same or similar results for the \wlonef{}, \wloa{}, and \wloaf{}. The only difference is that  $\|\phi^{(t)}_{\textsf{WL}}(G_i)\| \neq \|\phi^{(t)}_{\textsf{WLOA}}(G_i)\|$ and thus the radii and bounds change slightly.
Concretely, for the \wlonef{}, we get an identical dependency on the margin $\lambda$, the number of iterations, and the number of vertices.
\begin{corollary}\label{thm:VCWLF}
	Let $\cF$ be a finite set of graphs.
	For any $T, \lambda>0$, we have,
	\begin{align*}
		 & \vcdim(\Hb_{r,\lambda}(\cE_{\mathsf{WL}, \cF}(n,d_T))) \in{\Theta}(\nicefrac{r^2}{\lambda^2}), \text{ for } r=\sqrt{T+1}n  \text{ and } n\geq \nicefrac{r^2}{\lambda^2}            \\
		 & \vcdim(\Hb_{1,\lambda}(\widebar{\cE}_{\mathsf{WL}, \cF}(n,d_T)))\in{\Theta}(\nicefrac{1}{\lambda^2}), \text{ for } r=\sqrt{T/(T+1)}  \text{ and }  n\geq\nicefrac{r^2}{\lambda^2}.
	\end{align*}
\end{corollary}
Similarly, by changing the radii from $\sqrt{T}n$ to $\sqrt{Tn}$, we get the following results for the \wloa{}.
\begin{proposition}\label{thm:VCWLOA}
	For any $T, \lambda>0$, we have,
	\begin{align*}
		 & \vcdim(\Hb_{r,\lambda}(\cE_{\mathsf{WLOA}}(n,d_T))) \in{\Theta}(\nicefrac{r^2}{\lambda^2}), \text{ for } r=\sqrt{(T+1)n}  \text{ and } n\geq \nicefrac{r^2}{\lambda^2},         \\
		 & \vcdim(\Hb_{1,\lambda}(\widebar{\cE}_{\mathsf{WLOA}}(n,d_T)))\in{\Theta}(\nicefrac{1}{\lambda^2}), \text{ for } r=\sqrt{T/(T+1)}  \text{ and }  n\geq\nicefrac{r^2}{\lambda^2}.
	\end{align*}
\end{proposition}
Moreover, we get an analogous result for the \wloaf{} kernel.
\begin{corollary}\label{thm:VCWLOAF}
	Let $\cF$ be a finite set of graphs. For any $T, \lambda>0$, we have,
	\begin{align*}
		 & \vcdim(\Hb_{r,\lambda}(\cE_{\mathsf{WLOA}, \cF}(n,d_T))) \in{\Theta}(\nicefrac{r^2}{\lambda^2}), \text{ for } r=\sqrt{(T+1)n}  \text{ and } n\geq \nicefrac{r^2}{\lambda^2},        \\
		 & \vcdim(\Hb_{1,\lambda}(\widebar{\cE}_{\mathsf{WLOA}, \cF}(n,d_T)))\in{\Theta}(\nicefrac{1}{\lambda^2}), \text{ for } r=\sqrt{T/(T+1)}  \text{ and } n\geq\nicefrac{r^2}{\lambda^2}.
	\end{align*}
\end{corollary}

Therefore, using $\cF$ permits the above statements to be feasible for smaller values of $n$ or $\lambda$.

\paragraph{Margin-based bounds on the VC dimension of MPNNs and more expressive architectures}
In the following, we lift the above results to MPNNs. Assume a fixed but arbitrary number of layers $T \geq 0$, vertices $n > 0$, and an embedding dimension $d>0$. In addition, we denote the following class of graph embeddings
\begin{equation*}
	\cE_{\textsf{MPNN}}(n,d,T) \coloneqq \{ G\mapsto m(G)   \mid G \in \cG_n \text{ and } m \in  \MPNN_{\mathsf{mlp}}(d,T) \},
\end{equation*}
i.e., the set of $d$-dimensional vectors computable by simple $T$-layer MPNNs over the set of $n$-order graphs. Now, the following result lifts~\cref{thm:VCWL} to MPNNs.
\begin{proposition}\label{prop:matchingvc_mpnn}
	For any $n, T>0$ and sufficiently large $d>0$, we have,
	\begin{align*}
		 & \vcdim(\mathbb{H}_{r,\lambda}(\cE_{\textsf{MPNN}}(n,d,T))) \in {\Theta}(\nicefrac{r^2}{\lambda^2}), \text{ for } r=\sqrt{T+1}n  \text{ and } n\geq \nicefrac{r^2}{\lambda^2},         \\
		 & \vcdim(\mathbb{H}_{1,\lambda}(\cE_{\textsf{MPNN}}(n,d,T))) \in{\Theta}(\nicefrac{1}{\lambda^2}),\text{ for } r=\sqrt{T/(T+1)} \text{ and } n\geq\nicefrac{r^2}{\lambda^2}.\tag*{\qed}
	\end{align*}
\end{proposition}
Moreover, we can lift~\cref{thm:VCWLF} to \MPNNF{} architectures by defining $\cE_{\textsf{MPNN},\cF}(n,d,T)$ analogously to the above.
\begin{corollary}\label{prop:matchingvc_mpnn_f}
	Let $\cF$ be a finite set of graphs. For any $n, T>0$ and sufficiently large $d>0$, we have,
	\begin{align*}
		 & \vcdim(\mathbb{H}_{r,\lambda} (\cE_{\textsf{MPNN},\cF}(n,d,T))) \in{\Theta}(\nicefrac{r^2}{\lambda^2}),\text{ for } r=\sqrt{T+1}n  \text{ and } n\geq \nicefrac{r^2}{\lambda^2},             \\
		 & \vcdim(\mathbb{H}_{1,\lambda} (\cE_{\textsf{MPNN},\cF}(n,d,T))) \in{\Theta}(\nicefrac{1}{\lambda^2}), \text{ for } r=\sqrt{T/(T+1)} \text{ and } n\geq\nicefrac{r^2}{\lambda^2}. \tag*{\qed}
	\end{align*}
\end{corollary}
We can also lift the results to the MPNN versions of the \wloa{} and \wloaf; see the appendix for details. These results are somewhat restrictive since we only consider MPNNs that behave like linear classifiers by definition of the considered functions. The above implies that the upper bound does not hold for general MPNNs since they can separate non-linearly separable data under mild conditions.

\paragraph{Implications of the results}
Previous lower and upper bounds only considered the feature space's dimensionality, implying worse generalization performance for more expressive variants of the \wlone, not aligned with empirical results, e.g.,~\cite{Bou+2020}. Our results show that more expressive power only sometimes results in worse generalization properties. Hence, a more fine-grained analysis is needed to understand when more expressive power, e.g., through the \wlonef{} or \wloaf{}, improves generalization performance.
For example, if more expressive power makes the data linearly separable, leading to a positive margin, or increases the margin, our results imply improved generalization performance. In fact, in the following, we leverage our results to understand when more expressivity leads to linear separability and an increased margin. Hence, our results indicate that the data's margin can be used as a yardstick to assess the generalization properties of Weisfeiler--Leman-based kernels, MPNNs, and their more expressive variants in a more fine-grained and data-dependent manner.

\subsection{Examples of when more power separates the data}
We next aim to understand when \wlone's more expressive variants, such as the \wlonef{}, can linearly separate the data, resulting in a positive margin. Therefore, we first derive data distributions where the \wlone{} kernel cannot separate the data points. In contrast, in the case of normalized feature vectors, the \wlonef{} separates them with the largest possible margin.
\begin{proposition}\label{thm:seperator}
	For every $n \geq 6$, there exists a pair of non-isomorphic $n$-order graphs $(G_n$, $H_n)$ and a graph $F$
	such that, for $\cF\coloneqq\{F\}$ and for all number of rounds $T \geq 0$, it holds that
	\begin{align*}
		\norm[\bigg]{\widebar{\phi^{(T)}_{\textsf{WL}}}(G_n) - \widebar{\phi^{(T)}_{\textsf{WL}}}(H_n)} = 0, \quad\text{ and }\quad
		\norm[\bigg]{\widebar{\phi^{(T)}_{\textsf{WL}, \cF}}(G_n) - \widebar{\phi^{(T)}_{\textsf{WL},\cF}}(H_n)} = \sqrt{2}.\tag*{\qed}
	\end{align*}
\end{proposition}
Moreover, we can also lift the above result to MPNN and \tMPNNF{} architectures.
\begin{proposition}\label{thm:seperator_mpnn}
	For every $n \geq 6$, there exists a pair of non-isomorphic $n$-order graphs $(G_n$, $H_n)$ and
	a graph $F$,
	such that, for all number of layers $T \geq 0$, and widths $d > 0$, and all $m \in  \MPNN_{\mathsf{mlp}}(d,T)$,
	it holds that
	\begin{align*}
		\norm[\bigg]{m(G_n) - m(H_n)} = 0,
	\end{align*}
	while for sufficiently large $d>0$, there exists an $\widehat m \in  \MPNN_{\mathsf{mlp},\cF}(d,T)$, for $\cF\coloneqq\{F\}$, such that
	\begin{align*}
		\norm[\bigg]{\widehat{m}(G_n) - \widehat{m}(H_n)} = \sqrt{2}. \tag*{\qed}
	\end{align*}
\end{proposition}
However, more than \emph{merely distinguishing} the graphs based on their structure is often required. The following result shows that data distributions exist such that the \wlone{} kernel can perfectly separate each pair of non-isomorphic graphs while \new{is unable to separate the data linearly}. In fact, the construction implies data distributions where the \wlone{} kernel cannot do better than random guessing on the test set, which we also empirically verify in~\cref{sec:experiments}.
\begin{proposition}\label{thm:separability}
	For every $n \geq 10$, there exists a set of pair-wise non-isomorphic (at most) $n$-order graphs $S$, a concept $c \colon S \to \{ 0,1 \}$, and
	a graph $F$, such that the graphs in the set $S$,
	\begin{enumerate}
		\item are pair-wise distinguishable by \wlone{} after one round,
		\item are \emph{not} linearly separable under the normalized \wlone{} feature vector $\widebar{\phi^{(T)}_{\textsf{WL}}}$, concerning the concept $c$, for any $T \geq 0$,
		\item and are linearly separable under the normalized \wlonef{} feature vector $\widebar{\phi^{(T)}_{\textsf{WL},\cF}}$,  concerning the concept $c$ and, for all $T \geq 0$, where $\cF\coloneqq \{F\}$.
	\end{enumerate}
	Moreover, the results also work for the unnormalized feature vectors.\qed
\end{proposition}
We can derive more general results when placing stronger conditions on the data distribution.
\begin{proposition}\label{thm:fsep}
	Let $n \geq 6$  and let $\cF$ be a finite set of graphs. Further, let $c \colon \cG_n \to \{ 0,1 \}$ be a concept such that, for all $T \geq 0$, the graphs are \emph{not} linearly separable under the normalized \wlone{} feature vector $\widebar{\phi^{(T)}_{\textsf{WL}}}$, concerning the concept $c$. Further, assume that for all graphs $G \in \cG_n$ for which $c(G) = 0$, it holds that there is at least one vertex $v \in V(G)$ such it is contained in a subgraph of $G$ that is isomorphic to a graph in the set $\cF$, while no such vertices exist in graphs $G$ for which $c(G) = 1$.  Then the graphs are linearly separable under the normalized \wlonef{} feature vector $\widebar{\phi^{(T)}_{\textsf{WL},\cF}}$, concerning the concept $c$.\qed
\end{proposition}
Hence, the above results imply that expressive kernel and neural architectures such as the \wlonef{} and \MPNNF{} help make the data separable, creating a positive margin and implying, by~\cref{subsec:VCdim}, improved generalization performance.  In addition, the results imply that the Weisfeiler--Leman-based graph isomorphism perspective is too simplistic to understand the generalization properties of such kernel and MPNN architectures.

\subsection{Examples of when more power shrinks the margin}
While the previous results showed that more expressive power can make the data linearly separable and improve generalization performance, adding expressive power might also decrease the data's margin. The following result shows that data distributions exist such that more expressive power leads to a smaller margin, implying, by~\cref{subsec:VCdim}, a worsened generalization performance.
\begin{proposition}\label{thm:shrink_with_more_power}
	For every $n \geq 10$, there exists a pair of $2n$-order graphs $(G_{n}, H_{n})$ and a graph $F$,
	such that, for $\cF\coloneqq\{F\}$ and for all number of rounds $T > 0$, it holds that
	\begin{align*}
		\norm[\bigg]{\widebar{\phi^{(T)}_{\textsf{WL}, \cF}}(G_n) - \widebar{\phi^{(T)}_{\textsf{WL},\cF}}(H_n)} < \norm[\bigg]{\widebar{\phi^{(T)}_{\textsf{WL}}}(G_n) - \widebar{\phi^{(T)}_{\textsf{WL}}}(H_n)}.\tag*{\qed}
	\end{align*}
\end{proposition}
The above results easily generalize to other \wlonef{}-based kernels, i.e., more expressive power does not always result in increased generalization performance. Hence, in the following subsection, we derive precise conditions when more expressive power provably leads to better generalization performance.

\subsection{When more power grows the margin}\label{subsec:morepowergrows}

Here, we study when adding expressive power provably leads to improved generalization performance. We start with the \wlonef{} and then move to the \wloaf{}, where more interesting results can be shown.
\paragraph{The \wlonef{} kernel} The following result shows that, under some assumptions, data distributions exist such that the \wlonef{} kernel always leads to a larger margin than the \wlone{} kernel.
\begin{proposition}\label{thm:grow}
	Let $n > 0$ and $G_n$ and $H_n$ be two \emph{connected} $n$-order graphs. Further, let $\cF \coloneqq \{ F \}$
	such that there is at least one vertex in $V(G_n)$ contained in a subgraph of $G_n$ isomorphic to the graph $F$. For the graph $H_n$, no such vertices exist. Further, let $T \geq 0$ be the number of rounds to reach the stable partition of $G_n$ and $H_n$ under \wlone, and assume
	$$
		(\phi^{(T)}_{\textsf{WL}}(G_n), 1), (\phi^{(T)}_{\textsf{WL}}(H_n), 0) \text{ is } (r_1, \lambda_1)\text{-separable, with margin } \lambda_1<\sqrt{2n}.
	$$
	Then,
	\begin{equation}
		\!\!\!(\phi^{(T)}_{\textsf{WL}, \cF}(G_n), 1), (\phi^{(T)}_{\textsf{WL},\cF}(H_n), 0) \text{ is } (r_2, \lambda_2)\text{-separable, with margin } r_2\leq r_1 \text{ and } \lambda_2\geq\lambda_1. \tag*{\qed}
	\end{equation}
\end{proposition}
Hence, in terms of generalization properties, we observe that $\nicefrac{r_1^2}{\lambda_1^2}\geq \nicefrac{r_2^2}{\lambda_2^2}$ and hence we obtain lower margin-based bounds by using $\cF$.

\paragraph{The \wloaf{} kernel} It is challenging to improve~\cref{thm:grow}, i.e., to derive weaker conditions such that \wlonef{} provably leads to an increase of the margin over the \wlone{} kernel. This becomes more feasible, however, for the \wloaf{} kernel. First, note that for two graphs $G$ and $H$, it holds that
\begin{align}\label{eq:unaryWL}
	k_{\textsf{WLOA}}(G,H) & \coloneqq \sum_{t \in [T]\cup\{0\}} \sum_{c \in \Sigma_t}\nonumber \min(\phi_t(G)_c,\phi_t(H)_c)                                        \\
	                       & = \sum_{t \in [T]\cup\{0\}} \sum_{c \in \Sigma_t} \sum^n_{j=1}\mathbbm{1}_{\phi_t(G)_c\geq j \land \phi_t(H)_c\geq j}                   \\\nonumber
	                       & = \sum_{t \in [T]\cup\{0\}} \sum_{c \in \Sigma_t} \sum^n_{j=1} \mathbbm{1}_{\phi_t(G)_c\geq j}\mathbbm{1}_{\phi_t(H)_c\geq j}.\nonumber
\end{align}
\Cref{eq:unaryWL} provides an intuition for the feature map of the kernel $k_{\textsf{WLOA}}$, namely, a unary encoding of the count for each color in each iteration. Such a feature map is natural as the inner product between unary encodings of $a$ and $b$ is the minimum of $a$ and $b$. This also implies that, for a graph $G\in \mathcal G_n$, it holds that $\|\phi^{(T)}_{\textsf{WLOA}}(G)\| = \sqrt{Tn}$, i.e., we can easily bound the norm of the feature vector.

Now, the \wloa{} simplifies the computation of distances between two graphs $G, H \in \mathcal G_n$, since
\begin{align*}
	\left\lVert\phi^{(T)}_{\textsf{WLOA}}(G)-\phi^{(T)}_{\textsf{WLOA}}(H)\right\rVert & =\sqrt{\left\lVert\phi^{(T)}_{\textsf{WLOA}}(G)\right\rVert^2+\left\lVert\phi^{(T)}_{\textsf{WLOA}}(H)\right\rVert^2-2\phi^{(T)}_{\textsf{WLOA}}(G)^\tran\phi^{(T)}_{\textsf{WLOA}}(H)} \\
	                                                                                   & =\sqrt{2Tn-2k^{(T)}_{\textsf{WLOA}}(G,H)}.
\end{align*}
Therefore, due to the monotonicity of the square root, margin increases are directly controlled by the kernel value $k^{(T)}_{\textsf{WLOA}}(G,H)$.
For the \wloaf{}, since by~\cref{prop:WLFfiner} the \wlonef{} computes a finer color partition than the \wlone{}, pairwise distances can not decrease compared to the \wloa, resulting in the following statement.
\begin{proposition}\label{prop:pairwise}
	Let $\cF$ be a finite set of graphs. Given two graphs $G$ and $H$,
	\begin{equation}
		\left\lVert\phi^{(T)}_{\textsf{WLOA},\cF}(G)-\phi^{(T)}_{\textsf{WLOA},\cF}(H)\right\rVert\geq\left\lVert\phi^{(T)}_{\textsf{WLOA}}(G)-\phi^{(T)}_{\textsf{WLOA}}(H)\right\rVert.\tag*{\qed}
	\end{equation}
\end{proposition}
The above result motivates considering the margin as a linear combination of pairwise distances, leading to the following statement.
\begin{theorem}\label{thm:OAmargin}
	Let $(G_1,y_1),\dotsc,(G_s,y_s)$ in $\cG_n\times\{0,1\}$ be a (graph) sample that is linearly separable in the \wloa{} feature space with margin $\gamma$. If
	\begin{equation}
		\begin{split}
			\min_{y_i\neq y_j} & \left\lVert\phi^{(T)}_{\textsf{WLOA},\cF}(G_i)-\phi^{(T)}_{\textsf{WLOA},\cF}(G_j)\right\rVert^2 - \left\lVert\phi^{(T)}_{\textsf{WLOA}}(G_i)-\phi^{(T)}_{\textsf{WLOA}}(G_j)\right\rVert^2 >                           \\
			\max_{y_i=y_j}     & \left\lVert\phi^{(T)}_{\textsf{WLOA},\cF}(G_i)-\phi^{(T)}_{\textsf{WLOA},\cF}(G_j)\right\rVert^2 - \left\lVert\phi^{(T)}_{\textsf{WLOA}}(G_i)-\phi^{(T)}_{\textsf{WLOA}}(G_j)\right\rVert^2,\label{eq:OAAssumption}
		\end{split}
	\end{equation}
	that is, if the minimum increase in distances between classes is strictly larger than the maximum increase in distance within each class, then the margin $\lambda$ increases when $\cF$ is considered. \qed
\end{theorem}
This statement is, in fact, also valid for the \wlonef{}. However, developing meaningful conditions on the graph structure so that \cref{eq:OAAssumption} is valid is more challenging. For the \wloa, the following results derive conditions guaranteeing that \cref{eq:OAAssumption} is fulfilled.
\begin{theorem}\label{thm:wloa}
	Let $G$ and $H$ be $n$-order graphs and let $\cF$ be a finite set of graphs, $T \geq 0$, and $C_\cF(c)$ be the set of colors that color $c$ under \wlone{} is split into under \wlonef{}, i.e., $\phi_{t}(G)_c = \sum_{c'\in C_\cF(c)}\phi_{\cF, t}(G)_{c'}.$ Then the following statements are equivalent,
	\begin{enumerate}
		\item
		      $
			      \left\lVert\phi^{(T)}_{\textsf{WLOA},\cF}(G)-\phi^{(T)}_{\textsf{WLOA},\cF}(H)\right\rVert = \left\lVert\phi^{(T)}_{\textsf{WLOA}}(G)-\phi^{(T)}_{\textsf{WLOA}}(H)\right\rVert.
		      $
		\item
		      $
			      \forall\, t\in[T]\cup\{0\}\, \forall\, c\in \Sigma_t \colon \phi_t(G)_{c}\geq \phi_t(H)_c \iff \forall\, c'\in C_\cF(c) \colon \phi_{\cF, t}(G)_{c'} \geq \phi_{\cF,t}(H)_{c'}.
		      $ \qed
	\end{enumerate}
\end{theorem}
Hence, we can easily derive conditions under which the \wloaf{} leads to a strict margin increase. That is, we get a strict increase in distances if, and only, if the \wlonef{} splits up a color $c$ under \wlone{} such that the occurrences of this color $c$ are larger or equal in one graph over the other while for at least one of the resulting colors under \wlonef{}, refining the color $c$, the relation is strictly reversed.
\begin{corollary}\label{cor:wloa}
	Let $G$ and $H$ be $n$-order graphs and let $\cF$ be a finite set of graphs and let $T \geq 0$, and $C_\cF(c)$ be the set of colors that color $c$ under \wlone{} is split into under \wlonef{}, i.e., $\phi_{t}(G)_c = \sum_{c'\in C_\cF(c)}\phi_{\cF, t}(G)_{c'}.$ The following statements are equivalent
	\begin{enumerate}
		\item
		      $
			      \left\lVert\phi^{(T)}_{\textsf{WLOA},\cF}(G)-\phi^{(T)}_{\textsf{WLOA},\cF}(H)\right\rVert > \left\lVert\phi^{(T)}_{\textsf{WLOA}}(G)-\phi^{(T)}_{\textsf{WLOA}}(H)\right\rVert.
		      $
		\item
		      $
			      \exists\, t\in[T]\cup\{0\}\, \exists\, c\in \Sigma_t\ \colon \neg\left(\phi_t(G)_{c}\geq \phi_t(H)_c \iff \forall\, c'\in C_\cF(c) \colon \phi_{\cF, t}(G)_{c'} \geq \phi_{\cF,t}(H)_{c'}\right).
		      $ \qed
	\end{enumerate}
\end{corollary}
Specifically, this implies that using \cref{thm:wloa} and \cref{cor:wloa} as assumptions on the distances between graphs within one class and between two classes, respectively, implies a margin increase via \cref{thm:OAmargin}. Then \cref{thm:VCWLOA} and \cref{thm:VCWLOAF} imply a decrease in VC-dimension and consequently an increase in generalization performance when using $\cF$. See \cref{APP:cor:example} in the appendix for an example where the above conditions are met.

\section{Large margins and gradient flow}\label{sec:sgd}

\cref{prop:matchingvc_mpnn,prop:matchingvc_mpnn_f} ensure the existence of parameter assignment such that MPNN and \tMPNNF{} architectures generalize. However, it remains unclear how to find them. Hence, building on the results in~\citet{JiT19}, we now show that, under some assumptions, MPNNs exhibit an ``alignment'' property whereby gradient flow pushes the network's weights toward the maximum margin solution.

\paragraph{Formal setup}
We consider MPNNs following \cref{def:MPNN} and consider graph classification tasks using a readout layer. We make some simplifying assumptions and consider \emph{linear} MPNNs. That is, we set the aggregation function $\AGG$ to summation, and $\UPD$ at layer $i$ is summation followed by a dense layer with trainable weight matrix $\vec{W}^{(i)} \in \Rb^{d_{i}\times d_{i-1}}$. Let $G$ be an $n$-order graph, if we pack the node embeddings $\mathbf{h}_v^{(i)}$ into an $d_i \times n$ matrix $\vec X^{(i)}$ whose $v^\text{th}$ column is $\mathbf{h}_v^{(i)}$, then
\[
	\vec{X}^{(i+1)} = \vec{W}^{(i+1)} \vec{X}^{(i)} \vec{A}'(G),
\]
where $\vec{A}'(G) \coloneqq \vec{A}(G)+\vec{I}_n$, $\vec{I}_n\in\Rb^{n\times n}$ is the $n$-dimensional identity matrix, and $\vec{X} = \vec{X}^{(0)}$ is the $d_0 \times n$ matrix whose columns correspond to vertices' initial features; we also write $d = d_0$. For the permutation-invariant readout layer, we use simple summation of the final node embeddings and assume that $\vec{X}^{(L)}$ is transformed into a prediction $\hat{y}$ as follows,
\[
	\hat{y} = \RO\bigl( \vec X^{(L)} \bigr) = \vec{X}^{(L)} \cdot \vec{1}_n,
\]
where $ \vec{1}_n$ is the $n$-element all-one vector. Note that since we desire a scalar output, we will have $d_L = 1$.

Suppose our training dataset is $\{(G_i, \vec X_i, y_i)\}_{i=1}^k$, where $\vec X_i \in \Rb^{d \times n_i}$ is a set of $d$-dimensional node features over an $n_i$-order graph $G_i$ with $|V(G_i)| = n_i$, and $y_i \in \{-1, +1\}$ for all $i$. We use a loss function $\ell$ with the following assumption.

\begin{restatable}{ass}{lossassumption} \label{assumption:loss}
	The loss function $\ell \colon \Rb \to \Rb^+$ has a continuous derivative $\ell'$ such that $\ell'(x) < 0$ for all $x$, $\lim_{x\to-\infty} \ell(x) = \infty$, and $\lim_{x\to\infty} \ell(x) = 0$. \qed
\end{restatable}

The empirical risk induced by the MPNN is
\begin{align*}
	\cR(\vec W^{(L)}, \dotsc, \vec W^{(1)}) & = \frac{1}{k} \sum_{i=1}^k \ell(y_i, \hat{y}_i)                               \\
	                                        & = \frac{1}{k} \sum_{i=1}^k \ell(\wprod \vec Z_i \vec{A}'(G)^L \vec{1}_{n_i}),
\end{align*}
where $\wprod = \vec W^{(L)} \vec W^{(L-1)} \cdots \vec W^{(1)}$, and $\vec Z_i = y_i \vec X_i$.

We consider gradient flow. In gradient flow, the evolution of $\vec W = (\wmat^{(L)}, \wmat^{(L-1)}, \dots, \wmat^{(1)})$ is given by $\{\vec W(t) \colon t\geq 0\}$, where there is an initial state $\vec W(0)$ at $t=0$, and
\[
	\frac{d\vec W(t)}{dt} = -\nabla \cR(\vec W(t)).
\]
We make one additional assumption on the initialization of the network.
\begin{restatable}{ass}{initassumption} \label{assumption:risk}
	The initialization of $\wmat$ at $t=0$ satisfies $\nabla\cR(\wmat(0)) \neq \cR(0) = \ell(0)$. \qed
\end{restatable}

\paragraph{Alignment Theorems}
We now assume the data is MPNN-separable, i.e., there is a set of weights that correctly classifies every data point. More specifically, assume there is a vector $\bar{\vec{u}} \in \Rb^d$ such that $y_i \cdot \bar{\vec{u}}^\tran \vec X_i \vec{A}'(G_i)^L \vec{1}_{n_i} > 0$ for all $i$. Furthermore, the maximum margin is given by
\[
	\gamma = \max_{\|\bar{\vec{u}} \| = 1} \min_{1\leq i\leq k} y_i \cdot \bar{\vec u}^\tran \vec  X_i \vec{A}'(G_i)^L \vec{1}_{n_i} > 0,
\]
while the corresponding solution $\bar{\vec u} \in \Rb^d$ is given by
\[
	\argmax_{\|\bar{\vec u}\| = 1} \min_{1\leq i\leq k} y_i \cdot \bar{\vec u}^\tran \vec X_i \vec{A}'(G_i)^L \vec{1}_{n_i}.
\]
Furthermore, those $\vec{v}_i = \vec{Z}_i \vec{A}'(G_i)^L \vec{1}_{n_i}$ for which $\langle \bar{\vec{u}}, \vec{v}_i \rangle = \gamma$ are called \emph{support vectors}.

Our first main result shows that under gradient flow, the trainable weight vectors of our MPNN architecture get ``aligned.''
\begin{restatable}{theorem}{mainalignmentgradflow} \label{thm:main-alignment-gradflow}
	Suppose \cref{assumption:loss} and \cref{assumption:risk} hold.
	Let $\vec u_i(t)\in \Rb^{d_i}$ and $\vec v_i(t)\in \Rb^{d_{i-1}}$ denote the left and right singular vectors, respectively, of $\vec W^{(i)}(t) \in \Rb^{d_i \times d_{i-1}}$. Then, we have the following using the Frobenius norm $\|\!\cdot\!\|_F$:
	\begin{itemize}
		\item For $j = 1, 2, \dots, L$, we have
		      \[
			      \lim_{t\to\infty} \left\| \frac{\wj(t)}{\|\wj(t)\|_F} - \vec u_j(t) \vec v_j(t)^\tran \right\|_F = 0.
		      \]
		\item Also,
		      \[
			      \lim_{t\to\infty} \left| \left\langle \frac{(\vec W^{(L)}(t) \cdots \vec W^{(1)}(t))^\tran}{\prod_{j=1}^L \|\wj(t)\|_F} , \vec{v}_1 \right\rangle \right| = 1.
		      \]
	\end{itemize}
\end{restatable}

Furthermore, we can show that under mild assumptions, the trainable weights converge to the maximum margin solution $\bar{\vec u}$.

\begin{restatable}{ass}{supportassumption} \label{assumption:supportspan}
	The support vectors $\vec{v}_i = \vec{Z}_i \vec{A}'(G_i)^L \vec{1}_{n_i}$ span $\Rb^d$.
\end{restatable}
Note that, for unlabeled graphs, due to separability, the above assumption is trivially fulfilled.
\begin{restatable}[Convergence to the maximum margin solution]{theorem}{mainmaxmargin} \label{thm:main-maxmargin}
	Suppose \cref{assumption:loss} and \cref{assumption:supportspan} hold. Then, for the exponential loss function $\ell(x) = e^{-x}$, under gradient flow, we have that the learned weights of the MPNN converge to the maximum margin solution, i.e.,
	\[
		\lim_{t\to\infty} \frac{\vec W^{(L)}(t) \vec W^{(L-1)}(t) \cdots \vec W^{(1)}(t)}{\|\vec W^{(L)}(t)\|_F \|\vec W^{(L-1)}(t)\|_F \cdots \|\vec W^{(1)}(t)\|_F} =  \bar{\vec u}.
	\]
\end{restatable}
The results can be straightforwardly adjusted to \tMPNNF{} architectures.

\section{Limitations, possible road maps, and future work}

While our findings represent the first explicit link between a dataset's margin and the expressive power of an architecture, several key questions remain unanswered. While the empirical results of~\Cref{sec:experiments} suggest that our VC dimension bounds are practically applicable, it is mostly unclear under which conditions variants of stochastic gradient descent converge to a large margin solution for over-parameterized MPNN architectures. While we show convergence to the maximum solution for linear MPNNs, it is still being determined for which kind of non-linear activations similar results hold. Additionally, the role of an architecture's expressive power in this convergence process is poorly understood. Secondly, our bounds lack explicit information about graph structure. As a result, future research should explore incorporating graph-specific parameters or developing refined results tailored to relevant graph classes, such as tree, planar, or bipartite graphs.

\section{Experimental evaluation}
\label{sec:experiments}

In the following, we investigate to what extent our theoretical results translate into practice. Specifically, we answer the following questions.
\begin{description}
	\item[Q1] Does adding expressive power make datasets more linearly separable?
	\item[Q2] Can the increased generalization performance of a more expressive variant of the \wlone{} algorithm be explained by an increased margin?
	\item[Q3] Does the \wloaf{} lead to increased predictive performance?
	\item[Q4] Do the results lift to MPNNs?
\end{description}
The source code of all methods and evaluation procedures is available at \url{https://www.github.com/chrsmrrs/wl\_vc_expressivity}.

\begin{table}[t]
	\caption{Experimental validation of~\cref{thm:separability} for different numbers of vertices ($n$), reporting mean test accuracies and margins. \textsc{Nls}---Not linearly separable. \textsc{Dnc}---Did not compute due to implicit kernel.
	}
	\label{fig:separability}
	\centering

	\resizebox{.70\textwidth}{!}{ 	\renewcommand{\arraystretch}{1.05}
		\begin{tabular}{@{}lcccc@{}} \toprule
			\multirow{3}{*}{\vspace*{4pt}\textbf{Algorithm}} & \multicolumn{4}{c}{\textbf{Number of vertices} ($n$).}                                                                                                                                                                                                                                                                            \\\cmidrule{2-5}
			                                                 & 16                                                                             & 32                                                                            & 64                                                                             & 128                                                                             \\ \toprule
			\wlone                                           & 46.6 {\scriptsize $\pm 1.1$} \textsc{Nls}                                      & 47.3 {\scriptsize $\pm 1.7$} \textsc{Nls}                                     & 47.1  {\scriptsize $\pm 1.1$} \textsc{Nls}                                     & 46.5 {\scriptsize $\pm 0.7$} \textsc{Nls}                                       \\
			\wloa                                            & 36.8   {\scriptsize $\pm 1.3$} \textsc{Dnc}                                    & 37.4  {\scriptsize $\pm 1.7$} \textsc{Dnc}                                    & 37.3  {\scriptsize $\pm 0.9$} \textsc{Dnc}                                     & 37.8 {\scriptsize $\pm 1.2$} \textsc{Dnc}                                       \\

			\textsf{MPNN}                                    & 47.9  {\scriptsize $\pm 0.7$} \textsc{Dnc}                                     & 49.0  {\scriptsize $\pm 1.6$} \textsc{Dnc}                                    & 48.3  {\scriptsize $\pm 1.3$} \textsc{Dnc}                                     & 47.9  {\scriptsize $\pm 1.6$} \textsc{Dnc}                                      \\
			\cmidrule{1-5}
			\wlonef                                          & \textbf{100.0} {\scriptsize $\pm 0.0$} \textbf{0.006}  {\scriptsize $<0.0001$} & \textbf{100.0} {\scriptsize $\pm 0.0$} \textbf{0.014} {\scriptsize $<0.0001$} & \textbf{100.0} {\scriptsize $\pm 0.0$} \textbf{0.030}  {\scriptsize $<0.0001$} & \textbf{100.0} { \scriptsize $\pm 0.0$}  \textbf{0.062} {\scriptsize $<0.0001$} \\
			\wloaf                                           & 100.0 {\scriptsize $\pm 0.0$} \textsc{Dnc}                                     & 100.0 {\scriptsize $\pm 0.0$} \textsc{Dnc}                                    & 100.0  {\scriptsize $\pm 0.0$} \textsc{Dnc}                                    & 100.0 {\scriptsize $\pm 0.0$} \textsc{Dnc}                                      \\
			$\textsf{MPNN}_{\mathcal{F}}$                    & 100.0  {\scriptsize $\pm $0.1} \textsc{Dnc}                                    & 100.0  {\scriptsize $\pm 0.1$} \textsc{Dnc}                                   & 100.0 {\scriptsize $<0.1$} \textsc{Dnc}                                        & 100.0  {\scriptsize $\pm 0.1$} \textsc{Dnc}                                     \\
			\bottomrule
		\end{tabular}
	}
\end{table}
\begin{table}
	\caption{Mean train, test accuracies, and margins of the kernel architectures on \textsc{TUDatasets} datasets for different subgraphs. \textsc{Dnc}---Did not compute due to implicit kernel.}
	\label{fig:tud}
	\centering
	\resizebox{1.0\textwidth}{!}{ 	\renewcommand{\arraystretch}{1.05}
		\begin{tabular}{@{}l <{\enspace}@{}lccccc@{}} \toprule
			\multirow{3}{*}{\vspace*{4pt} $\mathcal{F}$ }                                 & \multirow{3}{*}{\vspace*{4pt}\textbf{Algorithm}}                         & \multicolumn{5}{c}{\textbf{Dataset}}                                                                                                                                                                                                                                                                                                                                                                                                                                                                                                       \\\cmidrule{3-7}
			                                                                              &                                                                          & {\textsc{Enzymes}}                                                                                        & {\textsc{Mutag}}                                                                                             & {\textsc{Proteins}}                                                                                   &
			{\textsc{PTC\_FM}}                                                            & {\textsc{PTC\_MR}}                                                                                                                                                                                                                                                                                                                                                                                                                                                                                                                                                                                                    \\	\toprule
			---                                                                           & \wlone                                                                   & 90.4  {\scriptsize $\pm 4.4$} 34.1   {\scriptsize $\pm 1.7$} 0.023  {\scriptsize $\pm 0.002$}             & 88.9  {\scriptsize $\pm 1.4$}  83.7  {\scriptsize $\pm 2.1 $}  0.073 {\scriptsize $\pm 0.044$}               & 79.9      {\scriptsize $\pm 1.8$} 68.1  {\scriptsize $\pm 1.1$}  0.090  {\scriptsize $\pm  0.0189$}   & 70.1    {\scriptsize $\pm 3.4$} 55.7  {\scriptsize $\pm 2.7$}  0.196 {\scriptsize $\pm  0.104$}   & 67.0 {\scriptsize $\pm 2.2$}  54.2  {\scriptsize $\pm  2.2$}  0.296  {\scriptsize $\pm 0.194 $}     \\
			---                                                                           & \wloa                                                                    & 100.0 0.0 {\scriptsize $\pm 0.0$} 32.3   {\scriptsize $\pm 1.7$}  \textsc{Dnc}                            & 99.3 {\scriptsize $\pm 1.4$} 82.6  {\scriptsize $\pm 2.0$}  \textsc{Dnc}                                     & 96.6
			{\scriptsize $\pm 2.7$}  \textbf{73.9}  {\scriptsize $\pm 0.7$}  \textsc{Dnc} & 91.7 {\scriptsize $\pm 4.2$} 58.2  {\scriptsize $\pm 1.5$}  \textsc{Dnc} & 95.5 {\scriptsize $\pm 4.4$} 55.7 {\scriptsize $\pm 1.5$}  \textsc{Dnc}                                                                                                                                                                                                                                                                                                                                                                                                                                                                    \\
			\cmidrule{1-7}
			$C_3$                                                                         & \wlonef                                                                  & 97.0    {\scriptsize $\pm 1.8$}  37.9   {\scriptsize $\pm 1.8$}  0.021   {\scriptsize $\pm 0.002$}        & 88.1   {\scriptsize $\pm 2.4$}  83.3   {\scriptsize $\pm 2.0$}  0.134   {\scriptsize $\pm  0.082$}           & 89.1   {\scriptsize $\pm 2.3$} 65.3  {\scriptsize $\pm 1.1$}   0.078   {\scriptsize $\pm  0.090$}     & 72.1    {\scriptsize $\pm 3.1$} 57.0  {\scriptsize $\pm 2.2$}   0.167   {\scriptsize $\pm 0.070$} & 70.9    {\scriptsize $\pm 3.6$}  54.7  {\scriptsize $\pm 2.3$} 0.143   {\scriptsize $\pm 0.044 $}   \\
			$C_3$-$C_4$                                                                   & \wlonef                                                                  & 97.5    {\scriptsize $\pm 1.1$} \textbf{40.6}   {\scriptsize $\pm 1.7$} 0.020 {\scriptsize $\pm  0.001 $} & 88.6   {\scriptsize $\pm 1.2$} \textbf{84.7}   {\scriptsize $\pm 1.9$}  0.065    {\scriptsize $\pm 0.044$}   & 90.5    {\scriptsize $\pm 2.1$}  65.2    {\scriptsize $\pm 1.3$}  0.056   {\scriptsize $\pm 0.014$}   & 72.5   {\scriptsize $\pm 4.2$}  56.3  {\scriptsize $\pm 1.4$} 0.188   {\scriptsize $\pm 0.113$}   & 70.1  {\scriptsize $\pm 4.9$} 55.5  {\scriptsize $\pm 3.2$}   0.121  {\scriptsize $\pm  0.037 $}    \\
			$C_3$-$C_5$                                                                   & \wlonef                                                                  & 97.1    {\scriptsize $\pm 1.1$}  38.0 {\scriptsize $\pm 1.5$}   0.022  {\scriptsize $\pm  0.001 $}        & 89.9  {\scriptsize $\pm  1.8$}  83.0  {\scriptsize $\pm  2.1$}  0.072 {\scriptsize $\pm 0.066$}              & 91.6   {\scriptsize $\pm 2.1 $} 63.6   {\scriptsize $\pm 1.1$}  0.052    {\scriptsize $\pm 0.018$}    & 71.4   {\scriptsize $\pm 3.5$} 56.6   {\scriptsize $\pm 1.2$}  0.229    {\scriptsize $\pm 0.187$} & 68.8 {\scriptsize $\pm 2.9$}  55.5   {\scriptsize $\pm 1.7$}  0.138   {\scriptsize $\pm  0.067$}    \\
			$C_3$-$C_6$                                                                   & \wlonef                                                                  & 96.5  {\scriptsize $\pm 1.5$} 38.7   {\scriptsize $\pm 1.4$}   0.021   {\scriptsize $\pm 0.001$}          & 92.2 {\scriptsize $\pm 1.4$}  83.5  {\scriptsize $\pm 2.2$}   0.090   {\scriptsize $\pm  0.039 $}            & 92.1   {\scriptsize $\pm 2.4$}  64.9  {\scriptsize $\pm 0.9$}   0.050   {\scriptsize $\pm   0.019$}   & 74.8  {\scriptsize $\pm 2.7$} 57.2  {\scriptsize $\pm 2.8$}  0.193  {\scriptsize $\pm 0.167$}     & 73.2  {\scriptsize $\pm 4.2$} 56.5     {\scriptsize $\pm 1.9$}  0.171  {\scriptsize $\pm   0.123$}  \\
			\cmidrule{1-7}
			$K_3$                                                                         & \wlonef                                                                  & 96.4    {\scriptsize $\pm 2.5$}  37.6 {\scriptsize $\pm 0.9$} 0.021  {\scriptsize $\pm  0.001$}           & 89.5    {\scriptsize $\pm 1.8$}  84.0  {\scriptsize $\pm 1.8$}    0.086   {\scriptsize $\pm  0.064$}         & 87.2    {\scriptsize $\pm 2.4$} 64.9   {\scriptsize $\pm 1.6$} 0.059   {\scriptsize $\pm 0.021$}      & 71.7   {\scriptsize $\pm 4.1$} 57.0  {\scriptsize $\pm 2.1$}  0.150   {\scriptsize $\pm 0.081 $}  & 67.8  {\scriptsize $\pm 3.2$} 54.7 {\scriptsize $\pm  3.1$}  0.162     {\scriptsize $\pm 0.068$}    \\
			$K_3$-$K_4$                                                                   & \wlonef                                                                  & 96.8  {\scriptsize $\pm 3.0$} 36.8   {\scriptsize $\pm 1.4$}   0.020 {\scriptsize $\pm  0.002 $}          & 88.3    {\scriptsize $\pm 2.0$} \textbf{84.7}   {\scriptsize $\pm 1.8$}   0.100   {\scriptsize $\pm  0.064$} & 88.9    {\scriptsize $\pm 2.6$} 64.7   {\scriptsize $\pm 1.2$} 0.062     {\scriptsize $\pm 0.018$}    & 70.6  {\scriptsize $\pm 3.1$} 56.0   {\scriptsize $\pm 2.3$} 0.151   {\scriptsize $\pm 0.065 $}   & 68.6  {\scriptsize $\pm 4.2$}   55.6  {\scriptsize $\pm  2.2$}  0.135     {\scriptsize $\pm 0.049$} \\

			$K_3$-$K_5$                                                                   & \wlonef                                                                  & 96.5    {\scriptsize $\pm 2.2$} 36.6  {\scriptsize $\pm 1.8$}  0.021  {\scriptsize $\pm 0.001 $}          & 88.2   {\scriptsize $\pm 2.2$} 82.7   {\scriptsize $\pm 2.6$}   0.098   {\scriptsize $\pm  0.072$}           & 89.7   {\scriptsize $\pm 2.8$} 64.5  {\scriptsize $\pm 0.9$}  0.047   {\scriptsize $\pm  0.012$}      & 72.9   {\scriptsize $\pm 5.7$}  57.2  {\scriptsize $\pm 1.2$} 0.182  {\scriptsize $\pm  0.077$}   & 67.6   {\scriptsize $\pm 4.4$} 54.3   {\scriptsize $\pm 1.3$}    0.145   {\scriptsize $\pm  0.055$} \\

			$K_3$-$K_6$                                                                   & \wlonef                                                                  & 95.8  {\scriptsize $\pm 2.0$}   37.7  {\scriptsize $\pm 1.5$} 0.021  {\scriptsize $\pm 0.001$}            & 88.7    {\scriptsize $\pm 1.8$} 84.3  {\scriptsize $\pm 1.1$}  0.078   {\scriptsize $\pm  0.049$}            & 88.9    {\scriptsize $\pm 3.0$}  63.6  {\scriptsize $\pm 1.4$}     0.055    {\scriptsize $\pm 0.022$} & 71.0    {\scriptsize $\pm 2.7$} 56.0  {\scriptsize $\pm 2.4$} 0.147   {\scriptsize $\pm 0.047$}   & 69.6    {\scriptsize $\pm 3.7$} 55.0   {\scriptsize $\pm 2.1$}  0.133   {\scriptsize $\pm  0.038 $} \\

			\cmidrule{1-7}
			$C_3$                                                                         & \wloaf                                                                   & 100.0 {\scriptsize $\pm 0.0$} 36.7  {\scriptsize $\pm 1.8$}  \textsc{Dnc}                                 & 99.3    {\scriptsize $\pm 1.4$} 83.4   {\scriptsize $\pm 2.7$}  \textsc{Dnc}                                 & 100.0   {\scriptsize $\pm 0.0$}  67.6  {\scriptsize $\pm 0.9$}  \textsc{Dnc}                          & 91.7  {\scriptsize $\pm 4.2$}  \textbf{59.6} {\scriptsize $\pm 0.5$}  \textsc{Dnc}                & 93.7  {\scriptsize $\pm 5.7$} 56.1   {\scriptsize $\pm 1.4$}  \textsc{Dnc}                          \\
			$C_3$-$C_4$                                                                   & \wloaf                                                                   & 100.0  {\scriptsize $\pm 0.0$} 36.1  {\scriptsize $\pm 1.8$}  \textsc{Dnc}                                & 99.6 {\scriptsize $\pm 1.1$}  84.3   {\scriptsize $\pm 1.9$}  \textsc{Dnc}                                   & 99.6   {\scriptsize $\pm 1.1$} 66.5  {\scriptsize $\pm 0.4$}  \textsc{Dnc}                            & 88.0  {\scriptsize $\pm 4.7$} 59.5  {\scriptsize $\pm 1.2$}  \textsc{Dnc}                         & 94.6 {\scriptsize $\pm 3.4$}  55.5 {\scriptsize $\pm 1.7$}  \textsc{Dnc}                            \\
			$C_3$-$C_5$                                                                   & \wloaf                                                                   & 100.0 {\scriptsize $\pm 0.0$} 35.0  {\scriptsize $\pm 1.7$}  \textsc{Dnc}                                 & 99.7  {\scriptsize $\pm 1.0$} 82.2  {\scriptsize $\pm 1.3$}  \textsc{Dnc}                                    & 100.0   {\scriptsize $\pm 0.0 $}  65.9 {\scriptsize $\pm 0.5$}  \textsc{Dnc}                          & 89.7 {\scriptsize $\pm 4.6$} 58.6  {\scriptsize $\pm 1.2$}  \textsc{Dnc}                          & 92.8 4.7   {\scriptsize $\pm 4.7$} 55.4  {\scriptsize $\pm 1.0$}  \textsc{Dnc}                      \\
			$C_3$-$C_6$                                                                   & \wloaf                                                                   & 100.0 {\scriptsize $\pm 0.0$} 35.8  {\scriptsize $\pm 1.8$}  \textsc{Dnc}                                 & 98.3  {\scriptsize $\pm 3.2$} 83.1  {\scriptsize $\pm 2.6$}  \textsc{Dnc}                                    & 99.6    {\scriptsize $\pm 1.2$} 66.2 {\scriptsize $\pm 0.6$}  \textsc{Dnc}                            & 91.7  {\scriptsize $\pm 5.3$} 58.7  {\scriptsize $\pm 2.6$}  \textsc{Dnc}                         & 89.2 {\scriptsize $\pm 5.4$}   55.9  {\scriptsize $\pm 1.5 $}  \textsc{Dnc}                         \\
			\cmidrule{1-7}
			$K_3$                                                                         & \wloaf                                                                   & 100.0  {\scriptsize $\pm 0.0$} 37.0   {\scriptsize $\pm 1.6$}  \textsc{Dnc}                               & 100.0  {\scriptsize $\pm 0.0$}  83.2   {\scriptsize $\pm 1.7$}  \textsc{Dnc}                                 & 99.2  {\scriptsize $\pm 1.5$}  67.2  {\scriptsize $\pm 0.5$}  \textsc{Dnc}                            & 92.5  {\scriptsize $\pm 5.9$} 59.0   {\scriptsize $\pm 1.7$}  \textsc{Dnc}                        & 92.7  {\scriptsize $\pm 3.6$} 57.0   {\scriptsize $\pm 1.8$}  \textsc{Dnc}                          \\
			$K_3$-$K_4$                                                                   & \wloaf                                                                   & 100.0   {\scriptsize $\pm 0.0$} 37.7   {\scriptsize $\pm 1.4$}  \textsc{Dnc}                              & 99.3 {\scriptsize $\pm 1.5$} 82.1   {\scriptsize $\pm 1.6$}  \textsc{Dnc}                                    & 100.0   {\scriptsize $\pm 0.0$} 66.7   {\scriptsize $\pm 0.7$}  \textsc{Dnc}                          & 92.5  {\scriptsize $\pm 5.9$} 59.0   {\scriptsize $\pm 1.7$}  \textsc{Dnc}                        & 94.1  {\scriptsize $\pm 5.0$} 57.6  {\scriptsize $\pm 1.7$}  \textsc{Dnc}                           \\

			$K_3$-$K_5$                                                                   & \wloaf                                                                   & 100.0  {\scriptsize $\pm 0.0$} 37.1  {\scriptsize $\pm 1.4$}  \textsc{Dnc}                                & 100.0  {\scriptsize $\pm 0.0$}  82.5   {\scriptsize $\pm 1.8$}  \textsc{Dnc}                                 & 98.9   {\scriptsize $\pm 1.7$}  66.7 {\scriptsize $\pm 1.1$}  \textsc{Dnc}                            & 93.3  {\scriptsize $\pm 4.3$}  \textbf{59.6}   {\scriptsize $\pm 1.0$}  \textsc{Dnc}              & 94.6  {\scriptsize $\pm 2.7$} \textbf{57.1}   {\scriptsize $\pm 1.3 $}  \textsc{Dnc}                \\
			$K_3$-$K_6$                                                                   & \wloaf                                                                   & 100.0  {\scriptsize $\pm 0.0$} 38.3   {\scriptsize $\pm 1.6$}  \textsc{Dnc}                               & 99.6  {\scriptsize $\pm 1.1$} 83.4   {\scriptsize $\pm 1.4$}  \textsc{Dnc}                                   &
			98.  {\scriptsize $\pm 1.8$} 5  66.8 {\scriptsize $\pm 0.6$}  \textsc{Dnc}    & 94.6  {\scriptsize $\pm 4.6$} 58.9 {\scriptsize $\pm 1.1$}  \textsc{Dnc} & 93.2  {\scriptsize $\pm 6.1$}  56.4   {\scriptsize $\pm 2.5$}  \textsc{Dnc}                                                                                                                                                                                                                                                                                                                                                                                                                                                                \\
			\bottomrule
		\end{tabular}
	}
\end{table}

\paragraph{Datasets}  We used the well-known graph classification benchmark datasets from~\citet{Mor+2020}; see~\cref{statistics} for dataset statistics and properties.\footnote{All datasets are publicly available at \url{www.graphlearning.io}.} Specifically, we used the \textsc{Enzymes}~\citep{Sch+2004,Bor+2005}, \textsc{Mutag}~\citep{Deb+1991,Kri+2012}, \textsc{Proteins}~\citep{Dob+2003,Bor+2005}, \textsc{PTC\_FM}, and \textsc{PTC\_MR}~\citep{Hel+2001} datasets. \emph{To concentrate purely on the graph structure, we omitted potential vertex and edge labels.} Moreover, we created two sets of synthetic datasets. First, we created synthetic datasets to verify~\cref{thm:separability}. We followed the construction outlined in the proof of~\cref{thm:separability} to create 1\,000 graphs on 16, 32, 64, and 128 vertices. Secondly, we created 1\,000 Erdős–Rényi graphs with 20 vertices, each, using edge probabilities 0.05, 0.1, 0.2, and 0.3, respectively. Here, we set a graph's class to the number of subgraphs isomorphic to either $C_3$, $C_4$, $C_5$, or $K_4$, resulting in sixteen different datasets.

\paragraph{Graph kernel and MPNNs architectures} We implemented the (normalized) \wlone{}, \wloa{}, \wlonef{}, and the \wloaf{} in Python. For the MPNN experiments, we used the \textsc{GIN} layer~\citep{Xu+2018b} using  $\relu$ activation functions and fixed the feature dimension to 64. We used mean pooling and a two-layer MLP using a dropout of 0.5 after the first layer for all experiments for the final classification. For the \MPNNF{} architectures, we encoded the initial label function $l_{\cF}$ as a one-hot encoding.

\paragraph{Experimental protocol and model configuration} For the graph kernel experiments, for the \wloa{} variants, we computed the (cosine) normalized Gram matrix for each kernel and computed the classification accuracies using the $C$-SVM implementation of \textsc{LibSVM}~\citep{Cha+11}. Here, a large $C$ enforces linear separability with a large margin. We computed the $\ell_2$ normalized feature vectors for the other kernels and computed the classification accuracies using the linear SVM implementation of \textsc{LibLinear}~\citep{Fan+2008}. In both cases, we used 10-fold cross-validation. We repeated each 10-fold cross-validation ten times with different random folds and report average training and testing accuracies and standard deviations. We additionally report the margin on the training splits for the linear SVM experiments. For the kernelized SVM this was not possible. For the experiments on the \textsc{TUDatsets}, following the evaluation method proposed in~\citet{Mor+2020}, the $C$-parameter and numbers of iterations were selected from $\{10^{-3}, 10^{-2}, \dotsc, 10^{2},$ $10^{3}\}$ and $\{1, \dotsc, 5 \}$, respectively, using a validation set sampled uniformly at random from the training fold (using 10\,\% of the training fold). For \wlonef{} and \MPNNF, we used cycles and complete graphs on three to six vertices, respectively.

For the synthetic datasets, we set the $C$-parameter to $10^{10}$ to enforce linear separability and choose the number of iterations as with the \textsc{TUDatsets}. All kernel experiments were conducted on a workstation with 512\,GB of RAM using a single CPU core.

For the MPNN experiments, we also followed the evaluation method proposed in~\citet{Mor+2020}, choosing the number of layers from $\{1, \dotsc, 5 \}$, using a validation set sampled uniformly at random from the training fold (using 10\,\% of the training fold). We used an initial learning rate of 0.01 across all experiments with an exponential learning rate decay with patience of 5, a batch size of 128, and set the maximum number of epochs to 200. All MPNN experiments were conducted on a workstation with 512\,GB of RAM using a single core and one NVIDIA Tesla A100s with 80\,GB of GPU memory.

\begin{figure}[t]
	\centering
	\subcaptionbox{$C_3$}{\includegraphics[scale=0.26]{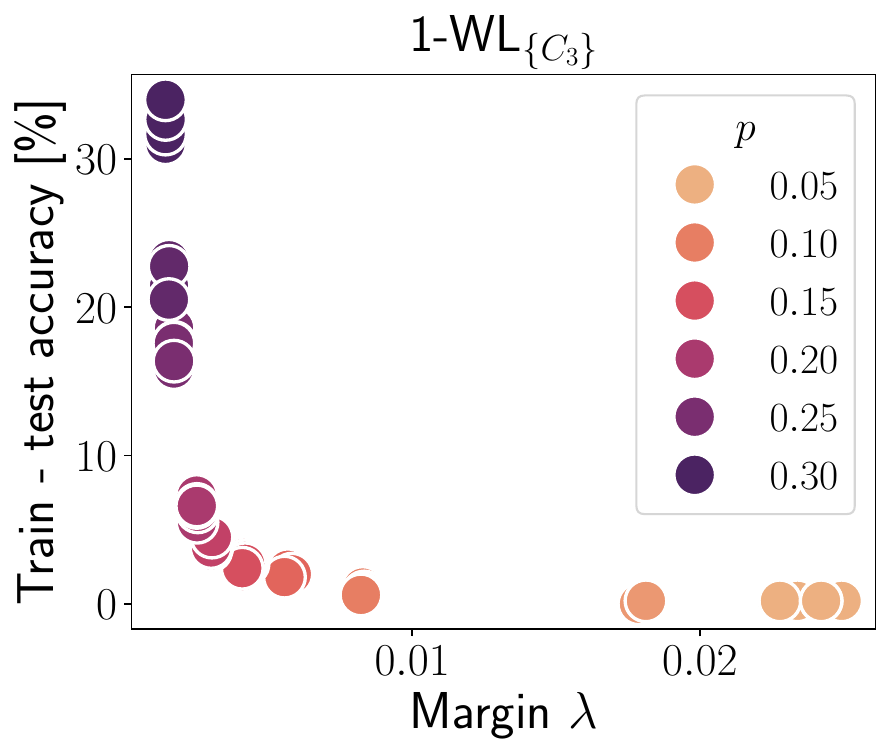}}
	\subcaptionbox{$C_4$}{\includegraphics[scale=0.26]{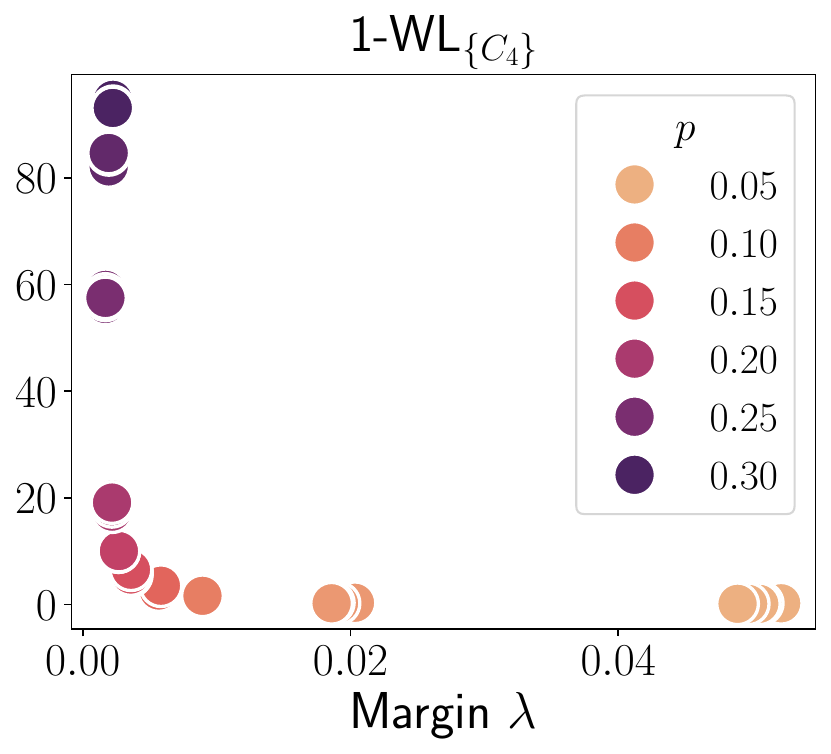}}
	\subcaptionbox{$C_5$}{\includegraphics[scale=0.26]{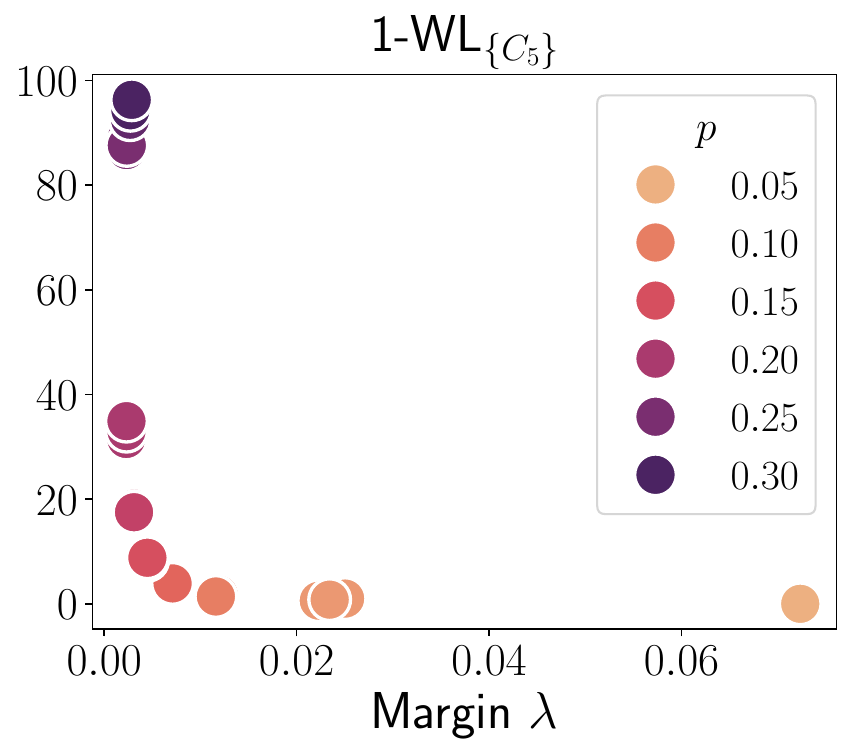}}
	\subcaptionbox{$K_4$}{\includegraphics[scale=0.26]{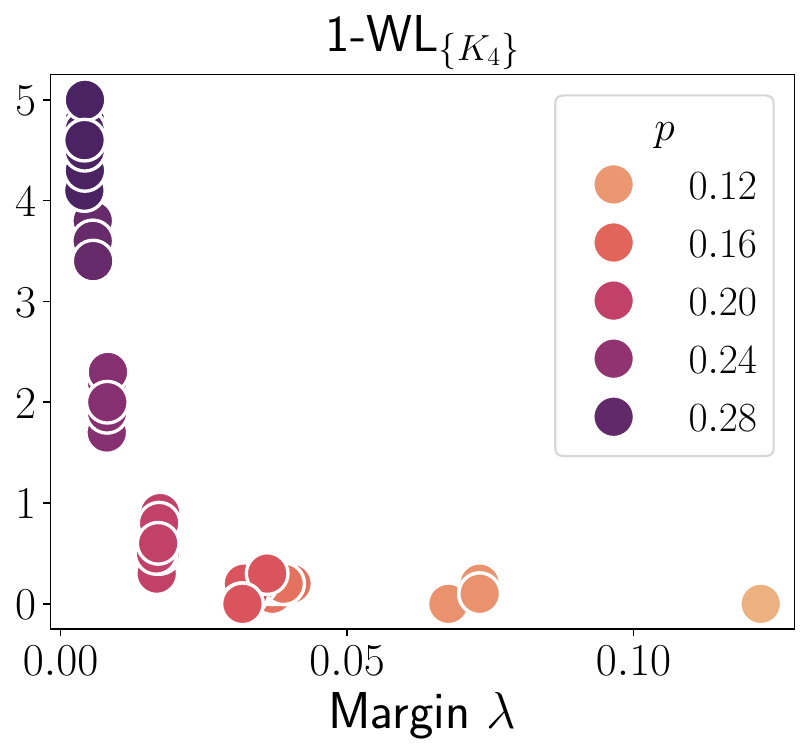}}
	\caption{Plot illustrating the relation between the margin and generalization error for the ER graphs for different choices of $\cF$.}
	\label{fig:plot}
\end{figure}

\subsection{Results and discussion} In the following we answer questions \textbf{Q1} to \textbf{Q4}.

\paragraph{\textbf{Q1} {\normalfont (\textit{``Does adding expressive power make datasets more linearly separable?'')}}} See~\cref{fig:separability,fig:er,fig:er_mpnn} (in the appendix). \cref{fig:separability} confirms~\cref{thm:separability}, i.e., the \wlone{} and \wloa{} kernels do not achieve accuracies better than random and cannot linearly separate the training data. The \wlonef{} and \wloaf{} kernel linearly separate the data while achieving perfect test accuracies. In addition,~\cref{fig:er} also confirms this for the ER graphs, i.e., for all datasets, the \wlone{} and \wloa{} kernels cannot separate the training data while the \wlonef{} and \wloaf{} can. Moreover, the subgraph-based kernels achieve the overall best predictive performance over all datasets, e.g., on the dataset using edge probability 0.2 and $\cF = \{ C_5 \}$ the test accuracies of the \wlonef{} improves over the \wlone{} by more than 57\,\%. Similar effects can be observed for the MPNN architectures; see~\cref{fig:er_mpnn}.

\paragraph{\textbf{Q2} {\normalfont (\textit{``Can the increased generalization performance of a more expressive variant of the \wlone{} algorithm be explained by an increased margin?'')}}} See~\cref{fig:tud,fig:er,fig:plot} (in the appendix).
On the \textsc{TUDatasets}, an increased margin often leads to less difference between train and test accuracy; see~\cref{fig:tud}. For example, on the \textsc{Proteins} datasets, the \wlonef, for all $\cF$, leads to a larger difference, while its margin is always \emph{strictly smaller} than \wlone's margin. Hence, the empirical results align with our theory, i.e., a smaller margin worsens the generalization error. Similar effects can be observed for all other datasets, except \textsc{Mutag}. On the ER dataset, comparing the \wlone{} and \wlonef, for all $\cF$, we can clearly confirm the theoretical results. That is, the \wlone{} cannot separate any dataset with a positive margin, while the \wlonef{} can, and we observe a decreased difference between \wlonef{}'s train and test accuracies compared to the \wlone. Analyzing the \wlonef{} further, for all $\cF$, a decreasing margin always results in an increased difference between test and train accuracies. For example, for $\cF = \{ C_4 \}$ and $p=0.05$, the \wlonef{} achieves a margin of 0.037 with a difference of 0.1\,\%, for $p = 0.1$, it achieves a margin of 0.009 with a difference of 1.8\,\%, for $p = 0.1$, it achieves a margin of 0.002 with a difference of 31.2\,\%, and, for $p = 0.3$, it achieve a margin of 0.003 with difference of 95.8\,\%, confirming the theoretical results. Moreover, see~\cref{fig:plot} of visual illustration of this observation.

\paragraph{\textbf{Q3} {\normalfont (\textit{``Does the \wloaf{} lead to increased predictive performance?'')}}} See~\cref{fig:separability,fig:er,fig:tud} (in the appendix). \cref{fig:tud} shows that the \wloaf{} kernel performs similarly to the \wlonef{}, while sometimes achieving better accuracies, e.g., on the \textsc{PTC\_FM} and \textsc{PTC\_MR} datasets. \cref{fig:separability} confirms this observation for the empirical verification of~\cref{thm:separability}, i.e., the \wloaf{} achieve perfect accuracy scores. The results are less clear for the ER datasets. On some datasets, e.g., edge probability 0.05, the \wloaf{} performs similarly to the \wlonef{} architecture. However, the algorithm leads to significantly worse predictive performance on other datasets. We speculate this is due to numerical problems of the kernelized SVM implementation.

\paragraph{\textbf{Q4} {\normalfont (\textit{``Do the results lift to MPNNs?'')}}} See~\cref{fig:separability,fig:tud_gnn,fig:er_mpnn} (in the appendix). \cref{fig:separability} shows that~\cref{thm:separability} also lifts to MPNNs, i.e., like the \wlonef{} kernel, the \MPNNF{} architecture achieves perfect prediction accuracies while the standard MPNNs does not perform better than random. In addition, on the TUDatasets, the \MPNNF{} architecture clearly outperforms the standard MPNN over all datasets; see~\cref{fig:tud_gnn}. For example, on the \textsc{Enzymes} dataset, the  \MPNNF{} architecture beats the MPNN by at least 7\,\%, for all subgraph choices. This observation holds across all datasets. Moreover, the \MPNNF{} architectures also achieve better predictive performance on all ER datasets, see~\cref{fig:er_mpnn}, compared to MPNNs. For example
on the dataset using edge probability 0.2 and $\cF = \{ C_5 \}$, the test accuracies of the \MPNNF{} architecture improves over the MPNN by more than 52\,\%.

\begin{table}[t]
	\caption{Mean train, test accuracies, and margins of kernel architectures on ER graphs for different levels of sparsity and different subgraphs. \textsc{Nls}---Not linearly separable. \textsc{Nb}---Only one class in the dataset.}
	\label{fig:er}
	\centering

	\resizebox{1.0\textwidth}{!}{ 	\renewcommand{\arraystretch}{1.05}
		\begin{tabular}{@{}c <{\enspace}@{}lcccc@{}} \toprule
			\multirow{3}{*}{\vspace*{4pt}\textbf{Subgraph}} & \multirow{3}{*}{\vspace*{4pt}\textbf{Algorithm}} & \multicolumn{4}{c}{\textbf{Edge probability}}                                                                                                                                                                                                                                                                                                                                                                                                               \\\cmidrule{3-6}
			                                                &                                                  & 0.05                                                                                                             & 0.1                                                                                                           & 0.2                                                                                                          & 0.3                                                                                                       \\	\toprule
			\multirow{4}{*}{ $C_3$ }                        & \wlone                                           & 94.2  {\scriptsize $\pm 1.7$}   88.6     {\scriptsize $\pm 0.7$}  \textsc{Nls}                                   & 96.3    {\scriptsize $\pm 4.6$}  42.12    {\scriptsize $\pm  1.7$}   0.013  {\scriptsize $\pm 0.001$}         & 44.1  {\scriptsize $\pm 12.4 $} 11.2  {\scriptsize $\pm $}  \textsc{Nls}                                     & 38.4  {\scriptsize $\pm  5.0$}  6.99   {\scriptsize $\pm 1.1$}  \textsc{Nls}                              \\
			                                                & \wloa                                            & 100.0  {\scriptsize $\pm 0.0$} 87.9   {\scriptsize $\pm 0.1$}   \textsc{Dnc}                                     & 100.0 {\scriptsize $\pm 0.0$}  44.6   {\scriptsize $\pm 1.1$}   \textsc{Dnc}                                  & 100.0   {\scriptsize $\pm 0.0$} 11.5  {\scriptsize $\pm 0.9$}   \textsc{Dnc}                                 & 100.0   {\scriptsize $\pm 0.0$}  5.2   {\scriptsize $\pm 0.3$}   \textsc{Dnc}                             \\

			                                                & \wlonef                                          & 100.0    {\scriptsize $\pm 0.0$}  \textbf{100.0}   {\scriptsize $\pm 0.0$} 0.037   {\scriptsize $\pm 0.0$}       & 100.0     {\scriptsize $\pm 0.0$}  \textbf{99.7}  {\scriptsize $\pm 0.1$}  0.009   {\scriptsize $<0.001$}     & 100.0   {\scriptsize $\pm 0.0$}  \textbf{100.0}  {\scriptsize $\pm 0.4$}    0.002   {\scriptsize $< 0.001$}  & 99.1    {\scriptsize $\pm 0.2$}  \textbf{64.7} {\scriptsize $\pm 0.8$} 0.001  {\scriptsize $< 0.001$}     \\
			                                                & \wloaf                                           & 100.0   {\scriptsize $\pm 0.0$} 98.6   {\scriptsize $\pm 0.0$}   \textsc{Dnc}                                    & 100.0   {\scriptsize $\pm 0.0$} 93.8   {\scriptsize $\pm 0.3$}   \textsc{Dnc}                                 & 100.0  {\scriptsize $\pm 0.0$}  42.0   {\scriptsize $\pm 1.0$}   \textsc{Dnc}                                & 100.0   {\scriptsize $\pm 0.0$} 7.4   {\scriptsize $\pm 0.2$}   \textsc{Dnc}                              \\
			\cmidrule{1-6}

			\multirow{4}{*}{ $C_4$ }                        & \wlone                                           & 95.1  {\scriptsize $\pm 1.0$} 92.3   {\scriptsize $\pm 0.2$}    \textsc{Nls}                                     & 89.7  {\scriptsize $\pm 3.9$}   46.8   {\scriptsize $\pm 0.8$}  \textsc{Nls}                                  & 48.7  {\scriptsize $\pm 10.1$}  5.4   {\scriptsize $\pm 0.5$}  \textsc{Nls}                                  & 46.6   {\scriptsize $\pm 10.8$}  2.2   {\scriptsize $\pm 0.4$}  \textsc{Nls}                              \\
			                                                & \wloa                                            & 100.0  {\scriptsize $\pm 0.0$}  92.6   {\scriptsize $\pm 0.0$}   \textsc{Dnc}                                    & 100.0 {\scriptsize $\pm 0.0$}  49.6   {\scriptsize $\pm 0.8$}   \textsc{Dnc}                                  & 100.0   {\scriptsize $\pm 0.0$} 5.13   {\scriptsize $\pm 0.6$}   \textsc{Dnc}                                & 100.0   {\scriptsize $\pm 0.0$} 1.7   {\scriptsize $\pm 0.3$}   \textsc{Dnc}                              \\
			                                                & \wlonef                                          & 100.0   {\scriptsize $\pm 0.0$}  \textbf{99.9}  {\scriptsize $\pm 0.1$}    0.037   {\scriptsize $\pm 0.001$}     & 100.0   {\scriptsize $\pm 0.0$}  \textbf{98.2}  {\scriptsize $\pm 0.2$}       0.009   {\scriptsize $< 0.001$} & 100.0   {\scriptsize $0.0$}  \textbf{78.9}  {\scriptsize $\pm 0.6$}    0.002   {\scriptsize $<0.001$}        & 100.0     {\scriptsize $\pm 0.0$}  \textbf{7.3} {\scriptsize $\pm 0.4$}      0.002 {\scriptsize $<0.001$} \\
			                                                & \wloaf                                           & 100.0   {\scriptsize $\pm 0.0$} 99.3   {\scriptsize $<0.1 $}   \textsc{Dnc}                                      & 100.0  {\scriptsize $\pm 0.0$} 93.7   {\scriptsize $\pm 0.2$}   \textsc{Dnc}                                  & 100.0   {\scriptsize $\pm $}  22.4  {\scriptsize $\pm 0.9$}   \textsc{Dnc}                                   & 100.0  {\scriptsize $\pm 0.0$} 2.8   {\scriptsize $\pm 0.6$}   \textsc{Dnc}                               \\
			\cmidrule{1-6}

			\multirow{4}{*}{ $C_5$ }                        & \wlone                                           & 97.2   {\scriptsize $\pm 0.5$}  95.8   {\scriptsize $\pm 0.3$}    \textsc{Nls}                                   & 69.3  {\scriptsize $\pm 6.6$}  53.5  {\scriptsize $\pm 0.6$}  \textsc{Nls}                                    & 65.1   {\scriptsize $\pm 14.9$}  4.3   {\scriptsize $\pm 0.7$}  \textsc{Nls}                                 & 64.8   {\scriptsize $\pm 9.9$} 1.26   {\scriptsize $\pm 0.2$}  \textsc{Nls}                               \\
			                                                & \wloa                                            & 100.0   {\scriptsize $\pm 0.0$} 95.9   {\scriptsize $\pm 0.0$}   \textsc{Dnc}                                    & 100.0   {\scriptsize $\pm 0.0$} 54.2   {\scriptsize $\pm 0.3$}   \textsc{Dnc}                                 & 100.0   {\scriptsize $\pm 0.0$} 4.7   {\scriptsize $\pm 0.5$}   \textsc{Dnc}                                 & 100.0 {\scriptsize $\pm 0.0$}  1.4   {\scriptsize $\pm 0.5$}   \textsc{Dnc}                               \\
			                                                & \wlonef                                          & 100.0    {\scriptsize $\pm 0.0$}    \textbf{99.9}   {\scriptsize $\pm 0.0$}  0.058     {\scriptsize $\pm 0.001$} & 100.0  {\scriptsize $\pm 0.0$}  \textbf{98.4} {\scriptsize $\pm 0.2$} 0.012  {\scriptsize $ <0.001$}          & 100.0 {\scriptsize $\pm 0.0$}  \textbf{68.8}  {\scriptsize $\pm 0.5$}   0.002   {\scriptsize $< 0.001 $}     & 100.0  {\scriptsize $\pm 0.0$}  \textbf{4.2}     {\scriptsize $\pm 0.5$}    0.003 {\scriptsize $< 0.001$} \\
			                                                & \wloaf                                           & 100.0  {\scriptsize $\pm 0.0$}  99.5  {\scriptsize $\pm 0.0$}  \textsc{Dnc}                                      & 100.0   {\scriptsize $\pm 0.0$}  91.8  {\scriptsize $\pm 0.4$}   \textsc{Dnc}                                 & 100.0  {\scriptsize $\pm 0.0$} 20.0   {\scriptsize $\pm 0.0$}   \textsc{Dnc}                                 & 100.0   {\scriptsize $\pm 0.0$} 2.6   {\scriptsize $\pm 0.2$}   \textsc{Dnc}                              \\
			\cmidrule{1-6}
			\multirow{4}{*}{ $K_4$ }                        & \wlone                                           & \textsc{Nb}                                                                                                      & 99.4  {\scriptsize $<0.1 $} 99.4  {\scriptsize $\pm 0.0$}  \textsc{Nls}                                       & 78.0   {\scriptsize $\pm 0.6$}   77.7  {\scriptsize $\pm 0.3$}  \textsc{Nls}                                 & 68.7    {\scriptsize $\pm 9.9$} 17.1   {\scriptsize $\pm 0.9$}  \textsc{Nls}                              \\
			                                                & \wloa                                            & \textsc{Nb}                                                                                                      & 100.0   {\scriptsize $\pm 0.0$} 99.4   {\scriptsize $\pm 0.0$}   \textsc{Dnc}                                 & 100.0 {\scriptsize $\pm 0.0$}  77.8   {\scriptsize $\pm 0.3$}   \textsc{Dnc}                                 & 100.0  {\scriptsize $\pm 0.0$}  20.6   {\scriptsize $\pm 0.9$}   \textsc{Dnc}                             \\
			                                                & \wlonef                                          & \textsc{Nb}                                                                                                      & 100.0  {\scriptsize $\pm 0.0$}  \textbf{100.0}  {\scriptsize $\pm 0.0$}   0.122    {\scriptsize $\pm 0.0$}    & 100.0    {\scriptsize $\pm 0.0$}  \textbf{99.9}   {\scriptsize $\pm 0.1$}  0.022     {\scriptsize $< 0.001$} & 100.0  {\scriptsize $\pm 0.0$}  \textbf{94.8}  {\scriptsize $\pm 0.4$}  0.004   {\scriptsize $<0.001$}    \\
			                                                & \wloaf                                           & \textsc{Nb}                                                                                                      & 100.0 {\scriptsize $\pm 0.0$}  99.4  {\scriptsize $\pm 0.0$}   \textsc{Dnc}                                   & 100.0   {\scriptsize $\pm 0.0$} 97.8   {\scriptsize $\pm 0.1$}   \textsc{Dnc}                                & 100.0  {\scriptsize $\pm 0.0$} 74.6   {\scriptsize $\pm 0.6$}   \textsc{Dnc}                              \\
			\bottomrule
		\end{tabular}
	}
\end{table}

\section{Conclusion}
Here, we focused on determining the precise conditions under which increasing the expressive power of MPNN or kernel architectures leads to a provably increased generalization performance. When viewed through graph isomorphism, we first showed that an architecture's expressivity offers limited insights into its generalization performance. Additionally, we focused on augmenting \wlone{} with subgraph information and derived tight upper and lower bounds for the architectures' VC dimension parameterized by the margin. Based on this, we derived data distributions where increased expressivity either leads to improved generalization performance or not. Finally, we introduced variations of expressive \wlone-based kernels and neural architectures with provable generalization properties. Our empirical study confirmed the validity of our theoretical findings.

Our theoretical results constitute an essential initial step in unraveling the conditions under which more expressive MPNN and kernel architectures yield enhanced generalization performance. Hence, our theory lays a solid foundation for the systematic and principled design of novel expressive MPNN architectures.

\section*{Impact Statement}
This paper presents work whose goal is to advance the field of machine learning. There are many potential societal consequences of our work, none of which we feel must be specifically highlighted here.

\section*{Acknowledgements}
Christopher Morris is partially funded by a DFG Emmy Noether grant (468502433) and RWTH Junior Principal Investigator Fellowship under Germany’s Excellence Strategy.

\setcitestyle{numbers}
\bibliography{bibliography}

\newpage
\appendix

\section{Simple MPNNs}\label{sec:simpleMPNNs}
Here, we provide more details on the simple MPNNs mentioned in~\cref{sec:MPNN}. That is,
for given $d$ and $L \in \Nb$, we define the class $\MPNN_{\mathsf{mlp}}(d,L)$ of simple MPNNs as $L$-layer
MPNNs for which, according to~\cref{def:MPNN}, for each $t \in [L]$, the aggregation function $\AGG^\tup{t}$ is
summation and the update function $\UPD^\tup{t}$ is a multilayer perceptron $\mathsf{mlp}^\tup{t} \colon \Rb^{2d}\to\Rb^d$ of width at most $d$. Similarly, the readout function in \cref{def:readout}
consists of a multilayer perceptron $\mathsf{mlp} \colon \Rb^d\to\Rb^d$ applied on the sum of all vertex features computed in layer $L$.\footnote{For simplicity, we assume that all feature dimensions of the layers are fixed to $d \in \Nb$.} More specifically,
MPNNs in $\MPNN_{\mathsf{mlp}}(d,L)$ compute on a labeled graph $G=(V(G),E(G),\ell)$ with $d$-dimensional initial vertex features
$\hb_v^\tup{0}\in\Rb^d$, consistent with $\ell$,
the following vertex features, for each $v\in V(G)$,
\begin{equation*}
	\hb_{v}^\tup{t} \coloneqq
	\mathsf{mlp}^\tup{t}\Bigl(\hb_{v}^\tup{t-1},\sum_{u\in N(v)}\hb_{u}^\tup{t-1}\Bigr) \in \Rb^{d},
\end{equation*}
for $t \in [L]$, and
\begin{equation*}
	\hb_G \coloneqq \mathsf{mlp}\Bigl(\sum_{v\in V(G)}\hb_{v}^{\tup{L}}\Bigr) \in \Rb^{d}.
\end{equation*}
Note that the class $\MPNN_{\mathsf{mlp}}(d,L)$ encompasses the GNN architecture derived in~\citet{Mor+2019} that has the same expressive power as the \wlone{} in distinguishing non-isomorphic graphs.

\section{Proofs missing from the main paper}

Here, we outline proofs missing in the main paper.

\subsection{Fundamentals}
Here, we prove some fundamental statements for later use.

\paragraph{Margin optimization} Let $(\vec x_1, y_1), \ldots, (\vec x_n, y_n) \in \mathbb R^d\times \{0,1\}$, $d>0$, be a linearly separable sample, and let $I^+ \coloneqq \{i\in[n] \mid y_i=1\}$ and $I^- \coloneqq \{i\in[n] \mid y_i=0\}$. Consider the well-known alternative---to the typical hard-margin SVM formulation---optimization problem for finding the minimum distance between the convex sets induced by the two classes, i.e.,
\begin{equation}
	\begin{aligned}
		2\lambda\coloneqq \min_{ \boldsymbol{\alpha}\in \mathbb R^{|I^+|}, \boldsymbol{\beta}\in \mathbb R^{|I^-|}} \quad & \lVert \vec x^+_{\boldsymbol{\alpha}} - \vec x^-_{\boldsymbol{\beta}}\rVert \\
		\textrm{s.t.} \quad                                                                                               & \vec x^+_{\boldsymbol{\alpha}} =\sum_{i\in I^+} \alpha_i \vec x_i, \quad
		\vec x^-_{\boldsymbol{\beta}} = \sum_{j\in I^-} \beta_j \vec x_j                                                                                                                                \\
		                                                                                                                  & \sum_{i\in I^+} \alpha_i = 1, \quad
		\sum_{j\in I^-} \beta_j = 1,                                                                                                                                                                    \\
		                                                                                                                  & \forall\, i\in I^+, j\in I^- \colon \alpha_i \geq 0, \beta_j \geq 0,
	\end{aligned}\label{eq:convexDistance}
\end{equation}
where $\boldsymbol{\alpha}$ and $\boldsymbol{\beta}$ are the variables determining the convex combinations for both the positive and negative classes.
Moreover, $\lambda$ is exactly the margin that is computed by the typical hard-margin SVM and from the optimal arguments $\boldsymbol{\alpha}^*$ and $\boldsymbol{\beta}^*$, we can compute the usual hard-margin solution $\mathbf{w}$ and $b$ as:
\begin{align*}
	\mathbf{w}:= & \frac{\vec x^+_{\boldsymbol{\alpha}^*} - \vec x^-_{\boldsymbol{\beta}^*}}{\lambda^2}                              \\
	b:=          & \frac{\lVert\vec x^-_{\boldsymbol{\beta}^*}\rVert^2 - \lVert\vec x^+_{\boldsymbol{\alpha}^*}\rVert^2}{2\lambda^2}
\end{align*}.

We can describe $\lVert \vec x^+_{\boldsymbol{\alpha}} - \vec x^-_{\boldsymbol{\beta}}\rVert^2$ by a sum of pairwise distances.
\begin{align}\label{eq:AOProofFinal}
	\norm[\bigg]{ \vec x^+_{\boldsymbol{\alpha}} - \vec x^-_{\boldsymbol{\beta}}}^2 \notag & =  \norm[\bigg]{ \sum_{i\in I^+} \alpha_i \vec x_i - \sum_{j\in I^-} \beta_j \vec x_j}^2                                                                                                                                    \\ \notag
	                                                                                       & =  \norm[\bigg]{ \sum_{i\in I^+} \alpha_i \vec x_i\sum_{j\in I^-} \beta_j - \sum_{j\in I^-} \beta_j \vec x_j\sum_{i\in I^+} \alpha_i}^2                                                                                     \\\notag
	                                                                                       & =  \norm[\bigg]{ \sum_{i\in I^+}\sum_{j\in I^-} \alpha_{i}\beta_{j} \vec x_i - \sum_{i\in I^+}\sum_{j\in I^-} \alpha_{i}\beta_{j} \vec x_j}^2                                                                               \\\notag
	                                                                                       & =  \norm[\bigg]{ \sum_{(i,j)\in I^+\times I^-} \delta_{i,j}(\vec x_i-\vec x_j)}^2\quad (\delta_{i, j} \coloneqq \alpha_i\beta_j)                                                                                            \\ \notag
	                                                                                       & = \sum_{(i,j)\in I^+\times I^-}\sum_{(k,l)\in I^+\times I^-} \delta_{i,j}\delta_{k,l}(\vec x_i-\vec x_j)^\tran (\vec x_k-\vec x_l)                                                                                          \\\notag
	                                                                                       & = \sum_{(i,j), (k,l)\in I^+\times I^-} \delta_{i,j}\delta_{k,l}(-\vec x^\tran_i \vec x_l-\vec x^\tran_j \vec x_k+ \vec x^\tran_i \vec x_k+ \vec x^\tran_j\vec x_l)                                                          \\
	                                                                                       & = \frac{1}{2}\sum_{\mathclap{(i,j), (k,l)\in I^+\times I^-}} \delta_{i,j}\delta_{k,l}(\lVert \vec x_i-\vec x_l\rVert^2+\lVert \vec x_j-\vec x_k\rVert^2-\lVert \vec x_i-\vec x_k\rVert^2-\lVert \vec x_j-\vec x_l\rVert^2).
\end{align}
We remark that the pairwise distances indexed by $(i, l)$ and $(j, k)$ represent inter-class distances, since $y_i = y_k = 1$ and $y_j = y_l = 0$. Along the same line, the pairwise distances
indexed by $(i, k)$ and $(j, l)$ represent intra-class distances.

\begin{proposition}\label{prop:marginincrease}
	Let $(\vec x_1, y_1), \dotsc, (\vec x_n, y_n)$ and $(\tilde{\vec x}_1, y_1), \dots, (\tilde{\vec x}_n, y_n)$ in $\mathbb R^d$ be two linearly-separable samples, with margins $\lambda$ and $\tilde{\lambda}$, respectively, with the same labels $y_i \in \{0, 1\}$. If
	\begin{align}
		\min_{y_i\neq y_j} \lVert\tilde{\vec x}_i-\tilde{\vec x}_j\rVert^2 - \lVert \vec x_i-\vec x_j\rVert^2 >
		\max_{y_i=y_j} \lVert\tilde{\vec x}_i-\tilde{\vec x}_j\rVert^2 - \lVert \vec x_i-\vec x_j\rVert^2,\label{eq:marginAssumption}
	\end{align}
	then $\tilde{\lambda}>\lambda$.
	That is, we get an increase in margin if the minimum increase in distances between classes considering the two samples is strictly larger than the maximum increase in distance within each class.
\end{proposition}
\begin{proof}
	Let
	$$\Delta_\text{min} \coloneqq \min_{y_i\neq y_j} \lVert\tilde{\vec x}_i-\tilde{\vec x}_j\rVert^2 - \lVert \vec x_i-\vec x_j\rVert^2,$$
	and
	$$\Delta_\text{max} \coloneqq \max_{y_i=y_j} \lVert\tilde{\vec x}_i-\tilde{\vec x}_j\rVert^2 - \lVert \vec x_i-\vec x_j\rVert^2.$$
	By \cref{eq:marginAssumption}, $\Delta_\text{min}>\Delta_\text{max}$. Starting at \cref{eq:AOProofFinal},
	\begin{align*}
		\lVert \vec x^+_{\boldsymbol{\alpha}} - \vec x^-_{\boldsymbol{\beta}}\rVert^2 & = \frac{1}{2}\sum_{(i,j)}\sum_{(k,l)} \alpha_{i,j}\alpha_{k,l}(\lVert \vec x_i- \vec x_l\rVert^2+\lVert \vec x_j-\vec x_k\rVert^2-\lVert \vec x_i- \vec x_k\rVert^2-\lVert \vec x_j- \vec x_l\rVert^2)                                                                 \\
		\begin{split}
			&<\frac{1}{2}\sum_{(i,j)}\sum_{(k,l)} \alpha_{i,j}\alpha_{k,l}(\lVert \vec x_i- \vec x_l\rVert^2+\Delta_\text{min}+\lVert \vec x_j- \vec x_k\rVert^2+\Delta_\text{min}\\
			&\quad\quad\quad\quad\quad\quad\quad\quad-\lVert \vec x_i-\vec x_k\rVert^2-\Delta_\text{max}-\lVert \vec x_j-\vec x_l\rVert^2-\Delta_\text{max}
		\end{split}                                                                                                                                                                                  \\
		                                                                              & \leq \frac{1}{2}\sum_{(i,j)}\sum_{(k,l)} \alpha_{i,j}\alpha_{k,l}(\lVert \tilde{\vec x}_i-\tilde{\vec x}_l\rVert^2+\lVert \tilde{\vec x}_j-\tilde{\vec x}_k\rVert^2-\lVert \tilde{\vec x}_i-\tilde{\vec x}_k\rVert^2-\lVert \tilde{\vec x}_j-\tilde{\vec x}_l\rVert^2) \\
		                                                                              & =\lVert \tilde{\vec x}^+_{\boldsymbol{\alpha}} - \tilde{\vec x}^-_{\boldsymbol{\alpha}}\rVert^2,
	\end{align*}
	where $\tilde{\vec x}^+_{\boldsymbol{\alpha}}$ and $\tilde{\vec x}^-_{\boldsymbol{\alpha}}$ are derived from applying the optimization (\ref{eq:convexDistance}) to the datapoints $(\tilde{\vec x}_1, y_1), \allowbreak\dots,\allowbreak (\tilde{\vec x}_n, y_n)$.

	In the following, we omit all of the conditions from \cref{eq:convexDistance} for simplicity. Let $\vec x^{+*}$ and $\vec x^{-*}$ be the representatives of the optimal solution to \cref{eq:convexDistance}, then
	\begin{equation*}
		\forall \boldsymbol{\alpha}: \gamma = \lVert \vec x^{+*} - \vec x^{-*}\rVert \leq \lVert \vec x^+_{\boldsymbol{\alpha}} - \vec x^-_{\boldsymbol{\alpha}}\rVert.
	\end{equation*}
	Hence,
	\begin{equation*}
		\forall \boldsymbol{\alpha}: \gamma = \lVert \vec x^{+*} - \vec x^{-*}\rVert \leq \lVert \vec x^+_{\boldsymbol{\alpha}} - \vec x^-_{\boldsymbol{\alpha}}\rVert < \lVert \tilde{\vec x}^+_{\boldsymbol{\alpha}} - \tilde{\vec x}^-_{\boldsymbol{\alpha}}\rVert,
	\end{equation*}
	which implies that
	\begin{equation*}
		\gamma < \min_{\boldsymbol{\alpha}}\lVert \tilde{\vec x}^+_{\boldsymbol{\alpha}} - \tilde{\vec x}^-_{\boldsymbol{\alpha}}\rVert \eqqcolon \tilde{\gamma},
	\end{equation*}
	showing the desired result.
\end{proof}

\paragraph{Concatenating feature vectors} We will consider concatenating two feature vectors and analyze how this affects attained margins. To this end, let $X\coloneqq \{(\vec x_i,y_i)\in\Rb^d\times\{0,1\}\mid i\in [n]\}$. When we split up $\Rb^d$ into $\Rb^{d_1}\times\Rb^{d_2}$, we write
$\vec x_i \coloneqq (\vec x_i^1,\vec x_i^2)$ with $\vec x_i^1\in\Rb^{d_1}$ and $\vec x_i^2\in\Rb^{d_2}$.

\begin{proposition}\label{prop:concat_margin}
	If $X \coloneqq \{(\vec x_{1},y_{1}),\dotsc,(\vec x_{n}, y_{n})\}$ is a sample, such that
	\begin{enumerate}
		\item $(\vec x^1_{1},y_{1}),\dotsc,(\vec x^1_{n}, y_{n})$ is $(r_1,\gamma_1)$-separable and
		\item $(\vec x^2_{1},y_{1}),\dotsc,(\vec x^2_{n}, y_{n})$ is $(r_2,\gamma_2)$-separable,
	\end{enumerate}
	then $(\vec x_1,y_1),\dots,(\vec x_n, y_n)$ is $(\sqrt{r_1^2+r_2^2} ,\sqrt{\gamma_1^2+\gamma_2^2})$-separable.
\end{proposition}
\begin{proof}
	Let $I \coloneqq I^+\,\dot\cup\, I^-$ satisfying $y_i=1$ if, and only, if $i\in I^+$ and $y_i=0$ if, and only, if $i\in I^-$, $p\coloneqq |I|$, $p^+\coloneqq|I^+|$ and $p^-\coloneqq |I^-|$.
	Further, let $\vec x_i^+ \coloneqq \vec x_i$, $(\vec x_i^1)^+\coloneqq (\vec x_i^1,0)$, $(\vec x_i^2)^+\coloneqq (0,\vec x_i^2)$ for $i\in I^+$, and $\vec x_i^- \coloneqq \vec x_i$, $(\vec x_i^1)^-\coloneqq (\vec x_i^1,0)$, and $(\vec x_i^2)^-\coloneqq (0,\vec x_i^2)$ for $i\in I^-$.
	We collect $\vec x_i^+$, $\vec x_i^-$, $(\vec x_i^1)^+$, $(\vec x_i^2)^+$, $(\vec x_i^1)^-$, and $(\vec x_i^2)^-$ into matrices
	$\mX^+\in\Rb^{p^+\times d}$, $\mX^-\in\Rb^{p^-\times d}$, $\mX_1^+$, $\mX_2^+\in\Rb^{p^+\times d}$, and $\mX_1^-,\mX_2^-\in\Rb^{p^-\times d}$.

	The margins $\gamma_1$, $\gamma_2$, and $\gamma$ (the margin of $(\vec x_1,y_1),\dots,(\vec x_n, y_n)$) are given by
	\begin{align*}
		\gamma_1\coloneqq\min_{\boldsymbol{\alpha}\in(\Rb^{+,p^+}, \boldsymbol{\beta}\in\Rb^{+,p^-}, \mathbf{1}^\tran\boldsymbol{\alpha}=1=\mathbf{1}^\tran \boldsymbol{\beta}}\|(\mX^+_1)^\tran \boldsymbol{\alpha}-(\mX^-_1)^\tran \boldsymbol{\beta}\| \\
		\gamma_2\coloneqq\min_{\boldsymbol{\alpha}\in\Rb^{+,p^+}, \boldsymbol{\beta}\in\Rb^{+,p^-}, \mathbf{1}^\tran \boldsymbol{\alpha}=1=\mathbf{1}^\tran\boldsymbol{\beta}}\|(\mX^+_2)^\tran\boldsymbol{\alpha}-(\mX^-_2)^\tran \boldsymbol{\beta}\|   \\
		\gamma\coloneqq\min_{\boldsymbol{\alpha}\in\Rb^{+,p^+}, \boldsymbol{\beta}\in\Rb^{+,p^-}, \mathbf{1}^\tran\boldsymbol{\alpha}=1=\mathbf{1}^\tran\boldsymbol{\beta}}\|(\mX^+)^\tran\boldsymbol{\alpha}-(\mX^-)^\tran \boldsymbol{\beta}\|,
	\end{align*}
	where $\Rb^{+}$ is the set of positive real numbers and $\mathbf{1}$ is a vector of ones of appropriate size. We have
	\begin{align*}
		\|(\mX^+)^\tran \boldsymbol{\alpha}-(\mX^-)^\tran\boldsymbol{\beta}\|^2 & =\|(\mX_1^+)^{\tran} {\alpha}_1
		+ (\mX_2^+)^{\tran} {\alpha}_2-  (\mX_1^-)^\tran {\beta}_1 - (\mX_2^-)^\tran{\beta}_2\|^2                                                       \\
		                                                                        & = \|((\mX_1^+)^{\tran} {\alpha}_1 -  (\mX_1^-)^\tran {\beta}_1)+
		((\mX_2^+)^{\tran} {\alpha}_2 - (\mX_2^-)^\tran {\beta}_2)\|^2                                                                                  \\
		                                                                        & =    \|(\mX_1^+)^{\tran} {\alpha}_1 -  (\mX_1^-)^\tran {\beta}_1\|^2+
		\|(\mX_2^+)^{\tran} {\alpha}_2 - (\mX_2^-)^\tran {\beta}_2\|^2.
	\end{align*}
	The latter terms attain, by assumption, minimal values of $\gamma_1$ and $\gamma_2$, respectively. Thus, $\gamma^2=\gamma_1^2+\gamma_2^2$. Also note that $\|\vec x_i\|^2\leq r_1^2+r_2^2$ for all $i \in I$. This implies that $(\vec x_1,y_1),\dotsc,(\vec x_n, y_n)$ is $(\sqrt{r_1^2+r_2^2} ,\sqrt{\gamma_1^2+\gamma_2^2})$-separable.
\end{proof}

\paragraph{Existence of regular graphs} The following result ensures the existence of enough regular graphs needed for the proof of~\cref{thm:VCWL} and its variants.
\begin{lemma}\label{lem:regular}
	For any even $n$ and all $i\in \{0,\dots,n-1\}$, there exists an $i$-regular graph with one orbit containing all vertices.
\end{lemma}
\begin{proof}
	Let $n$ be even, and let $c$ be an arbitrary natural number. We define
	$$E_\text{odd}\coloneqq \{(i,i+\nicefrac{n}{2}) \mid i\in[\nicefrac{n}{2}]\},$$ and $$E_c \coloneqq \{(i,i+c \text{ mod } n) \mid i \in[n]\},$$ where $\text{mod}$ is the modulo operator with equivalence classes $[n]$. It is easily verified that for any $C \in \Nb$, $([n], \bigcup_{c\in[C]}E_c)$ is a $2C$-regular graph. Also, $([n], E_\text{odd}\cup\bigcup_{c\in[C]}E_c)$ is a $2C+1$-regular graph. The permutation, in cycle notation, $(1,2,\dots,n)$ is an automorphism for both graphs, implying that all vertices are in the same orbit.
\end{proof}

\begin{remark}
	For any odd $n$, no $i$-regular graph exists with $i$ odd. This is a classical textbook question that can be verified by handshaking. For regular graphs, $$\sum_{i\in [n]} \textsf{deg}(i) = i\cdot n.$$ Summing the degrees for each vertex counts each edge twice. Thus, $i \cdot n$ must be even, and since $n$ is odd, $i$ must be even.
\end{remark}

\subsection{Expressive power of enhanced variants}

We now prove results on the expressive power of the \wlonef.

\begin{proposition}[{\cref{prop:WLFfiner} in the main paper}]
	\label{APP:prop:WLFfiner}
	Let $G$ be a graph and $\cF$ be a set of graphs. Then, for all rounds, the \wlonef{} distinguishes at least the same vertices as the \wlone.
\end{proposition}
\begin{proof}
	Using, induction on $t$, we show that, for all vertices $v,w \in V(G)$,
	\begin{equation}\label{wlf}
		C^{1,\cF}_{t}(v) = C^{1,\cF}_{t}(w) \text{ implies } C^{1}_{t}(v) = C^{1}_{t}(w).
	\end{equation}
	The base case, $t = 0$, is clear since \wlonef{} refines the single color class induced by $C^{1}_{0}$. For the induction, assume that \cref{wlf} holds and assume that, $C^{1,\cF}_{t+1}(v) = C^{1,\cF}_{t+1}(w)$ holds. Hence, $C^1_{t}(v) = C^1_{t}(w)$ and
	\begin{equation*}
		\oms  C^{1,\cF}_{t}(a) \mid a \in N(v) \cms = \oms C^{1,\cF}_{t}(b) \mid b \in N(w)  \cms
	\end{equation*}
	holds. Hence, there is a \new{color-preserving bijection} $\varphi \colon N(v) \to N(w)$ between the above two multisets, i.e., $C^{1,\cF}_{t}(a) =  C^{1,\cF}_{t}(\varphi(a))$, for $a \in N(v)$. Hence, by \cref{wlf}, $C^{1}_{t}(a) =  C^{1}_{t}(\varphi(a))$, for $a \in N(v)$. Consequently, it holds that $C^1_{t+1}(v) = C^1_{t+1}(w)$, proving the desired result.
\end{proof}
In addition, by choosing the set of graphs $\cF$ appropriately, \wlonef{} gets strictly more expressive than \wlone{} in distinguishing non-isomorphic graphs.
\begin{proposition}[{\cref{thm:wlonef_strong} in the main paper.}]
	\label{APP:thm:wlonef_strong}
	For every $n \geq 6$, there exists at least one pair of non-isomorphic graphs and a set of graphs $\cF$ containing a single constant-order graph, such that, for all rounds, \wlone{} does not distinguish them while \wlonef{} distinguishes them after a single round.
\end{proposition}
\begin{proof}
	For $n = 6$, we can choose a pair of a $6$-cycle and the disjoint union of two $3$-cycles. Since both graphs are $2$-regular, the \wlone{} cannot distinguish them. By choosing $\cF = \{ C_3 \}$, the \wlonef{} distinguishes them. For $n > 6$, we can simply pad the graphs with $n-6$ isolated vertices.
\end{proof}

\subsection{Margin-based upper and lower bounds on the VC dimension of Weisfeiler--Leman-based kernels}\label{APP:vcbounds}

We first state the upper bound that we will be using for all the following cases, which is a classical result, for instance based on fat-shattering.

\begin{lemma}[Theorem 1.6 in \citep{bartlett1999generalization}] Let  $\Sb\subseteq\Rb^d$.
	\begin{equation*}
		\vcdim(\Hb_{r,\lambda}(\Sb)) \in O(\nicefrac{r^2}{\lambda^2}).
	\end{equation*}
\end{lemma}

We now prove the VC dimension theory results from the main paper.
In the following, we will reuse our notation of splitting up $\Rb^d$ into $\Rb^{d_1}\times\Rb^{d_2}$. We write
$\vec x_i=(\vec x_i^{(1)},\vec x_i^{(2)})$ with $\vec x_i^{(1)}\in\Rb^{d_1}$ and $\vec x_i^{(2)}\in\Rb^{d_2}$. Further, let $(\vec x_i^{(1)})^+\coloneqq (\vec x_i^{(1)},0)$, and $(\vec x_i^{(2)})^+\coloneqq (0,\vec x_i^{(2)})$.
\begin{lemma}[{\cref{lem:VC_lower_F} in the main paper}]  \label{APP:lem:VC_lower_F}
	Let  $\Sb\subseteq\Rb^d$. If $\Sb$ contains $m\coloneqq\lfloor\nicefrac{r^2}{\lambda^2}\rfloor$
	vectors $\vec b_1,\dotsc,\vec b_m \in \Rb^d$ with $\vec b_i:=(\vec b_i^{(1)}, \vec b_i^{(2)})$ and $\vec b_1^{(2)},\dotsc,\vec b_m^{(2)}$ being pairwise orthogonal, $\|\vec b_i\|=r'$, and $\|\vec b_i^{(2)}\|=r$,
	then
	\begin{equation*}
		\vcdim(\Hb_{r',\lambda}(\Sb)) \in \Omega(\nicefrac{r^2}{\lambda^2}).
	\end{equation*}
\end{lemma}
\begin{proof}
	Following the argument in \citet{Alo+2021}, we show that the vectors
	$\vec b_1,\dotsc,\vec b_m$ can be shattered. Indeed, let $A$ and $B$
	be two arbitary sets partitioning $[m]$. Consider the vector
	\begin{equation*}
		\vec w\coloneqq \frac{\lambda}{r^2}\Bigg(\sum_{i\in A}(\vec b_i^{(2)})^+ - \sum_{i\in B}(\vec b_i^{(2)})^+\Bigg).
	\end{equation*}
	We observe that, because of assumptions underlying the vectors $\vec b_i$, we have
	\begin{align*}
		\vec w^\tran \vec b_j=\begin{cases} \left(\frac{\lambda}{r^2}\right) \cdot (\vec b_j^{(2)})^\tran \vec b_j^{(2)}=\lambda & \text{if $j\in A$}  \\
              -\left(\frac{\lambda}{r^2}\right) \cdot (\vec b_j^{(2)})^\tran \vec b_j^{(2)}=-\lambda
                                                                                                   & \text{if $j\in B$}.
		                      \end{cases}
	\end{align*}
	In other words, $\vec w$ witnesses that the distance between the convex hull
	of $\{\vec b_i\mid i\in A\}$ and $\{\vec b_i\mid i\in B\}$ is at least $2\lambda$, implying the result.
\end{proof}

In the following, we will heavily rely upon \cref{APP:lem:VC_lower_F} and more specifically we can construct $m = \lfloor\nicefrac{r^2}{\lambda^2}\rfloor$ graphs. Since we will be using regular graphs for simplicity where each regular graph has different regularity, we require $n\geq m$, which is true for $n \geq \nicefrac{r^2}{\lambda^2}$. Notice that this requirement can be relaxed, and we could, for instance, consider graphs with nodes of two regularities, which would significantly lower the requirement on $n$. However, for these graphs the construction and proofs would become significantly more complex as we would have to additionally deal with signal propagation within these graphs until we can guarantee orthogonality of the \wlone-feature vectors. For this type of proof, we believe $e^{\nicefrac{n}{e}} \geq \nicefrac{r^2}{\lambda^2}$ would need to hold. However, we leave this to future research.
\begin{theorem}\label{APP:thm:VCWL}
	For any $T, \lambda>0$, we have,
	\begin{align*}
		 & \vcdim(\Hb_{\sqrt{T+1}n,\lambda}(\cE_{\mathsf{WL}}(n,d_T))) \in{\Omega}(\nicefrac{r^2}{\lambda^2}), \text{ for } r=\sqrt{T}n  \text{ and } n\geq \nicefrac{r^2}{\lambda^2}, \\
		 & \vcdim(\Hb_{1,\lambda}(\widebar{\cE}_{\mathsf{WL}}(n,d_T)))\in{\Omega}(\nicefrac{1}{\lambda^2}), \text{ for } r=\sqrt{T/(T+1)} \text{ and } n\geq\nicefrac{r^2}{\lambda^2}.
	\end{align*}
\end{theorem}
\begin{proof}
	The upper bounds follow from the general upper bound described earlier. For the lower bound, we show that for even $n\geq r^2/\lambda^2$, there exist
	$m=\lfloor r^2/\lambda^2\rfloor$ graphs $G_1,\dotsc, G_m$ in $\cG_n$ such that the vectors $\vec b_i\coloneqq\phi^{(1)}_{\textsf{WL}}(G_i)$ and
	$\widebar{\vec b}_i\coloneqq\widebar{\phi}^{(1)}_{\textsf{WL}}(G_i)$ satisfy the assumptions of \cref{APP:lem:VC_lower_F}.  Indeed, we can simply consider $G_i$ to be an $(i-1)$-regular graph
	of order $n$; see \cref{lem:regular}. We break up the feature vectors into two parts: a one-dimensional part corresponding to the information related to the initial color and the remaining part containing all other information. We remark that for the \wlone{} and for unlabeled graphs, all vertices have the same initial color. The interesting information is contained in the second part. If we inspect the \wlone{} feature vectors, excluding the initial colors, for $T=1$ of $G_i$, we obtain $(0,\dotsc,\underbrace{n}_{\text{pos }i},\dotsc,0)$ in the unnormalized case, and
	$\frac{1}{\sqrt{1+1}n}(0,\dotsc,\underbrace{n}_{\text{pos }i},\dotsc,0)$
	in the normalized case. It is readily verified that $\vec b_i^{(2)}\coloneqq (0,\dotsc,\underbrace{n}_{\text{pos }i},\dotsc,0)$ and $\vec b_i^{(1)}$ being the remaining initial colors
	are vectors satisfying the assumptions of \cref{APP:lem:VC_lower_F} in the unnormalized case. For larger $T$, $\vec b_i^{(2)}$ is $\phi^{(1)}_{\textsf{WL}}(G_i)$ except for the initial colors. Note that $\|\vec b_i^{(2)}\| = \sqrt{T}n = r$ and $\|\vec b_i\| = \sqrt{T+1}n \colon= r'$.  For the normalized case, one simply needs to rescale with $\nicefrac{1}{r'}$. Note that for $T>0$, $\nicefrac{1}{2}\leq\nicefrac{r^2}{r'^2}<1$. This implies a lower bound of $\Omega(\frac{r^2}{\lambda^2})$ in the unnormalized case, and $\Omega(\frac{r^2}{\lambda^2r'^2})=\Omega(\frac{1}{\lambda^2})$ in the normalized case.

	So far, we assumed $n$ to be even. For odd $n$, there is a slight technicality in that we can construct all $r$-regular graphs where $r$ is even, i.e., we can construct $\nicefrac{n+1}{2}$ regular graphs. Analogously this means for odd $n$ and $\nicefrac{n+1}{2}\geq r^2/\lambda^2$, which is equivalent to $n \geq \nicefrac{2r^2}{\lambda^2}-1$, by a slight variant of \cref{APP:lem:VC_lower_F} this implies a lower bound of $\Omega(\nicefrac{2r^2}{\lambda^2}-1)=\Omega(\nicefrac{r^2}{\lambda^2})$. Analogous to the normalized case can be considered for odd $n$ and results in the same bound, which proves the desired result.
\end{proof}
\begin{theorem}
	\label{APP:thm:VCWLF}
	Let $\cF$ be a finite set of graphs.
	For any $T, \lambda>0$, we have,
	\begin{align*}
		 & \vcdim(\Hb_{\sqrt{T+1}n,\lambda}(\cE_{\mathsf{WL}, \cF}(n,d_T))) \in{\Omega}(\nicefrac{r^2}{\lambda^2}), \text{ for } r=\sqrt{T}n  \text{ and } n\geq \nicefrac{r^2}{\lambda^2}  \\
		 & \vcdim(\Hb_{1,\lambda}(\widebar{\cE}_{\mathsf{WL}, \cF}(n,d_T)))\in{\Omega}(\nicefrac{1}{\lambda^2}), \text{ for } r=\sqrt{T/(T+1)} \text{ and } n\geq\nicefrac{r^2}{\lambda^2}.
	\end{align*}
\end{theorem}
\begin{proof}
	This proof is analogous to the proof of~\cref{APP:thm:VCWL}. Note that in the proof above, we can choose the regular graphs such that all vertices in one graph are in the same orbit; see \cref{lem:regular}. This implies that if one vertex is colored according to $\cF$, all vertices are colored in the same color, and the feature vectors $\phi^{(1)}_{\textsf{WL}, \cF}(G_i)$ look exactly as described before, implying the result.
\end{proof}

A careful reader might wonder why we did not consider the initial colors in the proofs above. In the \wlone{}-case, the initial colors are the same for all graphs in $\cG_n$, i.e.,  the \wlone{} feature vectors take the form $(n, \dots)$. We could leverage this to reduce the radius of the hypothesis class slightly. However, when considering the \wlonef{}-case, the graphs in $\cF$ change the initial colors. Because of our regular graph construction from \cref{lem:regular}, all nodes within one graph share the same color, determined by a subset $F\subseteq\cF$, where $F$ contains all graphs that are subgraphs of the regular graph in question. Hence, $2^{|\cF|}$ possible initial colorings of graphs in $\cG_n$ exists. Also, in both cases, our regular graphs are not necessarily orthogonal in the dimensions of these initial colors. Therefore, we disregarded them in the constructions of $\vec w$ above.

\begin{theorem}
	\label{APP:thm:VCWLOA}
	For any $T, \lambda>0$, we have,
	\begin{align*}
		 & \vcdim(\Hb_{\sqrt{(T+1)n},\lambda}(\cE_{\mathsf{WLOA}}(n,d_T))) \in{\Omega}(\nicefrac{r^2}{\lambda^2}), \text{ for } r=\sqrt{Tn}  \text{ and } n\geq \nicefrac{r^2}{\lambda^2}, \\
		 & \vcdim(\Hb_{1,\lambda}(\widebar{\cE}_{\mathsf{WLOA}}(n,d_T)))\in{\Omega}(\nicefrac{1}{\lambda^2}), \text{ for } r=\sqrt{T/(T+1)} \text{ and } n\geq\nicefrac{r^2}{\lambda^2}.
	\end{align*}
\end{theorem}
\begin{proof}
	This proof is analogous to the proof of~\cref{APP:thm:VCWL} except $\|\phi^{(1)}_{\textsf{WLOA}}(G_i)\| = \sqrt{(T+1)n} \eqqcolon r'$ and $\|e_i\| = \sqrt{Tn} = r$. This implies a lower bound of $\Omega(\frac{r^2}{\lambda^2})$ in the unnormalized case, and $\Omega(\frac{r^2}{\lambda^2r'^2})=\Omega(\frac{1}{\lambda^2})$ in the normalized case, as desired.
\end{proof}

\begin{theorem}
	\label{APP:thm:VCWLOAF}
	Let $\cF$ be a finite set of graphs. For any $T, \lambda>0$, we have,
	\begin{align*}
		 & \vcdim(\Hb_{\sqrt{(T+1)n},\lambda}(\cE_{\mathsf{WLOA}, \cF}(n,d_T))) \in{\Omega}(\nicefrac{r^2}{\lambda^2}), \text{ for } r=\sqrt{Tn}  \text{ and } n\geq \nicefrac{r^2}{\lambda^2}, \\
		 & \vcdim(\Hb_{1,\lambda}(\widebar{\cE}_{\mathsf{WLOA}, \cF}(n,d_T)))\in{\Omega}(\nicefrac{1}{\lambda^2}), \text{ for } r=\sqrt{T/(T+1)} \text{ and } n\geq\nicefrac{r^2}{\lambda^2}.
	\end{align*}
\end{theorem}
\begin{proof}
	This proof is analogous to the proofs of~\cref{APP:thm:VCWLF,APP:thm:VCWLOA}.
\end{proof}

Note that the upper bound and the previous theorems on lower bounds imply tight bounds in $\cO$-notation.

\begin{corollary}[{\cref{thm:VCWL} in the main paper}]
	For any $T, \lambda>0$, we have,
	\begin{align*}
		 & \vcdim(\Hb_{r,\lambda}(\cE_{\mathsf{WL}}(n,d_T))) \in{\Theta}(\nicefrac{r^2}{\lambda^2}), \text{ for } r=\sqrt{T+1}n  \text{ and } n\geq \nicefrac{r^2}{\lambda^2},          \\
		 & \vcdim(\Hb_{1,\lambda}(\widebar{\cE}_{\mathsf{WL}}(n,d_T)))\in{\Theta}(\nicefrac{1}{\lambda^2}), \text{ for } r=\sqrt{T/(T+1)}  \text{ and } n\geq\nicefrac{r^2}{\lambda^2}.
	\end{align*}
\end{corollary}
\begin{corollary}[{\cref{thm:VCWLF} in the main paper}]
	Let $\cF$ be a finite set of graphs.
	For any $T, \lambda>0$, we have,
	\begin{align*}
		 & \vcdim(\Hb_{r,\lambda}(\cE_{\mathsf{WL}, \cF}(n,d_T))) \in{\Theta}(\nicefrac{r^2}{\lambda^2}), \text{ for } r=\sqrt{T+1}n  \text{ and } n\geq \nicefrac{r^2}{\lambda^2}            \\
		 & \vcdim(\Hb_{1,\lambda}(\widebar{\cE}_{\mathsf{WL}, \cF}(n,d_T)))\in{\Theta}(\nicefrac{1}{\lambda^2}), \text{ for } r=\sqrt{T/(T+1)}  \text{ and }  n\geq\nicefrac{r^2}{\lambda^2}.
	\end{align*}
\end{corollary}
\begin{corollary}[{\cref{thm:VCWLOA} in the main paper}]
	For any $T, \lambda>0$, we have,
	\begin{align*}
		 & \vcdim(\Hb_{r,\lambda}(\cE_{\mathsf{WLOA}}(n,d_T))) \in{\Theta}(\nicefrac{r^2}{\lambda^2}), \text{ for } r=\sqrt{(T+1)n}  \text{ and } n\geq \nicefrac{r^2}{\lambda^2},         \\
		 & \vcdim(\Hb_{1,\lambda}(\widebar{\cE}_{\mathsf{WLOA}}(n,d_T)))\in{\Theta}(\nicefrac{1}{\lambda^2}), \text{ for } r=\sqrt{T/(T+1)}  \text{ and }  n\geq\nicefrac{r^2}{\lambda^2}.
	\end{align*}
\end{corollary}
\begin{corollary}[{\cref{thm:VCWLOAF} in the main paper}]
	Let $\cF$ be a finite set of graphs. For any $T, \lambda>0$, we have,
	\begin{align*}
		 & \vcdim(\Hb_{r,\lambda}(\cE_{\mathsf{WLOA}, \cF}(n,d_T))) \in{\Theta}(\nicefrac{r^2}{\lambda^2}), \text{ for } r=\sqrt{(T+1)n}  \text{ and } n\geq \nicefrac{r^2}{\lambda^2},        \\
		 & \vcdim(\Hb_{1,\lambda}(\widebar{\cE}_{\mathsf{WLOA}, \cF}(n,d_T)))\in{\Theta}(\nicefrac{1}{\lambda^2}), \text{ for } r=\sqrt{T/(T+1)}  \text{ and } n\geq\nicefrac{r^2}{\lambda^2}.
	\end{align*}
\end{corollary}

\subsubsection{Colored margin bounds}\label{sec:colored_margins}
Given $T\geq 0$ and $C\subseteq \Nb$, we say that a graph $G$ has \new{color complexity} $(C,T)$ if the first $T$ iterations of \wlone{} assign colors to $G$ in the set $C$.
Let $\cG_{C,T}$ be the class of all graphs of color complexity $(C,T)$. We note that $\cG_{C,T}$ possibly contains infinitely many graphs. Indeed, if $C$ corresponds to the color assigned by \wlone{} to degree two nodes, then $\cG_{C,T}$ contains all $2$-regular graphs.

Let $\cE(C,T,d)$ be a class of graph embedding methods consisting of mappings from $\cG_{C,T}$ to $\Rb^d$. Separability is lifted to the setting by considering the set of partial concepts defined on $\cG_{C,T}$, as follows
\begin{align*}
	\Hb_{r,\lambda}(\cE(C,T,d)) \coloneqq \Big\{
	h\in\{0,1,\star\}^{\cG_{C,T}} \mathrel{\Big|}\,\forall\, & G_1,\dotsc,G_s\in\mathsf{supp}(h)\colon \\
	(                                                        & G_1,h(G_1)),\dotsc,(G_s,h(G_s))
	\text{ is $(r,\lambda)$-$\cE(n,d)$-separable} \Big\}.\nonumber
\end{align*}
Let $\widebar{\cE}_{\mathsf{WL}}(C,T,d)$ be the class of embeddings corresponding to the normalized \wlone{} kernel, i.e., $\widebar{\cE}_{\mathsf{WL}}(C,T,d)\coloneqq\{G\mapsto \widebar{\phi^{(T)}_{\textsf{WL}}}(G)\mid G\in\cG_{C,T}\}$. We note that $d$ is a constant depending on $|C|$ and $T$ we denote this constant by $d_{C,T}$. An immediate consequence of the proof of~\cref{thm:VCWL} is that we can obtain a margin-bound for infinite classes of graphs.
\begin{corollary}
	For any $T > 0$, $C\subseteq \Nb$, and $\lambda>0$, such that $\cG_{C,T}$ contains all regular graphs of degree $0,1,\ldots,\nicefrac{r^2}{\lambda^2}$, for $r=\sqrt{T/(T+1)}$,
	we have
	\begin{equation}\vcdim(\Hb_{1,\lambda}(\widebar{\cE}_{\mathsf{WL}}(C,T,d_{C,T})))\in{\Theta}(\nicefrac{1}{\lambda^2}).  \tag*{\qed}
	\end{equation}
\end{corollary}

\subsection{Margin-based bounds on the VC dimension of MPNNs and more expressive architectures}

We now lift the above results for the \wlone{} kernel to MPNNs. To prove~\cref{prop:matchingvc_mpnn}, we show that
$\cE_{\textsf{MPNN}}(n,d,T)$
contains $\cE_{\mathsf{WL}}(n,d_T)$. Thereto, the following result shows that MPNNs can compute the \wlone{} feature vector.
\begin{proposition}\label{thm:simulating}
	Let $\cG_n$ be the set of $n$-order graphs and let $S \subseteq \cG_n$. Then, for all $T \geq 0$, there exists a sufficiently wide $T$-layered simple MPNN architecture $\smmpnn_n \colon S \to \Rb^d$, for an appropriately chosen $d > 0$, such that, for all $G \in S$,
	\begin{equation*}
		\smmpnn_n(G) =  \phi^{(T)}_{\textsf{WL}}(G).
	\end{equation*}
\end{proposition}
\begin{proof}
	The proof follows the construction outlined in the proof of \citep[Proposition 2]{Mor+2023}. Let $s \coloneqq |S|$. Hence, $sn$ is an upper bound for the number of colors computed by \wlone{} over all $s$ graphs in one iteration.

	Now, by~\citet[Theorem 2]{Mor+2019}, there exists an MPNN architecture with feature dimension (at most) $n$ and consisting of $t$ layers such that for each graph $G \in S$ it computes \wlone-equivalent vertex features $\fb_v^{(t)}$ in $\Rb^{1 \times n}$ for $v \in V(G)$. That is, for vertices $v$ and $w$ in $V(G)$ it holds that
	\begin{equation*}
		\fb_{v}^{(t)} = \fb_{w}^{(t)} \iff C^1_{T}(v) = C^1_{T}(w).
	\end{equation*}
	We note, by the construction outlined in the proof of~\citet[Theorem 2]{Mor+2019}, that $\fb_{v}^{(t)}$, for $v \in V(G)$, is defined over the rational numbers. We further note that we can construct a single MPNN architecture for all $s$ graphs by applying~\citet[Theorem 2]{Mor+2019} over the disjoint union of the graphs in $S$. This increases the width from $n$ to $sn$. We now show how to compute the \wlone{} feature vector of a single iteration $t$. The overall feature vector can be obtained by (column-wise) concatenation over all layers.

	Since the vertex features are rational, there exists a number $M$ in $\Nbb$ such that $M \cdot \fb_{v}^{(t)}$ is in $\Nb^{1 \times sn}$ for all $v \in V(G)$ and  $G \in S$, i.e., a vector over $\Nb$. Now, let
	\begin{equation*}
		\mW' = \begin{bmatrix}
			K^{sn-1} & \cdots & K^{sn-1} \\
			\vdots   & \cdots & \vdots   \\
			K^{0}    & \cdots & K^{0}
		\end{bmatrix}
		\in \Nbb^{sn \times 2sn},
	\end{equation*}
	for a sufficiently large $K > 0$, then $\mathbf{k}_{v} \coloneqq M \cdot \fb_{v}^{(t)} \mW'$, for vertex $v \in V(G)$ and graph $G \in S$, computes a vector $\mathbf{k}_{v}$ in $\Nbb^{2sn}$ containing $2sn$ occurrences of a natural number uniquely encoding the color of the vertex $v$. We next turn $\mathbf{k}_{v}$ into a one-hot encoding. More specifically, we define
	\begin{equation*}
		\mathbf{h}_{v}'=\mathsf{lsig}(\mathbf{k}_{v} \circ (\mathbf{w}'')^\tran+\mathbf{b}),
	\end{equation*}
	where $\circ$ denotes element-wise multiplication, with $\mathbf{w}''=(1,-1,1,-1,\dotsc,1,-1)\in\Rb^{2sn}$ and $\mathbf{b}=(-c_1-1,c_1+1,-c_2-1,c_2+1, \dotsc, -c_{sn}-1,c_{sn}+1)\in\Rb^{2sn}$with $c_i$ the number encoding the $i$th color under \wlone{} at iteration $t$ on the set $S$.  We note that for odd $i$,
	\begin{equation*}
		(h_{v}')_i\coloneqq \mathsf{lsig}(C^1_t(v)-c_i-1)=\begin{cases}
			1 & C^1_t(v)\geq c_i  \\
			0 & \text{otherwise}.
		\end{cases}
	\end{equation*}
	and for even $i$,
	\begin{equation*}
		(h_{v}')_i\coloneqq \mathsf{lsig}(-C^1_t(v)+c_i+1)=\begin{cases}
			1 & C^1_t(v) \leq c_i \\
			0 & \text{otherwise}.
		\end{cases}
	\end{equation*}
	In other words, $((h_{v}')_i,(h_{v}')_{i+1})$ are both $1$ if and only if $C^1_t(v)=c_i$. We thus obtain one-hot encoding of the color $C^1_t(v)$ by
	combining $((h_{v}')_i,(h_{v}')_{i+1})$ using an ``AND'' encoding (e.g., $\mathsf{lsig}(x+y-1)$) applied to pairs of consecutive entries in $\mathbf{h}_{v}'$.
	That is,
	\begin{equation*}
		\hb_{v}\coloneqq\mathsf{lsig}\left(
		\hb_{v}'\cdot\begin{pmatrix}
			1      & 0      & \cdots & 0      \\
			1      & 0      & \cdots & 0      \\
			0      & 1      & \cdots & 0      \\
			0      & 1      & \cdots & 0      \\
			\vdots & \vdots & \ddots & \vdots \\
			0      & 0      & \cdots & 1      \\
			0      & 0      & \cdots & 1
		\end{pmatrix}- (1, 1,\dotsc,1)
		\right)\in\Rb^{sn}.
	\end{equation*}
	We obtain the overall \wlone{} vector by row-wise summation and concatenation over all layers. We remark that, for a single iteration, the maximal width of the whole construction is $2sn$.
\end{proof}
By the above proposition, MPNNs of sufficient width can compute the \wlone{} feature vectors. Moreover, the normalization can be included in the MPNN computation. Hence, we can prove the lower bound by simulating the proof of~\cref{thm:VCWL}. The upper bound follows by the same arguments as described at the beginning of~\cref{subsec:VCdim}. The above result can be easily extended to the \wlonef, implying~\cref{prop:matchingvc_mpnn_f}.
\begin{corollary}\label{thm:simulating_f}
	Let $\cG_n$ be the set of $n$-order graphs, let $S \subseteq \cG_n$, and let $\cF$ be a set of graphs. Then, for all $T \geq 0$, there exists a sufficiently wide $T$-layered MPNN architecture $\smmpnn_n \colon S \to \Rb^d$, for an appropriately chosen $d > 0$, such that, for all $G \in S$,
	\begin{equation*}
		\smmpnn_n(G) =  \phi^{(T)}_{\text{\wlonef}}(G).
	\end{equation*}
\end{corollary}
\begin{proof}[Proof sketch]
	By definition of the \wlonef{}, the algorithm is essentially the \wlone{} operating on a specifically vertex-labeled graph. Since~\citet[Theorem 2]{Mor+2019} also works for vertex-labeled graphs, the proof technique for~\cref{thm:simulating} can be straightforwardly lifted to the \wlonef.
\end{proof}
We can also extend~\cref{thm:simulating} to the \wloa{} and \wloaf{}, i.e., derive an MPNN architecture that can compute \wloa's and \wloaf's feature vectors. By that, we can extend~\cref{thm:VCWLOA,thm:VCWLOAF} to their corresponding MPNN versions.
\begin{proposition}\label{thm:simulating_wloa}
	Let $\cG_n$ be the set of $n$-order graphs and let $S \subseteq \cG_n$. Then, for all $T \geq 0$, there exists a sufficiently wide $T$-layered MPNN architecture $\smmpnn_n \colon S \to \Rb^d$, for an appropriately chosen $d > 0$, such that, for all $G \in S$,
	\begin{equation*}
		\smmpnn_n(G) =  \phi^{(T)}_{\textsf{WLOA}}(G).
	\end{equation*}
\end{proposition}
\begin{proof}
	By~\cref{thm:simulating}, there exists a $T$-layered MPNN architecture $\smmpnn_n \colon S \to \Rb^d$, for an appropriately chosen $d > 0$, such that, for all $G \in S$,
	\begin{equation*}
		\smmpnn_n(G) =  \phi^{(T)}_{\text{\wlone}}(G).
	\end{equation*}
	We now show how to transform $\phi^{(T)}_{\text{\wlone}}(G)$ into $ \phi^{(T)}_{\textsf{WLOA}}(G)$. We show the transformation for a single iteration $t \leq T$, i.e., transforming  $\phi_{t,\text{\wlone}}(G)$ into $\phi_{t,\text{\wloa}}(G)$. Let $C$ denote the number of colors at iteration $t$ of the \wlone{} over all $|S|$ graphs. Since $n$ is finite, $C$ is finite as well. That is, $\phi_{t,\text{\wlone}}(G)$ has $C$ entries. Hence, the number of components for $\phi_{t,\textsf{WLOA}}(G)$ is at most $Cn$. By multiplying $\phi_{t,\text{\wlone}}(G)$ with an appropriately chosen matrix $\vec{M} \in \{ 0,1 \}^{C \times Cn}$, we get a vector $\vec{r} \in \Rb^{Cn}$, where each entry of $\phi_{t,\text{\wlone}}(G)$ is repeated $n$ times.
	Specifically,
	\begin{equation*}
		\vec{M} \coloneqq
		\begin{pmatrix}
			1      & 0      & \cdots & 0      \\
			1      & 0      & \cdots & 0      \\
			\vdots & \vdots & \cdots & 0      \\
			1      & 0      & \cdots & 0      \\
			0      & 1      & \cdots & 0      \\
			0      & 1      & \cdots & 0      \\
			\vdots & \vdots & \ddots & 0      \\
			0      & 0      & \cdots & 1      \\
			0      & 0      & \cdots & 1      \\
			0      & 0      & \cdots & \vdots \\
			0      & 0      & \cdots & 1
		\end{pmatrix}
		\in  \{ 0,1 \}^{C \times Cn}.
	\end{equation*}
	Now let
	\begin{equation*}
		\vec{b} \coloneqq (1,2,\dotsc, n, 1,2,\dotsc, n, \dotsc, 1,2,\dotsc, n) \in \Rb^{Cn} \quad\text{ and }\quad  \vec{r'} \coloneqq \sign{(\vec{r} - \vec{b})}.
	\end{equation*}
	Observe that $\vec{r'} = \phi_{t,\text{\wloa}}(G)$, implying the result
\end{proof}

In a similar way as for~\cref{thm:simulating_f}, we can lift the above result to the \wloaf.
\begin{corollary}\label{thm:simulating_f_wloa}
	Let $\cG_n$ be the set of $n$-order graphs, let $S \subseteq \cG_n$, and let $\cF$ be a set of graphs. Then, for all $T \geq 0$, there exists a sufficiently wide $T$-layered MPNN architecture $\smmpnn_n \colon S \to \Rb^d$, for an appropriately chosen $d > 0$, such that, for all $G \in S$,
	\begin{equation*}
		\smmpnn_n(G) =  \phi^{(T)}_{\text{\wloaf}}(G). \tag*{\qed}
	\end{equation*}
\end{corollary}

\subsection{Increased separation power of the \texorpdfstring{\wlonef}{1-WLF}}

We now prove how more expressive architectures help to separate the data.

\begin{lemma}[{\cref{thm:seperator} in the main paper}]
	\label{APP:thm:seperator}
	For every $n \geq 6$, there exists a pair of non-isomorphic $n$-order graphs $(G_n$, $H_n)$ and a graph $F$
	such that, for $\cF\coloneqq\{F\}$ and for all number of rounds $T \geq 0$, it holds that
	\begin{align*}
		\norm[\bigg]{\widebar{\phi^{(T)}_{\textsf{WL}}}(G_n) - \widebar{\phi^{(T)}_{\textsf{WL}}}(H_n)} = 0, \quad\text{ and }\quad
		\norm[\bigg]{\widebar{\phi^{(T)}_{\textsf{WL}, \cF}}(G_n) - \widebar{\phi^{(T)}_{\textsf{WL},\cF}}(H_n)} = \sqrt{2}.
	\end{align*}
\end{lemma}
\begin{proof}
	Let $G_n \coloneqq C_n$, a cycle on $n$ vertices. Further, let $H_n \coloneqq C_{\lceil n/2 \rceil}  \dot{\cup}  C_{\lfloor n/2 \rfloor}$, a disjoint union of cycles on $\lceil n/2 \rceil$ and $\lfloor n/2 \rfloor$ vertices, respectively. Since the graphs $G_n$ and $H_n$ are regular, \wlone{} cannot distinguish them. Hence, the distances of the two feature vectors $\widebar{\phi^{(T)}_{\textsf{WL}}}(G_n)$ and  $\widebar{\phi^{(T)}_{\textsf{WL}}}(H_n)$ is zero. Moreover, by setting $\cF \coloneqq \{ C_{\lfloor n/2 \lfloor} \}$, \wlonef{} reaches the stable coloring at iteration $t = 0$ on both graphs and distinguishes the two graphs. In addition, the feature vectors $\widebar{\phi^{(T)}_{\textsf{WL},\cF}}(G_n)$ and $\widebar{\phi^{(T)}_{\textsf{WL},\cF}}(H_n)$ are orthonormal for all choices of $T \geq 0$, implying the result.
\end{proof}

\begin{proposition}[{\cref{thm:seperator_mpnn} in the main paper}]\label{thm:seperator_mpnn_APP}
	For every $n \geq 6$, there exists a pair of non-isomorphic $n$-order graphs $(G_n$, $H_n)$ and set of graphs $\cF$ of cardinality one, such that, for all number of layers $T \geq 0$, and widths $d > 0$, and all $m \in  \MPNN_{\mathsf{mlp}}(d,T)$,
	it holds that
	\begin{align*}
		\norm[\bigg]{m(G_n) - m(H_n)} = 0,
	\end{align*}
	while for sufficiently large $d>0$, there exists an $\widehat m \in  \MPNN_{\mathsf{mlp},\cF}(d,T)$, such that
	\begin{align*}
		\norm[\bigg]{\widehat{m}(G_n) - \widehat{m}(H_n)} = \sqrt{2}.
	\end{align*}
\end{proposition}
\begin{proof}
	We use the same graphs as in the proof of~\cref{APP:thm:seperator}.
	The first statement is a direct implication of~\citet[Theorem 1]{Mor+2019}, i.e., any MPNN is upper bounded by the \wlone{} in distinguishing vertices. The second statement follows from~\cref{thm:simulating,thm:simulating_f}. That is, an \tMPNNF{} architecture of sufficient width exactly computes the \wlonef\ feature vectors for the graphs of \cref{APP:thm:seperator}, as shown in the proof of~\cref{thm:simulating}.
\end{proof}

\begin{proposition}[{\cref{thm:separability} in the main paper}]
	\label{APP:thm:separability}
	For every $n \geq 10$, there exists a set of pair-wise non-isomorphic (at most) $n$-order graphs $S$, a concept $c \colon S \to \{ 0,1 \}$, and
	a graph $F$, such that the graphs in the set $S$,
	\begin{enumerate}
		\item are pair-wise distinguishable by \wlone{} after one round,
		\item are \emph{not} linearly separable under the normalized \wlone{} feature vector $\widebar{\phi^{(T)}_{\textsf{WL}}}$, concerning the concept $c$, for any $T \geq 0$,
		\item and are linearly separable under the normalized \wlonef{} feature vector $\widebar{\phi^{(T)}_{\textsf{WL},\cF}}$,  concerning the concept $c$ and , for all $T \geq 0$, where $\cF\coloneqq \{F\}$.
	\end{enumerate}
	Moreover, the results also work for the unnormalized feature vectors.
\end{proposition}
\begin{proof}
	We first outline the construction for the set $S$ with $|S| = 4$. However, the construction can easily be generalized to larger cardinalities. Concretely, we construct the graphs $G_1, \dotsc, G_4$. For odd $i$, the graph $G_i$ consists of the disjoint union of $i$ isolated vertices and a disjoint union of cycles on $\lceil n/2 \rceil -2$ and $\lfloor n/2 \rfloor-2$ vertices. Set $c(G_i) \coloneqq 0$. For even $i$, the graph $G_i$ consists of the disjoint union of $i$ isolated vertices and a cycle on $n-4$ vertices. Set $c(G_i) \coloneqq 1$.

	We proceed by showing items 1 to 3. Since the order of the graphs $G_1$ to $G_4$ is pair-wise different \wlone{} pair-wise distinguishes the graphs, showing item 1. We proceed with item 2. Since the disjoint union of cycles on $\lceil n/2 \rceil -2 $ and $\lfloor n/2 \rfloor -2$ vertices are $2$-regular, the \wlone{} cannot distinguish them. Therefore, the four \wlone{} feature vectors have the following forms,
	\begin{align*}
		 & \phi^{(T)}_{\textsf{WL}}(G_1) = (n-4,1,n-4,1,\dotsc, n-4, 1), \\
		 & \phi^{(T)}_{\textsf{WL}}(G_2) = (n-4,2,n-4,2,\dotsc, n-4, 2), \\
		 & \phi^{(T)}_{\textsf{WL}}(G_3) = (n-4,3,n-4,3,\dotsc, n-4, 3), \\
		 & \phi^{(T)}_{\textsf{WL}}(G_4) = (n-4,4,n-4,4,\dotsc, n-4, 4).
	\end{align*}
	Since, for each round $t \geq 0$, the feature vectors have the form $(n-4,1)$, $(n-4,2)$, $(n-4,3)$, and $(n-4,4)$, respectively, the resulting four unnormalized \wlone{} feature vectors are co-linear and, by the construction, of the concept $c$, they are not linearly separable.

	Note that, for each round $t \geq 0$, the feature vectors only differ by at most 3 in the second components. Hence, their respective $\ell_2$ norms are controlled by $n$, and their respective $\ell_2$ norms are also close. For a single iteration, when projecting the feature vectors onto the $1$-dimensonal unit sphere, by dividing by their respective $\ell_2$, the lexicographic order is preserved on the unit sphere, making them not linearly separable. For $t\geq 1$, notice that the vectors are just concatenations of the above. We can write the original vectors for $t=0$ as $\vec g_1$, $\vec g_2$, $\vec g_3$, and $\vec g_4$ and for larger $t>0$ as $[\vec g_1,\dotsc , \vec g_1], \dotsc, [\vec g_4,\dotsc, \vec g_4]$. Assume, for $t>0$, that the aforementioned vectors are linearly separable by a hyperplane $0 = \vec w^\tran \vec x+b$, then, by definition, we can decompose $\vec{w}=[ \vec w_0,  \vec w_1, \dotsc,  \vec w_t]$ such that $(\textsf{sign}(\sum_{i=0}^{T} \vec w_i^{\tran} \vec g_i+b)+1)/2 = c(\vec g_j)$ and by construction $\sum_{i=0}^{T} \vec w_i^{\tran} \vec g_j+b = (\sum_{i=0}^{T}\vec w_i)^\tran \vec g_j+b$. Thus, such $\vec{w}$ and $b$ would verify that $\vec{g}_1, \dotsc, \vec{g}_4$ are linearly separable, which is impossible. This implies that even for larger $t>0$ $\phi^{(T)}_{\textsf{WL}}(G_1), \dots, \phi^{(T)}_{\textsf{WL}}(G_4)$ are not linearly separable.

	For item 3, we set $\cF \coloneqq \{ C_{n-4} \}$, the cycle on $n-4$ vertices. Observe that the \wlonef{} feature vectors at $T=0$ for all four graphs have three components. Concretely,
	\begin{align*}
		\phi^{(0)}_{\textsf{WL},\cF}(G_1) = (1,0,n-4), \\
		\phi^{(0)}_{\textsf{WL},\cF}(G_2) = (2,n-4,0), \\
		\phi^{(0)}_{\textsf{WL},\cF}(G_3) = (3,0,n-4), \\
		\phi^{(0)}_{\textsf{WL},\cF}(G_4) = (4,n-4,0).
	\end{align*}
	Further, note that, for all four graphs, \wlonef{} reaches the stable coloring at $T = 0$. Hence, for all $T \geq 0$, the four vectors are
	linearly separable concerning the concept $c$.  Further, by dividing by their respective $\ell_2$ norms, this property is preserved. That is, for graphs with class label $0$, the first and third graph, third component of the \wlone{} feature vectors are equal to $(n-4)$ while for the other two graphs, this component is $0$. Hence, we can easily find a vector $\vec{w} \in \Rb^3$ that linearly separates the normalized \wlone{} feature vectors.
\end{proof}

\begin{proposition}[{\cref{thm:fsep} in the main paper}]\label{APP:thm:fsep}
	Let $n \geq 6$  and let $\cF$ be a finite set of graphs. Further, let $c \colon \cG_n \to \{ 0,1 \}$ be a concept such that, for all $T \geq 0$, the graphs are \emph{not} linearly separable under the normalized \wlone{} feature vector $\widebar{\phi^{(T)}_{\textsf{WL}}}$, concerning the concept $c$. Further, assume that for all graphs $G \in \cG_n$ for which $c(G) = 0$, it holds that there is at least one vertex $v \in V(G)$ such it is contained in a subgraph of $G$ that is isomorphic to a graph in the set $\cF$, while no such vertices exist in graphs $G$ for which $c(G) = 1$.  Then the graphs are linearly separable under the normalized \wlonef{} feature vector $\widebar{\phi^{(T)}_{\textsf{WL},\cF}}$, concerning the concept $c$.
\end{proposition}
\begin{proof}
	By assumption, for graphs $G$ with $c(G) = 0$ it holds that there exists an index $i \geq 0$ such that $\phi^{(0)}_{\textsf{WL},\cF}(G)_i \neq 0$ while for all graphs $H$ with $c(H) = 1$ it holds that $\phi^{(0)}_{\textsf{WL},\cF}(H)_i = 0$. Without loss of generality, assume this is the case for $i=0$. Hence, for all $T \geq 0$, we can find a vector $\vec{w} \coloneqq (1, 0, \dotsc, 0) \in \Rb^d$  with an appropriate number of components $d$ and $C > 0$ such that
	\begin{equation*}
		\Big\langle \vec{w}, \phi^{(T)}_{\textsf{WL},\cF}(G) \Big\rangle =
		\begin{cases}
			> C, & \text{ if } c(G) = 0, \\
			< C, & \text{ if } c(G) = 1.
		\end{cases}
	\end{equation*}
	Hence, the data is linearly separable by the \wlonef{} for $T \geq 0$, showing the result.
\end{proof}

\subsection{Results on shrinking the margin}

Here, we prove negative results, showing that using more expressive architecture can also decrease the margin.

\begin{proposition}[{\cref{thm:shrink_with_more_power} in the main paper}]
	\label{APP:thm:shrink_with_more_power}
	For every $n \geq 10$, there exists a pair of $2n$-order graphs $(G_{n}, H_{n})$ and a graph $F$,
	such that, for $\cF\coloneqq\{F\}$ and for all number of rounds $T > 0$, it holds that
	\begin{align*}
		\norm[\bigg]{\widebar{\phi^{(T)}_{\textsf{WL}, \cF}}(G_n) - \widebar{\phi^{(T)}_{\textsf{WL},\cF}}(H_n)} < \norm[\bigg]{\widebar{\phi^{(T)}_{\textsf{WL}}}(G_n) - \widebar{\phi^{(T)}_{\textsf{WL}}}(H_n)}.
	\end{align*}
\end{proposition}
\begin{proof}
	Let $G_n$ be the disjoint union of a complete graph on $n$ vertices $K_{n}$ and  $n$ isolated vertices. Further, let $H_n$ be the disjoint union of the complete graph on three vertices $K_3$, the cycle on $(n-3)$ vertices $C_{n-3}$, and $n$ isolated vertex. By construction and definition of the \wlone{} feature vector, for $T > 0$,
	\begin{align*}
		\phi^{(T)}_{\textsf{WL}}(G_n) & = (2n,n,0,n, \dotsc, n,0,n) \in \mathbb{R}^{1 + 3T}, \text{ and } \\
		\phi^{(T)}_{\textsf{WL}}(H_n) & = (2c,0,n,n, \dotsc, 0,n,n) \in \mathbb{R}^{1 + 3T}.
	\end{align*}
	Moreover, by setting $\cF \coloneqq \{ C_3 \}$ and by definition of the \wlonef{} feature vector, for $T > 0$,
	\begin{align*}
		\phi^{(T)}_{\textsf{WL},\cF}(G_n) & = (0, 2n,   n,0,0, n, \dotsc,  n,0,0, n) \in \mathbb{R}^{2 + 4T} \text{ and } \\
		\phi^{(T)}_{\textsf{WL},\cF}(H_n) & = (3,2n-3,  0,n-3,3,n, \dotsc, 0,n-3,3,n) \in \mathbb{R}^{2 + 4T}.
	\end{align*}
	Since $\sqrt{2n^2} = \sqrt{2}n > \sqrt{3\cdot 3^2 + n^2 + (n-3)^2}$ for $n \geq 10$, the statement is clear for the unnormalized \wlone{} feature vector is clear. Now, for the normalized \wlone{} feature vector, the component counting the $n$ isolated vertices in each iteration $T > 0$ will sufficiently downgrade the (constant) contribution of the \wlonef{} feature vector for iteration $0$, leading to a larger $\ell_2$ distance of the \wlone{} feature vectors for iterations $T> 0$.  That is, for the first two iterations $ \norm{\widebar{\phi^{(1)}_{\textsf{WL}}}(G_n) - \widebar{\phi^{(1)}_{\textsf{WL}}}(H_n)} = \sqrt{\nicefrac{2}{6}}$, which, is always strictly larger than $ \norm{\widebar{\phi^{(1)}_{\textsf{WL}, \cF}}(G_n) - \widebar{\phi^{(1)}_{\textsf{WL},\cF}}(H_n)}$, which can be verified by tedious calculation. The argument can be extended to iterations $t > 1$.
\end{proof}

\subsection{Results on growing the margin}

We now prove results showing that more expressive architectures often result in an increased margin.
\begin{proposition}[{\cref{thm:grow} in the main paper}]\label{APP:thm:grow}
	Let $n > 0$ and $G_n$ and $H_n$ be two connected $n$-order graphs. Further, let $\cF \coloneqq \{ F \}$
	such that there is at least one vertex in $V(G_n)$ contained in a subgraph of $G_n$ isomorphic to the graph $F$. For the graph $H_n$, no such vertices exist. Further, let $T \geq 0$ be the number of rounds to reach the stable partition of $G_n$ and $H_n$ under \wlone, and assume
	\begin{equation*}
		(\phi^{(T)}_{\textsf{WL}}(G_n), 1), (\phi^{(T)}_{\textsf{WL}}(H_n), 0) \text{ is } (r_1, \lambda_1)\text{-separable, with } \lambda_1<\sqrt{2n}.
	\end{equation*}
	Then,
	\begin{equation*}
		(\phi^{(T)}_{\textsf{WL}, \cF}(G_n), 1), (\phi^{(T)}_{\textsf{WL},\cF}(H_n), 0) \text{ is } (r_2, \lambda_2)\text{-separable, with } r_2\leq r_1 \text{ and } \lambda_2\geq\lambda_1.
	\end{equation*}
\end{proposition}
\begin{proof}
	First, by assumption, there exists a vertex $v \in V(G_n)$ that, by definition of the \wlonef, gets assigned a color at initialization of the \wlonef{} that never gets assigned to any vertex in the graphs $G_n$ and $H_n$ under the \wlone. Secondly, since the graphs $G_n$ and $H_n$ are connected, at the stable partition of $G_n$ and $H_n$ under \wlonef, it holds that
	\begin{equation*}
		\langle \phi_{\cF,T}(G_n), \phi_{\cF,T}(H_n)  \rangle = 0,
	\end{equation*}
	i.e., the feature vector of $G_n$ and $H_n$ under \wlone{} are orthogonal. Hence, the $\ell_2$ distance between the two vectors is $\sqrt{\|\phi_{\cF,T}(G_n)\|^2+\|\phi_{\cF,T}(H_n)\|^2}$, due to the pythagorean theorem. This $\ell_2$ distance is minimized for minimized norms, which, since the sum of all elements is $n$ and all elements are in $\mathbb N$, is minimized for $\|\phi_{\cF,T}(G_n)\|^2=n=\|\phi_{\cF,T}(H_n)\|^2$. Thus the above $\ell_2$ distance is at least $\sqrt{2n}$m and
	$$
		(\phi_{\cF, T}(G_n), 1), (\phi_{\cF, T}(H_n), 0) \text{ is } (n, \sqrt{2n})\text{-separable}.
	$$
	We now we just need to lower-bound by how much previous iterations can decrease this distance. Again, since $G_n$ is connected, we know that at every iteration $t < T$, one vertex in the graphs $G_n$ gets assigned a color that never gets assigned to $H_n$ under the \wlonef. Since we want to lower bound the $\ell_2$ distance between feature vectors between \wlonef{} for all iteration $i < T$, we can assume that for the graph $G_n$, the normalized \wlonef{} feature vector has the form $(n-(i+1), i+1)$, since this maximizes the $\ell_2$ norm of the vector. For the graph $H$, we can assume the form $(n, 0)$, again maximizing the vector's $\ell_2$ norm. For simplicity, we can assume that their $\ell_2$ distance is $0$. By applying~\cref{prop:concat_margin} $T$ times iteratively, it follows that
	$$
		(\phi^{(T)}_{\textsf{WL}, \cF}(G_n), 1), (\phi^{(T)}_{\textsf{WL},\cF}(H_n), 0) \text{ is } (\sqrt{T}n, \sqrt{2n})\text{-separable.}
	$$
	Due to \cref{APP:prop:WLFfiner}, we can further reduce the radius $\sqrt{T}n$ to be $r_2\leq r_1$, since the colors under \wlonef{} are finer than under \wlone{}.
	Hence, by assumption, $\lambda_2\geq\sqrt{2n}>\lambda_1$ and the result follows.
\end{proof}
Notice, that under mild assumptions due to \cref{APP:thm:VCWL} and \cref{APP:thm:VCWLF}, \cref{APP:thm:grow} implies that the VC dimension for the \wlonef{} decreases, since $\frac{r_1^2}{\lambda_1^2}>\frac{r_2^2}{\lambda_2^2}$, and thus generalization improves.

\paragraph{Growing the margin for \wlonef{} and \wloaf{}} We now look into conditions under which the \wlonef{} and \wloaf{} provably increase the margin. The first result that we can directly state that holds for both \wlonef{} and \wloaf{} is the following direct application of \cref{prop:marginincrease}.

\begin{theorem}[{\cref{thm:OAmargin} in the main paper}] \label{APP:thm:OAmargin}
	Let $(G_1,y_1),\dotsc,(G_s,y_s)$ in $\cG_n\times\{0,1\}$ be a (graph) sample that is linearly separable in the \wloa{} feature space with margin $\gamma$. If
	\begin{equation*}
		\begin{split}
			\min_{y_i\neq y_j} & \left\lVert\phi^{(T)}_{\textsf{WLOA},\cF}(G_i)-\phi^{(T)}_{\textsf{WLOA},\cF}(G_j)\right\rVert^2 - \left\lVert\phi^{(T)}_{\textsf{WLOA}}(G_i)-\phi^{(T)}_{\textsf{WLOA}}(G_j)\right\rVert^2 >                           \\
			\max_{y_i=y_j}     & \left\lVert\phi^{(T)}_{\textsf{WLOA},\cF}(G_i)-\phi^{(T)}_{\textsf{WLOA},\cF}(G_j)\right\rVert^2 - \left\lVert\phi^{(T)}_{\textsf{WLOA}}(G_i)-\phi^{(T)}_{\textsf{WLOA}}(G_j)\right\rVert^2,
		\end{split}
	\end{equation*}
	that is, the minimum increase in distances between classes is strictly larger than the maximum increase in distance within each class, then the margin $\lambda$ increases when $\cF$ is considered.
\end{theorem}
\begin{proof}
	Direct application of \cref{prop:marginincrease} by considering $\vec x_i \coloneqq \phi^{(T)}_{\textsf{WLOA}}(G_i)$ and $\tilde{\vec x}_i \coloneqq \phi^{(T)}_{\textsf{WLOA},\cF}(G_i)$.
\end{proof}
Given \cref{APP:thm:OAmargin} the question remains what kind of properties a graph sample must fulfill to satisfy \cref{APP:thm:OAmargin}. Deriving such a property for the \wlone{} seems to be rather difficult. However, for the \wloa{}, we can derive the following results.

\begin{proposition}[{\cref{prop:pairwise} in the main paper}]\label{APP:prop:pairwise}
	Given two graphs $G$ and $H$,
	\begin{equation*}
		\left\lVert\phi^{(T)}_{\textsf{WLOA},\cF}(G)-\phi^{(T)}_{\textsf{WLOA},\cF}(H)\right\rVert\geq\left\lVert\phi^{(T)}_{\textsf{WLOA}}(G)-\phi^{(T)}_{\textsf{WLOA}}(H)\right\rVert.
	\end{equation*}
\end{proposition}
\begin{proof}
	Consider \cref{prop:WLFfiner} applied to the disjoint union of $G$ and $H$, which implies that any color class $c$ in $G$ or $H$ after $t\leq T$ steps of \wlone{} gets refined to be finer, i.e., $c=\bigcup_j c_j$ for some color classes $c_j$ in $G$ or $H$ after $t$ steps when applying $\cF$. Note also that each color class is naturally partitioned into subsets in $G$ and $H$. This implies that \wloa{} can only increase the pairwise distance when $\cF$ is added. Note that for any color class $c$ shattered into $c_1, \dots, c_f$ as above,
	\begin{equation*}
		\phi_{t}(G)_{c} = \sum^f_{j=1} \phi_{\cF, t}(G)_{c_j},\quad\phi_{t}(H)_{c} = \sum^f_{j=1} \phi_{\cF, t}(H)_{c_j}.
	\end{equation*}
	This trivially implies that
	\begin{equation*}
		\sum^f_{j=1} \min(\phi_{\cF, t}(G)_{c_j}, \phi_{\cF, t}(H)_{c_j}) \leq \min(\phi_{t}(G)_{c}, \phi_{t}(H)_{c}).
	\end{equation*}
	Which, since it holds for all color classes, implies
	\begin{align*}
		k^{(T)}_{\textsf{WLOA},\cF}(G, H)                                                          & \leq k^{(T)}_{\textsf{WLOA}}(G, H)                                                      \\
		\sqrt{2Tn-2k^{(T)}_{\textsf{WLOA},\cF}(G, H)}                                              & \geq \sqrt{2Tn-2k^{(T)}_{\textsf{WLOA}}(G, H)}                                          \\
		\left\lVert\phi^{(T)}_{\textsf{WLOA},\cF}(G)-\phi^{(T)}_{\textsf{WLOA},\cF}(H)\right\rVert & \geq\left\lVert\phi^{(T)}_{\textsf{WLOA}}(G)-\phi^{(T)}_{\textsf{WLOA}}(H)\right\rVert.
	\end{align*}
	This concludes the proof.
\end{proof}

In fact, for the pairwise distances to increase strictly, we can use the following equivalences.
Let $C_\cF(c)$ be the set of colors that color $c$ under \wlone{} is split into under \wlonef{}. I.e., $\phi_{t}(G)_c = \sum_{c'\in C_\cF(c)}\phi_{\cF, t}(G)_{c'}.$

\begin{proposition}[{Implies~\cref{cor:wloa} in the main paper}]\label{APP:thm:wloa}
	The following statements are equivalent,
	\begin{enumerate}
		\item
		      $
			      \left\lVert\phi^{(T)}_{\textsf{WLOA},\cF}(G)-\phi^{(T)}_{\textsf{WLOA},\cF}(H)\right\rVert = \left\lVert\phi^{(T)}_{\textsf{WLOA}}(G)-\phi^{(T)}_{\textsf{WLOA}}(H)\right\rVert.
		      $
		\item
		      $
			      k^{(T)}_{\textsf{WLOA},\cF}(G, H)= k^{(T)}_{\textsf{WLOA}}(G, H).
		      $
		\item
		      $
			      \forall\, t\in[T]\cup\{0\}, \forall c\in \Sigma_t:\min(\phi_{t}(G)_{c}, \phi_{t}(H)_{c}) = \sum_{c'\in C_\cF(c)} \min(\phi_{\cF, t}(G)_{c'}, \phi_{\cF, t}(H)_{c'}).
		      $
		\item
		      $
			      \forall\, t\in[T]\cup\{0\}, \forall c\in \Sigma_t: \left(\phi_t(G)_{c}\geq \phi_t(H)_c \iff \forall c'\in C_\cF(c):\phi_{\cF, t}(G)_{c'} \geq \phi_{\cF,t}(H)_{c'}\right).
		      $
	\end{enumerate}
\end{proposition}
\begin{proof}
	$(1) \iff (2)$: This equivalence follows from the formulas for the kernel obtained in~\cref{subsec:morepowergrows}. Indeed
	\begin{align*}
		k^{(T)}_{\textsf{WLOA},\cF}(G, H)                                                          & = k^{(T)}_{\textsf{WLOA}}(G, H)                                                      \\
		                                                                                           & \Updownarrow                                                                         \\
		\sqrt{2Tn-2k^{(T)}_{\textsf{WLOA},\cF}(G, H)}                                              & = \sqrt{2Tn-2k^{(T)}_{\textsf{WLOA}}(G, H)}                                          \\ &\Updownarrow \text{\small \cref{subsec:morepowergrows}}\\
		\left\lVert\phi^{(T)}_{\textsf{WLOA},\cF}(G)-\phi^{(T)}_{\textsf{WLOA},\cF}(H)\right\rVert & =\left\lVert\phi^{(T)}_{\textsf{WLOA}}(G)-\phi^{(T)}_{\textsf{WLOA}}(H)\right\rVert.
	\end{align*}
	$(2) \iff (3)$: Indeed, we have
	\begin{align*}
		k^{(T)}_{\textsf{WLOA},\cF}(G, H)                                             & = k^{(T)}_{\textsf{WLOA}}(G, H)                                                                                         \\&\Updownarrow \text{\small (by definition)}\\
		\sum_{t \in [T]\cup\{0\}} \sum_{c \in \Sigma_t} \min(\phi_t(G)_c,\phi_t(H)_c) & = \sum_{t \in [T]\cup\{0\}} \sum_{c \in \Sigma_t} \sum_{c' \in C_\cF(c)} \min(\phi_{\cF, t}(G)_c',\phi_{\cF, t}(H)_c').
	\end{align*}
	Furthermore, by the proof of \cref{APP:prop:pairwise},
	$$\forall\, t\in [T]\cup\{0\}, c\in\Sigma_t: \min(\phi_t(G)_c,\phi_t(H)_c) \geq \sum_{c' \in C_\cF(c)} \min(\phi_{\cF, t}(G)_c',\phi_{\cF, t}(H)_c').$$
	Hence,
	\begin{align*}
		k^{(T)}_{\textsf{WLOA},\cF}(G, H)                                                          & = k^{(T)}_{\textsf{WLOA}}(G, H)                                             \\&\Updownarrow                                  \\
		\forall\, t\in[T]\cup\{0\}, \forall\, c\in \Sigma_t:\min(\phi_{t}(G)_{c}, \phi_{t}(H)_{c}) & = \sum_{c'\in C_\cF(c)} \min(\phi_{\cF, t}(G)_{c'}, \phi_{\cF, t}(H)_{c'}).
	\end{align*}
	$(4) \implies (3)$:
	Assume $(4)$ to be true, i.e.,
	$$
		\forall\, t\in[T]\cup\{0\}, \forall c\in \Sigma_t: \phi_t(G)_{c}\geq \phi_t(H)_c \iff \forall\, c'\in C_\cF(c):\phi_{\cF, t}(G)_{c'} \geq \phi_{\cF,t}(H)_{c'}.
	$$
	Consider any $t\in[T]\cup\{0\}$ and $c\in \Sigma_t$, and assume, without loss of generality,
	$$
		\phi_t(G)_{c}\geq \phi_t(H)_c,
	$$
	then by assumption
	$$
		\forall\, c'\in C_\cF(c) \colon \phi_{\cF, t}(G)_{c'} \geq \phi_{\cF,t}(H)_{c'}.
	$$
	This implies
	$$
		\min(\phi_{t}(G)_{c}, \phi_{t}(H)_{c}) = \phi_{t}(H)_{c},
	$$
	and
	$$
		\forall\, c'\in C_\cF(c) \colon \min(\phi_{t}(G)_{c'}, \phi_{t}(H)_{c'}) = \phi_{t}(H)_{c'},
	$$
	and thus,
	\begin{align*}
		\forall\, t\in[T]\cup\{0\}, \forall c\in \Sigma_t:\min(\phi_{t}(G)_{c}, \phi_{t}(H)_{c}) & = \phi_{t}(H)_{c}                                                           \\
		                                                                                         & = \sum_{c'\in C_\cF(c)} \phi_{\cF, t}(H)_{c'}                               \\
		                                                                                         & = \sum_{c'\in C_\cF(c)} \min(\phi_{\cF, t}(G)_{c'}, \phi_{\cF, t}(H)_{c'}).
	\end{align*}
	$(3) \implies (4)$:
	Assume $(4)$ to be false, i.e.,
	\begin{align*}
		\exists\, t\in[T]\cup\{0\}, \exists \, c\in \Sigma_t \exists\, c'\in C_\cF(c) \colon & (\phi_t(G)_{c}\geq \phi_t(H)_c\wedge \phi_{\cF, t}(G)_{c'} < \phi_{\cF,t}(H)_{c'})  \\
		\vee                                                                                 & (\phi_t(G)_{c}\leq \phi_t(H)_c\wedge \phi_{\cF, t}(G)_{c'} > \phi_{\cF,t}(H)_{c'}).
	\end{align*}
	Without loss of generality, assume $\phi_t(G)_{c}\geq \phi_t(H)_c \wedge \phi_{\cF, t}(G)_{c'} < \phi_{\cF,t}(H)_{c'}.$ This implies
	$$
		\min(\phi_{t}(G)_{c}, \phi_{t}(H)_{c}) = \phi_{t}(H)_{c},
	$$
	and
	\begin{align*}
		 & \sum_{c''\in C_\cF(c)} \min(\phi_{\cF, t}(G)_{c''}, \phi_{\cF, t}(H)_{c''})                                                                    \\
		 & = \min(\phi_{\cF, t}(G)_{c'}, \phi_{\cF, t}(H)_{c'}) + \sum_{c''\in C_\cF(c), c''\neq c'} \min(\phi_{\cF, t}(G)_{c''}, \phi_{\cF, t}(H)_{c''}) \\
		 & = \phi_{\cF, t}(G)_{c'} + \sum_{c''\in C_\cF(c), c''\neq c'} \min(\phi_{\cF, t}(G)_{c''}, \phi_{\cF, t}(H)_{c''})                              \\
		 & \leq  \phi_{\cF, t}(G)_{c'} + \min(\phi_{t}(G)_{c}-\phi_{\cF, t}(G)_{c'}, \phi_{t}(H)_{c}-\phi_{\cF, t}(H)_{c'})                               \\
		 & =     \phi_{\cF, t}(G)_{c'} +\phi_{t}(H)_{c}-\phi_{\cF, t}(H)_{c'}                                                                             \\
		 & <     \phi_{t}(H)_{c} = \min(\phi_{t}(G)_{c}, \phi_{t}(H)_{c}),
	\end{align*}
	proving that $(3)$ is false if $(4)$ is false and by contraposition $(3) \implies (4)$.
\end{proof}

\begin{corollary}\label{APP:cor:example}
	For $n \geq 8$, there exists a graph sample $(G_1,y_1),\dotsc,(G_s,y_s) \in \cG_n\times\{0,1\}$ and $\cF$ such that
	\begin{enumerate}
		\item
		      $
			      \forall i,j \in [s] \colon y_i = y_j \implies
			      \left\lVert\phi^{(T)}_{\textsf{WLOA},\cF}(G_i)-\phi^{(T)}_{\textsf{WLOA},\cF}(G_j)\right\rVert = \left\lVert\phi^{(T)}_{\textsf{WLOA}}(G_i)-\phi^{(T)}_{\textsf{WLOA}}(G_j)\right\rVert.
		      $
		\item
		      $
			      \forall i,j \in [s] \colon y_i \neq y_j \implies
			      \left\lVert\phi^{(T)}_{\textsf{WLOA},\cF}(G_i)-\phi^{(T)}_{\textsf{WLOA},\cF}(G_j)\right\rVert > \left\lVert\phi^{(T)}_{\textsf{WLOA}}(G_i)-\phi^{(T)}_{\textsf{WLOA}}(G_j)\right\rVert.
		      $
	\end{enumerate}
\end{corollary}
\begin{proof}
	In this proof, we will construct two graphs, i.e., $s=2$, such that the above conditions hold. Note that the above conditions are the conditions from \cref{APP:thm:wloa} and \cref{cor:wloa}; thus, we can and will use their equivalent notions.

	The graph sample will be $(G_1, 1), (G_2, 0) \in \cG_8\times\{0,1\}$. For both graphs, we start with the 8-cycle $C_8$ and add edges. For $G_1$ we add all skip-lengths of 2, so $G_1 \coloneqq ([8], \{(i,i+1 \textsf{ mod }8) \mid i\in [8]\}\cup\{(i,i+2\textsf{ mod }8) \mid i\in [8]\})$. As for $G_2$, we add all skip-lengths of 3, so $G_2 \coloneqq ([8], \{(i,i+1 \textsf{ mod }8) \mid i\in [8]\}\cup\{(i,i+3\textsf{ mod }8) \mid i\in [8]\})$. Note that $G_1$ contains triangles, while $G_2$ does not. Also, note that both graphs only have one orbit (verified by the permutation in cycle notation $(1,2,3,4,5,6,7,8)$). Consider $\cF \coloneqq \{C_3\}$ to contain a triangle.

	Notice that the first condition is true since we are considering one graph of each class and thus $G_i=G_j$. As for the second condition, we will consider the equivalent notion that
	$$
		\exists\, t\in[T]\cup\{0\}\, \exists\, c\in \Sigma_t\ \colon \neg\left(\phi_t(G)_{c}\geq \phi_t(H)_c \iff \forall\, c'\in C_\cF(c):\phi_{\cF, t}(G)_{c'} \geq \phi_{\cF,t}(H)_{c'}\right).
	$$
	We must prove this statement to be true for our construction.
	Consider $t=0$, $c$ to be the color that all nodes in $G_1$ and $G_2$ are colored (not considering $\cF$) and let $c'$ be the color of being contained in a triangle (considering $\cF$). Then by definition
	$$
		\phi_t(G_1)_{c}\leq \phi_t(G_2)_c \wedge \phi_{\cF, t}(G_1)_{c'} > \phi_{\cF,t}(G_2)_{c'}.
	$$
	The first inequality holds since they are in fact equal $n=\phi_t(G_1)_{c}\leq \phi_t(G_2)_c=n$ and the second holds, since all nodes in $G_1$ are contained in triangles and no nodes in $G_2$ are contained in triangles and thus $n=\phi_{\cF, t}(G_1)_{c'} > \phi_{\cF,t}(G_2)_{c'}=0$ which proves the statement above is true.

	This construction can be extended such that both classes contain more than one graph by considering $n$-order 4-regular graphs, where one class contains graphs where all nodes are contained in triangles, and the other class contains graphs where no node is contained in a triangle. By definition, the first condition will always hold since the graphs in each class cannot be distinguished by \wlone{} or \wlonef{}, i.e., the intra-class distances will be 0 regardless of considering $\cF$.
\end{proof}

\section{Large margins and stochastic gradient descent}

We present the proofs of our main results from~\cref{sec:sgd}, i.e., \cref{thm:main-alignment-gradflow} and \cref{thm:main-maxmargin}. We will need some supporting lemmas, which we state and prove next. We note that the proof structure is close to the one in \citet{JiT19}.

\subsubsection{Setup}
Recall that we consider \emph{linear} $L$-layer MPNNs following \cref{def:MPNN} with trainable weight matrices $\vec{W}^{(i)} \in \Rb^{d_{i}\times d_{i-1}}$. Moreover, in our \new{linear MPNN}, after $L$ layers, the final node embeddings $\vec{X}^{(L)}$ are given by
\[
	\vec{X}^{(L)} \coloneqq \wmat^{(L)} \wmat^{(L-1)} \cdots \wmat^{(1)} \vec{X}^{(0)} \vec A'(G)^L,
\]
where $\vec{A}'(G) \coloneqq \vec{A}(G)+\vec{I}_n$, $\vec{I}_n\in\Rb^{n\times n}$ is the $n$-dimensional identity matrix, and $\vec{X} = \vec{X}^{(0)}$ is the $d_0 \times n$ matrix whose columns correspond to vertices' initial features; $d = d_0$.

These node embeddings are then converted into predictions
\[
	\hat{y} \coloneqq \RO\bigl( \vec X^{(L)} \bigr) = \vec{X}^{(L)} \vec{1}_n = \wmat^{(L)} \wmat^{(L-1)} \cdots \wmat^{(1)} \vec{X}^{(0)} \vec A'(G)^L \vec{1}_n.
\]

In our analysis, we will often need to reason about the singular values of the weight matrices. For $j=1,2,\dots,L$, we let $\sigma_j(t)$ denote the largest singular value of $\wj(t)$, and we let $\vec{u}(t)$ and $\vec{v}(t)$ denote the left-singular and right-singular vectors, respectively, corresponding to this singular value.

Recall that the training dataset is $\{(G_i, \vec X_i, y_i)\}_{i=1}^k$, where $\vec X_i \in \Rb^{d_i \times n_i}$ is a set of $d_i$-dimensional node features over an $n_i$-order graph $G_i$ with $|V(G_i)| \eqqcolon n_i$, and $y_i \in \{-1, +1\}$ for all $i$. Also, we write $d=d_0$ for the input node feature dimension. We further recall that the loss function $\ell$ satisfies the following assumptions.

\lossassumption*

The empirical risk induced by the MPNN is
\begin{align*}
	\cR(\vec W^{(L)}, \dots, \vec W^{(1)}) & \coloneqq \frac{1}{k} \sum_{i=1}^k \ell(y_i, \hat{y}_i)                          \\
	                                       & = \frac{1}{k} \sum_{i=1}^k \ell(\wprod \vec{Z}_i \vec{A}'(G_i)^L \vec{1}_{n_i}),
\end{align*}
where $\wprod = \vec W^{(L)} \vec W^{(L-1)} \cdots \vec W^{(1)}$, and $ \vec{Z}_i = y_i \vec X_i$.

For convenience, it will often be useful to write $\cR$ as a function of the product $\wprod$. Let $\cR_1$ be the risk function $\cR$ written as a function of the product $\wprod$, i.e.,
\[
	\cR_1(\wprod) \coloneqq \frac{1}{k} \sum_{i=1}^k \ell(\wprod \vec{Z}_i \vec{A}'(G_i)^L \vec{1}_{n_i}).
\]

We will consider \new{gradient flow}. In gradient flow, the evolution of $\vec W = (\wmat^{(L)}, \wmat^{(L-1)}, \dots, \wmat^{(1)})$ is given by $\{\vec W(t) \colon t\geq 0\}$, where there is an initial state $\vec W(0)$ at $t=0$, and
\[
	\frac{d\vec W(t)}{dt} \coloneqq -\nabla \cR(\vec W(t)).
\]
Note that gradient flow satisfies the following:
\begin{equation}\label{eq:gradexpression}
	\frac{d\cR(\wmat(t))}{dt} = \left\langle \nabla\cR(\wmat(t)), \frac{d\wmat(t)}{dt}\right\rangle = -\|\nabla\cR(\wmat)\|_2^2 = -\sum_{j=1}^L \left\| \frac{\partial\cR}{\partial\wmat^{(j)}} \right\|_F^2,
\end{equation}
which implies that the risk never increases. The discrete version of this is given by
\[
	\vec W(t+1) \coloneqq \vec W(t) - \eta_t \nabla\cR(\vec W(t)),
\]
which corresponds to gradient descent with step size $\eta_t$. Recall that we make the following assumption on the initialization of the network under consideration:
\initassumption*

\subsubsection{Lemmas and Theorems}
The proof structure of our main theorems largely follows that of \citet{JiT19}, except with the main change that $x_i \mapsto X_i \vec{A}'(G)^L \vec{1}_{n}$ and $z_i \mapsto Z_i \vec{A}'(G)^L \vec{1}_{n}$. Many of the lemmas follow directly from the relevant lemma in \citet{JiT19} with this transformation; we therefore defer to their proofs for a number of lemmas.

We start with a lemma that relates the weight matrices at successive levels to each other under the dynamics of gradient flow. This is essentially Theorem 1 of \citet{AroraCH18} applied to our setting---our $\cR_1$ and $\cR$ correspond to $L^1$ and $L^N$, respectively, in the aforementioned work.
\begin{lemma}[Theorem 1 in \citet{AroraCH18}]\label{lem:winvariant}
	$(\wjp)^\tran(t) \wjp(t) - \wj(t) (\wj)^\tran(t)$ is a constant function of $t$.
\end{lemma}
\begin{proof}
	For each $j = 1, 2,\dots, L$,
	\[
		\frac{\partial\cR}{\partial \wj} = \prod_{i=j+1}^L (\wmat^{(i)})^\tran \cdot \frac{d\cR_1}{d\wprod}(\wmat^{(L)} \wmat^{(L-1)} \cdots \wmat^{(1)}) \cdot \prod_{i=1}^{j-1} (\wmat^{(i)})^\tran.
	\]
	Hence, $\dot{\wmat}^{(j)} = \frac{d\wmat}{dt}$ is given by
	\begin{align*}
		\dot{\wmat}^{(j)} & = -\nabla \cR(\vec W(t))                                                                                                                                                          \\
		                  & = - \eta  \prod_{i=j+1}^L (\wmat^{(i)}(t))^\tran \cdot \frac{d\cR}{d\wmat}(\wmat^{(L)}(t) \wmat^{(L-1)}(t) \cdots \wmat^{(1)}(t)) \cdot \prod_{i=1}^{j-1} (\wmat^{(i)}(t))^\tran.
	\end{align*}
	Right multiplying the equation for $j$ by $(\wj)^\tran (t)$ and left multiplying the equation for $j+1$ by $(\wjp)^\tran (t)$, we see that
	\[
		(\wjp)^\tran (t) \dot{\wmat}^{(j+1)} (t) = \dot{\wmat}^{(j)} (t) (\wj)^\tran (t).
	\]
	Adding the above equation to its transpose, we obtain
	\[
		(\wjp)^\tran (t) \dot{\wmat}^{(j+1)} (t) + (\dot{\wmat}^{(j+1)})^\tran (t) \wjp(t) = \dot{\wmat}^{(j)} (t) (\wj)^\tran (t) + \wj(t) (\dot{\wmat}^{(j)})^\tran (t).
	\]
	Note that this is equivalent to
	\[
		\frac{d}{dt}\left[(\wjp)^\tran (t) \wjp(t)\right] = \frac{d}{dt}\left[\wjp(t) (\wjp)^\tran (t)\right],
	\]
	which implies that $(\wjp)^\tran (t) \wjp(t) - \wjp(t) (\wjp)^\tran (t)$ does not depend on $t$, as desired.
\end{proof}

For the remainder of this section, let $B(R)$ denote the set of $\wmat = (\wmat^{(L)}, \wmat^{(L-1)}, \dots, \wmat^{(1)})$ for which each component is bounded by $R$ in Frobenius norm, i.e.,
\[
	B(R) = \left\{\wmat \colon \max_{1\leq j\leq L} \|\wj\|_F \leq R \right\}.
\]

We now present the following lemma, which shows that the partial derivative of the risk function with respect to the first weight matrix $\wmat^{(1)}$ is bounded away from 0 in Frobenius norm.
\begin{lemma}\label{lem:riskpartialbd}
	For any $R > 0$, there exists a constant $\epsilon_R > 0$ such that for any $t\geq 1$ and  $(\vec{W}^{(L)}(t), \vec{W}^{(L-1)}(t), \cdots, \vec{W}^{(1)}(t)) \in B(R)$, we have $\|\partial\cR(t)/\partial \vec{W}^{(1)}(t)\|_F \geq \epsilon_R$.
\end{lemma}
\begin{proof}
	The lemma is the same as the first part of Lemma 2.3 in \cite{JiT19}. Therefore, we defer to the proof there.
\end{proof}

Our main interest in \cref{lem:riskpartialbd} is that it allows us to prove the following important corollary, which establishes that under gradient flow, the weight matrices grow unboundedly in Frobenius norm and do not spend much time inside a ball of any fixed finite radius.
\begin{corollary}\label{cor:finite-measure}
	Under gradient flow subject to \cref{assumption:loss} and \cref{assumption:risk}, $\{t\geq 0 \colon \vec{\vec{W}}(t) \in B(R)\}$ has finite measure.
\end{corollary}
\begin{proof}
	The corollary corresponds to the second part of \citet{JiT19}. We reproduce the proof here. Note that since $d\cR(\vec{W}(t))/dt = -\|\nabla\cR(\vec{W}(t))\|_F^2 \leq 0$ for all $t \geq 0$ (see \cref{eq:gradexpression}),
	\begin{align*}
		\cR(\vec{W}(0)) & \geq -\int_0^\infty \frac{dR(\vec{W}(t))}{dt}\,dt                                                                      \\
		                & = \int_0^\infty \left\| \frac{\partial \cR(t)}{\partial \vec{W}(t)} \right\|_F^2\, dt                                  \\
		                & = \int_0^\infty \left(\sum_{j=1}^L \left\| \frac{\partial \cR(t)}{\partial \vec{W}^{(j)}(t)} \right\|_F^2 \right)\, dt \\
		                & \geq \int_0^\infty \left\| \frac{\partial \cR(t)}{\partial \vec{W}^{(1)}(t)} \right\|_F^2\, dt                         \\
		                & \geq \int_1^\infty \left\| \frac{\partial \cR(t)}{\partial \vec{W}^{(1)}(t)} \right\|_F^2\, dt                         \\
		                & \geq \epsilon(R)^2 \int_1^\infty \Ibb[\vec{W}(t) \in B(R)]\,dt,
	\end{align*}
	where the final implication holds due to \cref{lem:riskpartialbd}. Since $\cR(\wmat(0))$ is finite, this implies that $\{t\geq 0 \colon \vec{W}(t) \in B(R)\}$ has finite measure.
\end{proof}

We now define the following notation for convenience:
\begin{align*}
	\vec{B}_j(t) & \coloneqq \vec{W}^{(j)}(t) (\vec{W}^{(j)})^\tran (t) - \vec{W}^{(j+1)}(t) (\vec{W}^{(j+1)})^\tran (t), \text{ and }                       \\
	D            & \coloneqq \left(\max_{1\leq j\leq L} \|\vec{W}^{(j)}(0)\|_F^2 \right) - \|\vec{W}^{(L)}(0)\|_F^2 + \sum_{j=1}^{L-1} \|\vec{B}_j(0)\|_2^2.
\end{align*}

While the previous corollary allows us to show the unboundedness of the weight matrices in the Frobenius norm, we often need to reason about the weight matrices in the standard operator norm. The following lemma shows that the two norms can not differ by too much.
\begin{lemma}\label{lem:dbd}
	For every $1\leq i \leq L$, we have $\|\vec{W}^{(i)}\|_F^2 - \|\vec{W}^{(i)}\|_2^2 \leq D$.
\end{lemma}
\begin{proof}
	A proof appears in \cite{JiT19}; see part 1 of Lemma 2.6.
\end{proof}

The next lemma is the key to establishing the ``alignment'' property. Roughly speaking, it establishes that the largest left singular vector of a weight matrix gets minimally aligned with the largest right singular vector of the weight matrix in the successive round of message passing.
\begin{lemma} \label{lem:succ-singularvalues}
	For all $1\leq j \leq L$, we have
	\[
		\langle \vec v_{j+1}, \vec u_j \rangle^2 \geq 1 - \frac{D + \|\wj(0)\|_2^2 + \|\wjp(0)\|_2^2}{\sigma_{j+1}^2}.
	\]
\end{lemma}
\begin{proof}
	Once again, the proof appears in \cite{JiT19} (see part 2 of Lemma 2.6).
\end{proof}

The previous two lemmas can be used to establish the following lemma, which shows that each (normalized) weight matrix tends to a rank-1 approximation given by its top left and right singular vectors, and the (normalized) partial product of weight matrices tend to the relevant right singular vector of the final weight matrix in the product. We note that the first part of the lemma appears in Theorem 2.2 of \cite{JiT19}; however, the second part does not appear explicitly in their work (although the proof is similar to the third part of Lemma 2.6 in \cite{JiT19}). Therefore, we provide the proof below.
\begin{lemma} \label{lem:alignment}
	Suppose $\min_{1\leq j\leq L} \|\wj(t)\|_F \to \infty$ as $t \to\infty$. For any $1\leq j\leq L$, we have,
	\begin{itemize}
		\item $\wj(t)/\|\wj(t)\|_F \to \vec{u}_j(t) \vec{v}_j(t)^\tran$ as $t\to\infty$.
		\item Also,
		      \[
			      \left| \frac{\vec{W}^{(L)}(t) \vec{W}^{(L-1)}(t) \cdots \vec{W}^{(j)}(t)}{\|\vec{W}^{(L)}(t)\|_F \|\vec{W}^{(L-1)}(t)\|_F \cdots \|\vec{W}^{(j)}(t)\|_F} \vec{v}_j(t) \right| \to 1
		      \]
		      as $t\to\infty$.
	\end{itemize}
\end{lemma}
\begin{proof}
	Since $\|\wj(t)\|_F \to \infty$, \cref{lem:dbd} implies that, as $t\to\infty$, $\|\wj(t)\|_2 \to \infty$, and, moreover, the singular values of $\vec{W}^{(j)}(t)$ beyond the top singular value are dominated by $\|\wj(t)\|_F$. Thus, $\wj(t)/\|\wj(t)\|_F \to \vec{u}_j(t) \vec{v}_j(t)^\tran$, which establishes the first part.

	For the second part, note that by \cref{lem:succ-singularvalues} and the fact that $\sigma_j = \|\wj(t)\|_2 \to \infty$, we have that $|\langle \vec{u}_j(t), \vec{v}_{j+1}(t) \rangle| \to 1$. Hence, for any $j$, we have
	\begin{align*}
		\left| \frac{\vec{W}^{(L)} \vec{W}^{(L-1)} \cdots \vec{W}^{(j)}}{\|\vec{W}^{(L)}\|_F \|\vec{W}^{(L-1)}\|_F \cdots \|\vec{W}^{(j)}\|_F} \vec{v}_j \right| & \to \left| (\vec{u}_L \vec{v}_L^\tran) \cdots (\vec{u}_j \vec{v}_j^\tran) \vec{v}_j \right|                                    \\
		                                                                                                                                                         & = \left| \vec{u}_L (\vec{v}_L^\tran \vec{u}_{L-1}) \cdots (\vec{v}_{j+1}^\tran \vec{u}_j) (\vec{v}_j^\tran \vec{v}_j)  \right| \\
		                                                                                                                                                         & \to |\vec{u}_L|                                                                                                                \\
		                                                                                                                                                         & = 1
	\end{align*}
	as $t\to\infty$, which completes the proof.
\end{proof}

The following theorem shows that under gradient flow, the risk goes to zero as $t \to \infty$, while the Frobenius norm of each weight matrix tends to infinity. The theorem corresponds to parts 1 and 2 of Theorem 2.2 in \citet{JiT19}; therefore, we defer to the proofs there.
\begin{theorem}[Parts 1 and 2 of Theorem 2.2 in \cite{JiT19}] \label{thm:risktozero}
	We have the following:
	\begin{itemize}
		\item $\lim_{t\to\infty} \cR(\wmat(t)) = 0$.
		\item For all $i = 1,2,\dots, L$, we have $\lim_{t\to\infty} \|\vec{W}^{(i)}(t)\|_F = \infty$.
	\end{itemize}
\end{theorem}
\begin{proof}
	See the proof of parts 1 and 2 of Theorem 2.2 in \cite{JiT19}.
\end{proof}

Our main alignment result for linear MPNNs is the following, whose proof follows easily from the previous lemmas.
\mainalignmentgradflow
\begin{proof}
	Note that by \cref{thm:risktozero}, we have that $\|\wj\|_F \to \infty$ for every $j$. Thus, the first part of \cref{lem:alignment} implies the first part of the theorem. Note that setting $j=1$ in the second part of \cref{lem:alignment} implies the second part of the theorem, completing the proof.
\end{proof}

\subsubsection{Margin}

We now state results on the margin.
\begin{lemma}
	Suppose the data set $\{(\vec{X}_i, y_i)\}_{i=1}^k$ and $G_i$ on $n_i$ nodes are sampled according to \cref{assumption:supportspan}. Let $S \subset \{1, 2, \dots, k\}$  be the set of indices for support vectors. Then,
	\begin{align}\label{eq:minmax-innerprod}
		\min_{\substack{\|\pmb\xi\|_2 = 1 \\ \langle\pmb\xi, \bar{\vec{u}}\rangle = 0}} \max_{i\in S} \left\langle \pmb\xi, Z_i \vec{A}'(G_i)^L \vec{1}_{n_i} \right\rangle  > 0
	\end{align}
	with probability $1$ over the sampling.
\end{lemma}
\begin{proof}
	First, we note that there are $s \leq d$ support vectors; furthermore, each support vector $Z_i \vec{A}'(G_i)^L \vec{1}_{n_i}$ has a corresponding dual variable $\alpha_i$ that is \emph{positive}, so that
	\begin{equation}
		\sum_{i\in S} \alpha_i Z_i \vec{A}'(G_i)^L \vec{1}_{n_i} = \bar{\vec{u}}. \label{eq:alphaeq}
	\end{equation}
	This follows from \citet{SoudryHNGS18} (see Lemma 12 in Appendix B), which was also used by \citet{JiT19}).

	Next, assume for the sake of contradiction that there exists $\pmb\xi$ with $\|\pmb\xi\|_2 = 1$ and $\langle \pmb\xi, \bar{\vec{u}} \rangle = 0$ but
	\[
		\max_{1\leq i\leq k} \left\langle \pmb\xi, Z_i \vec{A}'(G_i)^L \vec{1}_{n_i} \right\rangle \leq 0.
	\]
	Then, note that
	\begin{align*}
		0 & = \langle \pmb\xi, \bar{\vec{u}} \rangle                                                       \\
		  & = \left\langle \pmb\xi, \sum_{i\in S} \alpha_i Z_i \vec{A}'(G_i)^L \vec{1}_{n_i} \right\rangle \\
		  & = \sum_{i\in S} \alpha_i \left\langle \pmb\xi, Z_i \vec{A}'(G_i)^L \vec{1}_{n_i} \right\rangle \\
		  & \leq 0.
	\end{align*}
	This implies that $\left\langle \pmb\xi, Z_i \vec{A}'(G_i)^L \vec{1}_{n_i} \right\rangle = 0$ for all $i\in S$, which contradicts our assumption that the support vectors span the entirety of $\Rb^d$. This completes the proof.
\end{proof}

\begin{lemma} \label{lem:wperp-innerprod}
	Suppose \cref{assumption:supportspan} holds. Let $\ell$ be the exponential loss given by $\ell(x) = e^{-x}$. For almost all data, if $\vec{w}
		\in\Rb^d$ satisfies $\langle \vec{w}, \vec{u} \rangle \geq 0$ and $\vec{w}^\perp$, the projection of $\vec{w}$ on to the subspace of $\Rb^d$ orthogonal to $\vec{u}$, satisfies $\|\vec{w}^\perp\|_2 \geq \frac{1+\ln(k)}{\alpha}$, then $\langle \vec{w}^\perp, \nabla\cR(\vec{w})\rangle \geq 0$ (recall $\alpha$ from \cref{eq:alphaeq}).
\end{lemma}
\begin{proof}
	Let $\vec{v}_j = \vec{Z}_j \vec{A}'(G_j)^L \vec{1} = y_j \vec{X}_j \vec{A}'(G_j)^L$. Moreover, for any $\vec{z} \in \Rb^d$ let $\vec{z} = \vec{z}^\parallel + \vec{z}^\perp$, where $\vec{z}^\parallel$ is the projection of $\vec{z}$ on to $\vec{u}$ and $\vec{z}^\perp$ is the component of $\vec{z}$ orthogonal to $\vec{u}$. Let $j' = \argmax_{j\in S} \langle -\vec{w}^\perp, \vec{v}_j \rangle$ (recall that $S$ is the index set for support vectors).. We note that $-\langle \vec{w}^\perp, \vec{v}_{j'}^\perp \rangle = -\langle \vec{w}^\perp, \vec{v}_{j'} \rangle \geq \alpha \|\vec{w}^\perp\|$, where $\alpha$ is the quantity on the lefthand side of \eqref{eq:minmax-innerprod}.

	Observe that
	\begin{align}
		\langle \vec{w}^\perp, \nabla\cR(\vec{w}^\tran) \rangle & = \frac{1}{k} \sum_{i=1}^k \ell'(\langle \vec{w}, \vec{v}_i\rangle) \cdot \langle \vec{w}^\perp, \vec{v}_i \rangle \nonumber                                     \\
		                                                        & = -\frac{1}{k} \sum_{i=1}^k \exp(-\langle \vec{w}, \vec{v}_i\rangle) \cdot \langle \vec{w}^\perp, \vec{v}_i^\perp \rangle \nonumber                              \\
		                                                        & = -\frac{1}{k} \exp(-\langle \vec{w}, \vec{v}_{j'} \rangle) \cdot \langle \vec{w}^\perp, \vec{v}_{j'}^\perp \rangle - \frac{1}{k} \sum_{\substack{1\leq i \leq k \\ \langle \vec{w}^\perp, \vec{v}_i^\perp \rangle \geq 0}} \exp(-\langle \vec{w}, \vec{v}_{i} \rangle) \cdot \langle \vec{w}^\perp, \vec{v}_{i}^\perp \rangle. \label{eq:grad-innerprod}
	\end{align}

	The first term on the righthand side of \eqref{eq:grad-innerprod} can be bounded as follows:
	\begin{align}
		-\frac{1}{k} \exp(-\langle \vec{w}, \vec{v}_{j'} \rangle) \cdot \langle \vec{w}^\perp, \vec{v}_{j'}^\perp \rangle & = -\frac{1}{k} \exp\left(-\langle \vec{w}, \vec{v}_{j'}^\parallel \rangle - \langle  \vec{w}, \vec{v}_{j'}^\perp \rangle \right) \cdot \langle \vec{w}^\perp, \vec{v}_{j'}^\perp \rangle \nonumber                                 \\
		                                                                                                                  & = -\frac{1}{k} \exp\left(-\langle \vec{w}^\parallel, \vec{v}_{j'}^\parallel\rangle \right) \exp\left( -\langle \vec{w}^\perp, \vec{v}_{j'}^\perp \rangle \right) \cdot \langle \vec{w}^\perp, \vec{v}_{j'}^\perp \rangle \nonumber \\
		                                                                                                                  & \geq \frac{1}{k} \exp(-\langle \vec{w}, \gamma\vec{u}\rangle ) \exp(\alpha\|\vec{w}^\perp\|)\cdot \alpha\|\vec{w}^\perp\|. \label{eq:grad-innerprod-first}
	\end{align}
	For the second term in \eqref{eq:grad-innerprod}, we have
	\begin{align}
		\sum_{\substack{1\leq i \leq k                                                                           \\ \langle \vec{w}^\perp, \vec{v}_i^\perp \rangle \geq 0}} -\frac{1}{k} \exp(-\langle \vec{w}, \vec{v}_{i} \rangle) \cdot \langle \vec{w}^\perp, \vec{v}_{i}^\perp \rangle &= \sum_{\substack{1\leq i \leq k\\ \langle \vec{w}^\perp, \vec{v}_i^\perp \rangle \geq 0}} -\frac{1}{k} \exp(-\langle \vec{w}, \gamma \vec{\bar{u}} \rangle) \exp(-\langle \vec{w}, \vec{v}_{i} - \gamma \vec{\bar{u}} \rangle) \cdot \langle \vec{w}^\perp, \vec{v}_{i}^\perp \rangle \nonumber\\
		 & \geq \sum_{\substack{1\leq i \leq k                                                                   \\ \langle \vec{w}^\perp, \vec{v}_i^\perp \rangle \geq 0}} -\frac{1}{k} \exp(-\langle \vec{w}, \gamma \vec{\bar{u}} \rangle) \exp(-\langle \vec{w}^\perp, \vec{v}_i^\perp \rangle) \cdot \langle \vec{w}^\perp, \vec{v}_{i}^\perp \rangle \nonumber\\
		 & \geq  \sum_{\substack{1\leq i \leq k                                                                  \\ \langle \vec{w}^\perp, \vec{v}_i^\perp \rangle \geq 0}} \frac{1}{k} \exp(-\langle \vec{w}, \gamma \vec{\bar{u}} \rangle) (-e^{-1}) \nonumber\\
		 & \geq \exp(-\langle \vec{w}, \gamma \vec{\bar{u}} \rangle) (-e^{-1}), \label{eq:grad-innerprod-second}
	\end{align}
	since $xe^{-x} \leq -e^{-1}$ for $x\geq 0$, and the assumption $\langle \vec{w}, \vec{u} \rangle \geq 0$ along with the fact that $\vec{v}_i$ has margin at least $\gamma$ implies that $\langle \vec{w}, \vec{v}_i - \gamma \vec{u} - \vec{v}_i^\perp \rangle \geq 0$.

	By plugging \eqref{eq:grad-innerprod-first} and \eqref{eq:grad-innerprod-second} into \eqref{eq:grad-innerprod}, we obtain
	\[
		\langle \vec{w}^\perp, \nabla\cR(\vec{w}^\tran) \rangle \geq \exp(-\langle \vec{w}, \gamma\vec{\bar{u}}\rangle ) \left[ \frac{1}{k} \exp(\alpha\|\vec{w}^\perp\|)\cdot \alpha \|\vec{w}^\perp\| - e^{-1} \right].
	\]
	Finally, note that since $\|\vec{w}^\perp\| \geq (1 + \ln(k))/\alpha$ (by the assumption in the lemma), $\frac{1}{k} \exp(\alpha\|\vec{w}^\perp\|)\cdot \alpha \|\vec{w}^\perp\| - e^{-1} \geq 0$, which completes the proof.
\end{proof}

Our main theorem establishes the convergence of linear MPNNs to the maximum margin solution.
\mainmaxmargin
\begin{proof}
	The proof follows that of Theorem 2.8 in \cite{JiT19}, except that one uses \cref{assumption:supportspan} along with the transformations $x_i \mapsto \vec{X}_i \vec{A}'(G)^L \vec{1}_{n}$ and $z_i \mapsto \vec{Z}_i \vec{A}'(G)^L \vec{1}_{n}$, where the relevant support vectors are of the form $\vec{Z}_i \vec{A}'(G)^L \vec{1}_{n}$. The proof follows similarly from \cref{lem:wperp-innerprod} as in \cite{JiT19}.
\end{proof}

\section{Additional experimental results}\label{APP:exp}

Here, we state more experimental results.

\begin{table}
	\begin{center}
		\caption{Dataset statistics and properties.\label{statistics}}
		\resizebox{0.7\textwidth}{!}{ 	\renewcommand{\arraystretch}{1.05}
			\begin{tabular}{@{}lcccc@{}}\toprule
				\multirow{3}{*}{\vspace*{4pt}\textbf{Dataset}} & \multicolumn{4}{c}{\textbf{Properties}}                                                                                             \\
				\cmidrule{2-5}
				                                               & Number of  graphs                       & Number of classes/targets & $\varnothing$ Number of nodes & $\varnothing$ Number of edges \\ \midrule
				$\textsc{Enzymes}$                             & 600                                     & 6                         & 32.6                          & 62.1                          \\
				$\textsc{Mutag}$                               & 188                                     & 2                         & 17.9                          & 19.8                          \\
				$\textsc{Proteins}$                            & 1\,113                                  & 2                         & 39.1                          & 72.8                          \\
				$\textsc{PTC\_FM}$                             & 349                                     & 2                         & 14.1                          & 14.5                          \\
				$\textsc{PTC\_MR}$                             & 344                                     & 2                         & 14.3                          & 14.7                          \\
				\bottomrule
			\end{tabular}}
	\end{center}
\end{table}

\begin{table*}
	\caption{Mean train, test accuracies, and margins of kernel architectures on ER graphs for different levels of sparsity and different subgraphs. \textsc{Nls}---Not linearly separable. \textsc{Nb}---Only one class in the dataset.}
	\label{fig:er}
	\centering

	\resizebox{1.0\textwidth}{!}{ 	\renewcommand{\arraystretch}{1.05}
		\begin{tabular}{@{}c <{\enspace}@{}lcccc@{}} \toprule
			\multirow{3}{*}{\vspace*{4pt}\textbf{Subgraph}} & \multirow{3}{*}{\vspace*{4pt}\textbf{Algorithm}} & \multicolumn{4}{c}{\textbf{Edge probability}}                                                                                                                                                                                                                                                                                                                                                                                                               \\\cmidrule{3-6}
			                                                &                                                  & 0.05                                                                                                             & 0.1                                                                                                           & 0.2                                                                                                          & 0.3                                                                                                       \\	\toprule
			\multirow{4}{*}{ $C_3$ }                        & \wlone                                           & 94.2  {\scriptsize $\pm 1.7$}   88.6     {\scriptsize $\pm 0.7$}  \textsc{Nls}                                   & 96.3    {\scriptsize $\pm 4.6$}  42.12    {\scriptsize $\pm  1.7$}   0.013  {\scriptsize $\pm 0.001$}         & 44.1  {\scriptsize $\pm 12.4 $} 11.2  {\scriptsize $\pm $}  \textsc{Nls}                                     & 38.4  {\scriptsize $\pm  5.0$}  6.99   {\scriptsize $\pm 1.1$}  \textsc{Nls}                              \\
			                                                & \wloa                                            & 100.0  {\scriptsize $\pm 0.0$} 87.9   {\scriptsize $\pm 0.1$}   \textsc{Dnc}                                     & 100.0 {\scriptsize $\pm 0.0$}  44.6   {\scriptsize $\pm 1.1$}   \textsc{Dnc}                                  & 100.0   {\scriptsize $\pm 0.0$} 11.5  {\scriptsize $\pm 0.9$}   \textsc{Dnc}                                 & 100.0   {\scriptsize $\pm 0.0$}  5.2   {\scriptsize $\pm 0.3$}   \textsc{Dnc}                             \\

			                                                & \wlonef                                          & 100.0    {\scriptsize $\pm 0.0$}  \textbf{100.0}   {\scriptsize $\pm 0.0$} 0.037   {\scriptsize $\pm 0.0$}       & 100.0     {\scriptsize $\pm 0.0$}  \textbf{99.7}  {\scriptsize $\pm 0.1$}  0.009   {\scriptsize $<0.001$}     & 100.0   {\scriptsize $\pm 0.0$}  \textbf{100.0}  {\scriptsize $\pm 0.4$}    0.002   {\scriptsize $< 0.001$}  & 99.1    {\scriptsize $\pm 0.2$}  \textbf{64.7} {\scriptsize $\pm 0.8$} 0.001  {\scriptsize $< 0.001$}     \\
			                                                & \wloaf                                           & 100.0   {\scriptsize $\pm 0.0$} 98.6   {\scriptsize $\pm 0.0$}   \textsc{Dnc}                                    & 100.0   {\scriptsize $\pm 0.0$} 93.8   {\scriptsize $\pm 0.3$}   \textsc{Dnc}                                 & 100.0  {\scriptsize $\pm 0.0$}  42.0   {\scriptsize $\pm 1.0$}   \textsc{Dnc}                                & 100.0   {\scriptsize $\pm 0.0$} 7.4   {\scriptsize $\pm 0.2$}   \textsc{Dnc}                              \\
			\cmidrule{1-6}

			\multirow{4}{*}{ $C_4$ }                        & \wlone                                           & 95.1  {\scriptsize $\pm 1.0$} 92.3   {\scriptsize $\pm 0.2$}    \textsc{Nls}                                     & 89.7  {\scriptsize $\pm 3.9$}   46.8   {\scriptsize $\pm 0.8$}  \textsc{Nls}                                  & 48.7  {\scriptsize $\pm 10.1$}  5.4   {\scriptsize $\pm 0.5$}  \textsc{Nls}                                  & 46.6   {\scriptsize $\pm 10.8$}  2.2   {\scriptsize $\pm 0.4$}  \textsc{Nls}                              \\
			                                                & \wloa                                            & 100.0  {\scriptsize $\pm 0.0$}  92.6   {\scriptsize $\pm 0.0$}   \textsc{Dnc}                                    & 100.0 {\scriptsize $\pm 0.0$}  49.6   {\scriptsize $\pm 0.8$}   \textsc{Dnc}                                  & 100.0   {\scriptsize $\pm 0.0$} 5.13   {\scriptsize $\pm 0.6$}   \textsc{Dnc}                                & 100.0   {\scriptsize $\pm 0.0$} 1.7   {\scriptsize $\pm 0.3$}   \textsc{Dnc}                              \\
			                                                & \wlonef                                          & 100.0   {\scriptsize $\pm 0.0$}  \textbf{99.9}  {\scriptsize $\pm 0.1$}    0.037   {\scriptsize $\pm 0.001$}     & 100.0   {\scriptsize $\pm 0.0$}  \textbf{98.2}  {\scriptsize $\pm 0.2$}       0.009   {\scriptsize $< 0.001$} & 100.0   {\scriptsize $0.0$}  \textbf{78.9}  {\scriptsize $\pm 0.6$}    0.002   {\scriptsize $<0.001$}        & 100.0     {\scriptsize $\pm 0.0$}  \textbf{7.3} {\scriptsize $\pm 0.4$}      0.002 {\scriptsize $<0.001$} \\
			                                                & \wloaf                                           & 100.0   {\scriptsize $\pm 0.0$} 99.3   {\scriptsize $<0.1 $}   \textsc{Dnc}                                      & 100.0  {\scriptsize $\pm 0.0$} 93.7   {\scriptsize $\pm 0.2$}   \textsc{Dnc}                                  & 100.0   {\scriptsize $\pm $}  22.4  {\scriptsize $\pm 0.9$}   \textsc{Dnc}                                   & 100.0  {\scriptsize $\pm 0.0$} 2.8   {\scriptsize $\pm 0.6$}   \textsc{Dnc}                               \\
			\cmidrule{1-6}

			\multirow{4}{*}{ $C_5$ }                        & \wlone                                           & 97.2   {\scriptsize $\pm 0.5$}  95.8   {\scriptsize $\pm 0.3$}    \textsc{Nls}                                   & 69.3  {\scriptsize $\pm 6.6$}  53.5  {\scriptsize $\pm 0.6$}  \textsc{Nls}                                    & 65.1   {\scriptsize $\pm 14.9$}  4.3   {\scriptsize $\pm 0.7$}  \textsc{Nls}                                 & 64.8   {\scriptsize $\pm 9.9$} 1.26   {\scriptsize $\pm 0.2$}  \textsc{Nls}                               \\
			                                                & \wloa                                            & 100.0   {\scriptsize $\pm 0.0$} 95.9   {\scriptsize $\pm 0.0$}   \textsc{Dnc}                                    & 100.0   {\scriptsize $\pm 0.0$} 54.2   {\scriptsize $\pm 0.3$}   \textsc{Dnc}                                 & 100.0   {\scriptsize $\pm 0.0$} 4.7   {\scriptsize $\pm 0.5$}   \textsc{Dnc}                                 & 100.0 {\scriptsize $\pm 0.0$}  1.4   {\scriptsize $\pm 0.5$}   \textsc{Dnc}                               \\
			                                                & \wlonef                                          & 100.0    {\scriptsize $\pm 0.0$}    \textbf{99.9}   {\scriptsize $\pm 0.0$}  0.058     {\scriptsize $\pm 0.001$} & 100.0  {\scriptsize $\pm 0.0$}  \textbf{98.4} {\scriptsize $\pm 0.2$} 0.012  {\scriptsize $ <0.001$}          & 100.0 {\scriptsize $\pm 0.0$}  \textbf{68.8}  {\scriptsize $\pm 0.5$}   0.002   {\scriptsize $< 0.001 $}     & 100.0  {\scriptsize $\pm 0.0$}  \textbf{4.2}     {\scriptsize $\pm 0.5$}    0.003 {\scriptsize $< 0.001$} \\
			                                                & \wloaf                                           & 100.0  {\scriptsize $\pm 0.0$}  99.5  {\scriptsize $\pm 0.0$}  \textsc{Dnc}                                      & 100.0   {\scriptsize $\pm 0.0$}  91.8  {\scriptsize $\pm 0.4$}   \textsc{Dnc}                                 & 100.0  {\scriptsize $\pm 0.0$} 20.0   {\scriptsize $\pm 0.0$}   \textsc{Dnc}                                 & 100.0   {\scriptsize $\pm 0.0$} 2.6   {\scriptsize $\pm 0.2$}   \textsc{Dnc}                              \\
			\cmidrule{1-6}
			\multirow{4}{*}{ $K_4$ }                        & \wlone                                           & \textsc{Nb}                                                                                                      & 99.4  {\scriptsize $<0.1 $} 99.4  {\scriptsize $\pm 0.0$}  \textsc{Nls}                                       & 78.0   {\scriptsize $\pm 0.6$}   77.7  {\scriptsize $\pm 0.3$}  \textsc{Nls}                                 & 68.7    {\scriptsize $\pm 9.9$} 17.1   {\scriptsize $\pm 0.9$}  \textsc{Nls}                              \\
			                                                & \wloa                                            & \textsc{Nb}                                                                                                      & 100.0   {\scriptsize $\pm 0.0$} 99.4   {\scriptsize $\pm 0.0$}   \textsc{Dnc}                                 & 100.0 {\scriptsize $\pm 0.0$}  77.8   {\scriptsize $\pm 0.3$}   \textsc{Dnc}                                 & 100.0  {\scriptsize $\pm 0.0$}  20.6   {\scriptsize $\pm 0.9$}   \textsc{Dnc}                             \\
			                                                & \wlonef                                          & \textsc{Nb}                                                                                                      & 100.0  {\scriptsize $\pm 0.0$}  \textbf{100.0}  {\scriptsize $\pm 0.0$}   0.122    {\scriptsize $\pm 0.0$}    & 100.0    {\scriptsize $\pm 0.0$}  \textbf{99.9}   {\scriptsize $\pm 0.1$}  0.022     {\scriptsize $< 0.001$} & 100.0  {\scriptsize $\pm 0.0$}  \textbf{94.8}  {\scriptsize $\pm 0.4$}  0.004   {\scriptsize $<0.001$}    \\
			                                                & \wloaf                                           & \textsc{Nb}                                                                                                      & 100.0 {\scriptsize $\pm 0.0$}  99.4  {\scriptsize $\pm 0.0$}   \textsc{Dnc}                                   & 100.0   {\scriptsize $\pm 0.0$} 97.8   {\scriptsize $\pm 0.1$}   \textsc{Dnc}                                & 100.0  {\scriptsize $\pm 0.0$} 74.6   {\scriptsize $\pm 0.6$}   \textsc{Dnc}                              \\
			\bottomrule
		\end{tabular}
	}
\end{table*}

\begin{table}
	\caption{Mean train and test accuracies of the MPNN architectures on \textsc{TUDatasets} datasets for different subgraphs.}
	\label{fig:tud_gnn}
	\centering
	\resizebox{.76\textwidth}{!}{ 	\renewcommand{\arraystretch}{1.05}
		\begin{tabular}{@{}l <{\enspace}@{}lccccc@{}} \toprule
			\multirow{3}{*}{\vspace*{4pt} $\mathcal{F}$ } & \multirow{3}{*}{\vspace*{4pt}\textbf{Algorithm}} & \multicolumn{5}{c}{\textbf{Dataset}}                                                                                                                                                                                                                                                                                                                               \\\cmidrule{3-7}
			                                              &                                                  & {\textsc{Enzymes}}                                                     & {\textsc{Mutag}}                                                       & {\textsc{Proteins}}                                                    &
			{\textsc{PTC\_FM}}                            & {\textsc{PTC\_MR}}                                                                                                                                                                                                                                                                                                                                                                                                    \\	\toprule
			---                                           & \MPNN                                            & 25.5   {\scriptsize $\pm 1.9$} 20.9  {\scriptsize $\pm 1.5$}           & 84.5  {\scriptsize $\pm 1.2$}  82.4  {\scriptsize $\pm 1.9$}           & 68.4   {\scriptsize $\pm 0.8$} 67.1   {\scriptsize $\pm 0.8$}          & 61.4   {\scriptsize $\pm 1.8$} 59.2  {\scriptsize $\pm 2.4$}            & 56.7   {\scriptsize $\pm 2.3$}   53.7 {\scriptsize $\pm 3.0$} \\
			\cmidrule{1-7}
			$C_3$                                         & \MPNNF                                           & 61.3   {\scriptsize $\pm 4.1$} 29.2  {\scriptsize $\pm 1.5$}           & 84.2  {\scriptsize $\pm 1.9$} 81.8  {\scriptsize $\pm 2.6$}            & 72.7  {\scriptsize $\pm 0.8$}  67.8   {\scriptsize $\pm 0.8$}          & 61.9 {\scriptsize $\pm 2.4$} 59.7   {\scriptsize $\pm 2.4$}             & 57.3    {\scriptsize $\pm 0.9$}   \textbf{56.3}
			{\scriptsize $\pm 2.5$}                                                                                                                                                                                                                                                                                                                                                                                                                                               \\
			$C_3$-$C_4$                                   & \MPNNF                                           & 61.8   {\scriptsize $\pm 3.0$} 28.9  {\scriptsize $\pm 1.5$}           & 84.5   {\scriptsize $\pm 2.4$} 82.1  {\scriptsize $\pm 2.6$}           & 72.5 {\scriptsize $\pm 0.8$}    \textbf{69.0}  {\scriptsize $\pm 0.9$} & 62.0  {\scriptsize $\pm 0.8$} 59.4  {\scriptsize $\pm 1.9$}             & 57.5   {\scriptsize $\pm 0.8$} 54.3   {\scriptsize $\pm 1.3$} \\
			$C_3$-$C_5$                                   & \MPNNF                                           & 61.6  {\scriptsize $\pm 2.8$} 28.5  {\scriptsize $\pm 2.2$}            & 83.9 {\scriptsize $\pm 2.2$} 81.8   {\scriptsize $\pm 2.4$}            & 72.6   {\scriptsize $\pm 0.8$} 68.6   {\scriptsize $\pm 0.9$}          & 62.0  {\scriptsize $\pm 2.0$} 58.8 {\scriptsize $\pm 2.5$}              & 56.7   {\scriptsize $\pm 1.4$} 54.9  {\scriptsize $\pm 1.8$}  \\
			$C_3$-$C_6$                                   & \MPNNF                                           & 61.8  {\scriptsize $\pm 6.0$} 28.8  {\scriptsize $\pm 1.6$}            & 83.4   {\scriptsize $\pm 2.4$} 82.2   {\scriptsize $\pm 2.6$}          & 72.7   {\scriptsize $\pm 0.6$} 68.4   {\scriptsize $\pm 0.5$}          & 62.4    {\scriptsize $\pm 0.9$}  \textbf{60.5}  {\scriptsize $\pm 2.4$} & 57.2   {\scriptsize $\pm 1.3$} 54.0  {\scriptsize $\pm 2.0$}  \\
			\cmidrule{1-7}

			$K_3$                                         & \MPNNF                                           & 60.1  {\scriptsize $\pm 4.5$} 28.8   {\scriptsize $\pm 1.0$}           & 84.8 {\scriptsize $\pm 1.6$}  84.2   {\scriptsize $\pm 1.7$}           & 72.6    {\scriptsize $\pm 0.6$} 67.7    {\scriptsize $\pm 0.7$}        & 62.1   {\scriptsize $\pm 1.6$} 59.0 {\scriptsize $\pm 2.3$}             & 57.2   {\scriptsize $\pm 0.8$} 54.3   {\scriptsize $\pm 1.8$} \\
			$K_3$-$K_4$                                   & \MPNNF                                           & 61.8   {\scriptsize $\pm 5.5$} \textbf{29.5}   {\scriptsize $\pm 1.6$} & 83.2   {\scriptsize $\pm 2.6$} 80.9 {\scriptsize $\pm 2.4$}            & 73.2    {\scriptsize $\pm 0.8 $}  68.1    {\scriptsize $\pm 0.6$}      & 62.3   {\scriptsize $\pm 2.0$} 59.1    {\scriptsize $\pm 3.2$}          & 57.0  {\scriptsize $\pm 1.4$}  53.7  {\scriptsize $\pm 2.6$}  \\

			$K_3$-$K_5$                                   & \MPNNF                                           & 60.6 {\scriptsize $\pm 3.9$}  29.2   {\scriptsize $\pm 1.3$}           & 84.9    {\scriptsize $\pm 1.4$} \textbf{82.9}  {\scriptsize $\pm 2.1$} & 72.9  {\scriptsize $\pm 0.9$}  68.0    {\scriptsize $\pm 1.1$}         & 61.7  {\scriptsize $\pm 2.1$}  57.8   {\scriptsize $\pm 2.8$}           & 58.6   {\scriptsize $\pm 1.3$} 54.4   {\scriptsize $\pm 1.8$} \\

			$K_3$-$K_6$                                   & \MPNNF                                           & 60.3    {\scriptsize $\pm 4.7$} 28.8   {\scriptsize $\pm 1.4$}         & 84.6    {\scriptsize $\pm 1.3$}  82.4 {\scriptsize $\pm 2.5$}          & 73.4 {\scriptsize $\pm 0.7$} 68.5 {\scriptsize $\pm 0.9$}              & 62.2   {\scriptsize $\pm 1.3$} 59.3   {\scriptsize $\pm 1.6$}           & 57.7  {\scriptsize $\pm 1.0$} 54.4   {\scriptsize $\pm 2.8$}  \\
			\bottomrule
		\end{tabular}
	}
\end{table}
\begin{table}
	\caption{Mean test accuracies of MPNN architectures on ER graphs for different levels of sparsity and different subgraphs. \textsc{Nb}---Only one class in the dataset.}
	\label{fig:er_mpnn}
	\centering

	\resizebox{.45\textwidth}{!}{ 	\renewcommand{\arraystretch}{1.05}
		\begin{tabular}{@{}c <{\enspace}@{}lllll@{}} \toprule
			\multirow{3}{*}{\vspace*{4pt}\textbf{Subgraph}} & \multirow{3}{*}{\vspace*{4pt}\textbf{Algorithm}} & \multicolumn{4}{c}{\textbf{Probability}}                                                                                                                                \\\cmidrule{3-6}
			                                                &                                                  & 0.05                                     & 0.1                                     & 0.2                                      & 0.3                                     \\	\toprule
			\multirow{2}{*}{ $C_3$ }                        & $\text{MPNN}$                                    & 90.7   {\scriptsize $\pm 0.5$}           & 50.7   {\scriptsize $\pm 1.0$}          & 20.1   {\scriptsize $\pm 0.7$}           & 9.5    {\scriptsize $\pm 0.6$}          \\
			                                                & $\text{MPNN}_{\mathcal{F}}$                      & \textbf{98.8}   {\scriptsize $\pm 0.3$}  & \textbf{96.1} {\scriptsize $\pm 0.3$}   & \textbf{72.5}  {\scriptsize $\pm 1.3$}   & \textbf{32.9}  {\scriptsize $\pm 1.0$}  \\
			\cmidrule{1-6}
			\multirow{2}{*}{ $C_4$ }                        & $\text{MPNN}$                                    & 90.7  {\scriptsize $\pm 0.7$}            & 51.3 {\scriptsize $\pm 1.4$}            & 20.1   {\scriptsize $\pm 0.8$}           & 9.9  {\scriptsize $\pm 0.9 $}           \\
			                                                & $\text{MPNN}_{\mathcal{F}}$                      & \textbf{99.1}   {\scriptsize $\pm 0.2$}  & \textbf{96.2}   {\scriptsize $\pm 0.4$} & \textbf{73.3}    {\scriptsize $\pm 0.9$} & \textbf{33.1}   {\scriptsize $\pm 1.3$} \\
			\cmidrule{1-6}
			\multirow{2}{*}{ $C_5$ }                        & $\text{MPNN}$                                    & 90.6  {\scriptsize $\pm 0.5$}            & 51.3  {\scriptsize $\pm 1.2$}           & 20.1  {\scriptsize $\pm 0.6$}            & 9.9  {\scriptsize $\pm 0.6$}            \\
			                                                & $\text{MPNN}_{\mathcal{F}}$                      & \textbf{98.9}  {\scriptsize $\pm 0.2$}   & \textbf{96.2}   {\scriptsize $\pm 0.4$} & \textbf{73.6}  {\scriptsize $\pm 1.0$}   & \textbf{33.8}  {\scriptsize $\pm 1.0$}  \\
			\cmidrule{1-6}
			\multirow{2}{*}{ $K_4$ }                        & $\text{MPNN}$                                    & \textsc{Nb}                              & 51.0  {\scriptsize $\pm 1.2$}           & 20.1 {\scriptsize $\pm 1.0$}             & 10.1  {\scriptsize $\pm 1.3$}           \\
			                                                & $\text{MPNN}_{\mathcal{F}}$                      & \textsc{Nb}                              & \textbf{96.2}   {\scriptsize $\pm 0.5$} & \textbf{72.9}  {\scriptsize $\pm 0.9$}   & \textbf{33.6}   {\scriptsize $\pm 1.2$} \\
			\bottomrule
		\end{tabular}
	}
\end{table}

\end{document}